\documentclass[letterpaper]{article} 
\usepackage{aaai24}  
\usepackage{times}  
\usepackage{helvet}  
\usepackage{courier}  
\usepackage[hyphens]{url}  
\usepackage{graphicx} 
\usepackage{natbib}  
\usepackage{caption} 
\usepackage[linesnumbered,noend,ruled]{algorithm2e}

\usepackage{subcaption}
\usepackage{framed}
\usepackage{xcolor,mdframed}
\usepackage{adjustbox}
\usepackage{natbib}

\usepackage{amsmath,amsfonts}
\usepackage{amsthm}

\usepackage{multicol}

\newtheorem{theorem}{Theorem}

\newcommand{\citepw}[1]{\citeauthor{#1}~(\citeyear{#1})}

\newcommand{\fhat}[1]{{$\hat{f}(#1)$}}
\newcommand{\dhat}[1]{{$\hat{d}(#1)$}}
\newcommand{\fhatnoparens}[1]{{$\hat{f}$}}
\newcommand{\dhatnoparens}[1]{{$\hat{d}$}}

\newcommand{\aspect}{{\em aspect}}

\def \plotwidth {2.2in}

\newcommand{\incfig}[1]{
\includegraphics[width=\plotwidth{}]{#1}
}

\title{Rectangle Search: An Anytime Beam Search (Extended Version)}

\author {
    Sofia Lemons,\textsuperscript{\rm 1, 2}
    Wheeler Ruml,\textsuperscript{\rm 1}
    Robert C. Holte,\textsuperscript{\rm 3}
    Carlos Linares L\'opez,\textsuperscript{\rm 4}
}
\affiliations {
    \textsuperscript{\rm 1} University of New Hampshire\\
    \textsuperscript{\rm 2} Earlham College\\
    \textsuperscript{\rm 3} University of Alberta, Alberta Machine Intelligence Institute (Amii)\\
    \textsuperscript{\rm 4} Computer Science and Engineering Department, Universidad Carlos III de Madrid\\
    sofia@snlemons.com, ruml@cs.unh.edu, rholte@ualberta.ca, carlos.linares@uc3m.es
}

\begin{document}

\maketitle

\begin{abstract}

Anytime heuristic search algorithms try to find a (potentially
suboptimal) solution as quickly as possible and then work to find
better and better solutions until an optimal solution is obtained or
time is exhausted. The most widely-known anytime search algorithms are
based on best-first search.  In this paper, we propose a new
algorithm, rectangle search, that is instead based on beam search, a variant of breadth-first search.
It repeatedly explores alternatives at all depth levels and is thus best-suited to problems featuring deep local minima.
Experiments using a variety of popular search benchmarks suggest that rectangle search is competitive with fixed-width beam search and often performs better than the previous best anytime search algorithms.

\end{abstract}

\section{Introduction}

It is often convenient to have a heuristic search algorithm that can flexibly make use of however much time is available.  The search can be terminated whenever desired and returns the best solution found so far.  \citepw{dean:atp} termed these {\em anytime} algorithms. \citepw{russell:crt} further differentiated between {\em interruptible} algorithms, which quickly find a solution and then find better solutions as time passes, eventually finding an optimal plan if given sufficient time, and {\em contract} algorithms, which are informed of the termination time in advance and thus need only find a single solution before that time.
Anytime algorithms have been proposed as a useful tool for building intelligent systems \cite{zilberstein:uaa,zilberstein:ocr}.
While only a few contract search algorithms have been proposed \cite{dionne:das}, interruptible algorithms have been widely investigated and applied
\cite{likhachev:pld}.

As we review below, the most well-known interruptible anytime heuristic search algorithms are based on best-first search.  Best-first search is attractive as it is the basis for the optimally-efficient optimal search algorithm A* \cite{hart:fbh} and it is relatively well understood.  However, because anytime algorithms are intended for use cases in which the solutions found do not need to be proven optimal, and are not even expected to be optimal, it is not obvious that best-first search is the most appropriate choice of algorithmic architecture.

In this paper, we propose an interruptible algorithm called rectangle search that is inspired by beam search, which is based on breadth-first search rather than best-first search.  Rectangle search is straightforward and simple to implement: the beam incrementally widens and deepens.   We study rectangle search's performance experimentally on seven popular heuristic search benchmarks.  We find that, overall, rectangle search outperforms previously-proposed anytime search algorithms.  Furthermore, it tends to find solutions of comparable cost at similar times when compared to fixed-width beam search, with the added benefit of reducing solution cost over time. We investigate when it performs well or poorly, finding that it performs well in problems with large local minima.  Overall, rectangle search appears to mark a new state of the art for anytime search, while also serving as a replacement for fixed-width beam search.

\section{Background}

Before presenting rectangle search, we first review relevant prior work in anytime search and beam search.

\subsection{Anytime Search}


Most previous anytime algorithms are based on weighted A*~\cite{pohl.i:avoidance}, which is a best-first search using $f'(n) = g(n) + w \times h(n)$, where $g(n)$ is the cost to reach $n$ from the start, $h(n)$ is the estimated cost-to-go to the goal from $n$, and $w\ge1$. For example, anytime weighted A* (AWA*) \cite{hansen:ahs} runs weighted A* but retains a current incumbent solution and continues searching for better solutions until there are no open nodes with $f(n) = g(n) + h(n) < g(incumbent)$, thus proving that the incumbent is optimal.

Anytime Repairing A* (ARA*) \cite{likhachev:aap} is similar, but applies a schedule of decreasing weights ending with $w=1$. When a solution is found, it decreases the weight according to the schedule and reorders the open list. It terminates after finding a solution with $w=1$ or after exhausting all nodes with $f(n) < g(\mathit{incumbent})$.

 Anytime EES (AEES) \cite{thayer:aees} requires no weighting parameter and explicitly works to minimize the time between finding new solutions by using distance-to-go estimates $d(n)$. $d(n)$ is an estimate of the number of actions along a minimum-cost path from node $n$ to a goal. Often this can be computed while calculating $h(n)$. AEES maintains an open list ordered on an error-adjusted evaluation function \fhat{n}, a focal list ordered on an error-adjusted distance-to-go measurement \dhat{n}, and a cleanup list ordered on $f(n)$. It compares the current incumbent solution's cost to the lowest $f$-value among open nodes to determine a bound $w$ on solution suboptimality. The focal list maintains only nodes $n$ with \fhat{n} $< w \cdot$ \fhat{b} where $b$ is the lowest \fhatnoparens{}-valued node, because these are predicted to lead to a solution that is better than the current incumbent.  AEES was demonstrated to outperform other anytime algorithms including ARA*, Anytime Nonparametric A* (ANA*) \cite{vandenberg:ana}, and Restarting Weighted A* (RWA*) \cite{richter:jf}.

\subsection{Beam Search}


Beam search \cite{bisianiEAI1987} is an incomplete and suboptimal variant of breadth-first search. It expands a fixed maximum number of nodes at each depth level of the search, referred to as the beam width $b$. All nodes from the beam at the current level are expanded and then the best $b$ among those descendants are selected for the next level's beam. Only nodes at the same depth in the tree are compared.  It continues to search until either a solution is found or until no new states are reachable from the current depth. A variant of beam search using $d(n)$ for node selection called bead outperforms beam search using $f(n)$ or $h(n)$ in non-unit cost domains \cite{lemons:bsf}.

One issue with beam search is that when the beam width $b$ is increased, sometimes a higher cost solution is returned. The algorithm monobead \cite{lemons:bsf} addresses this issue. It regards the beam as an ordered sequence of numbered {\em slots}. Given nodes on the beam at current depth $d$, to fill beam slot $i$ at the next depth level, the node at slot $i$ of depth $d$'s beam is expanded, its children are added to a priority queue for depth $d+1$, and the best node from that queue is selected. This iterates for values of $i$ from $1$ to $b$, with the priority queue retaining any children that were not selected to fill previous slots. The node selected for slot $i$ at depth $d+1$ is thus restricted to be a child of a node in slots 1 through $i$ at depth $d$. This careful selection order prevents children of nodes at later slots from supplanting children of earlier slots and preserves any solutions that would have been found by searches with a narrower beam width.

\begin{figure}[t]
\centering
    \includegraphics[width=1.5in]{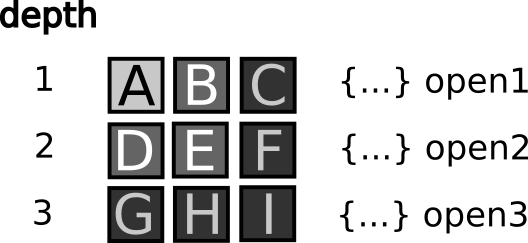}
    \caption{The exploration of rectangle search.}
    \label{fig:nodes}
\end{figure}

\begin{algorithm}[t]
\Begin{
$openlists \leftarrow [\emptyset]$\\ \label{rectangle:openlists}
$closed \leftarrow \emptyset$\\
$incumbent \leftarrow$ node with $g = \infty$\\
expand start and add children to $openlists[0]$\\
$depth \leftarrow 1$\\
\While { non-empty lists exist in $openlists$ } { \label{rectangle:mainloop}
\For { $i=0 \ldots length(openlists)-2$ } { \label{rectangle:iterateexisting}
select \& expand from $openlists[i]$\\
} 
extend $openlists$ with $\emptyset$ $aspect$ times\\ \label{rectangle:addnewopen} \label{rectangle:aspect}
\For { $j=i+1 \ldots length(openlists)-2$ } { \label{rectangle:iteratenew}
\For { $k=1 \ldots depth$ } {
select \& expand from $openlists[j]$\\
}
} 
$depth \leftarrow depth + aspect$\\
trim empty lists from start \& end of $openlists$\\ \label{rectangle:trimopens}
} 
return $incumbent$\\
} 
\caption{Pseudocode for rectangle search.}
\label{alg:rectanglecode}
\end{algorithm}

\SetKwRepeat{Do}{do}{while}

\begin{algorithm}[t]
\setcounter{AlgoLine}{16} 
\KwData{closed, openlists, i, incumbent}
\Begin{
\Do{$f(n) \geq g(incumbent)$ \label{SE:pruneexp}}{
$n \leftarrow$ remove first node from $openlists[i]$\\ \label{SE:selectn} \label{SE:nselected}
}
add $n$ to $closed$\\
$children \leftarrow expand(n)$\\ \label{SE:makechildren}
\For{each child in $children$} {
\If{$f(child) < g(incumbent)$} { \label{SE:prunegen}
\If{$child$ is a goal} { \label{SE:goalfound}
$incumbent \leftarrow child$ \\
report new incumbent \\
}
\Else {
$dup \leftarrow$ $child$'s entry in $closed$\\
\If{$child$ not in $closed$ or $g(child) < g(dup)$} { \label{SE:checkclosed}
add $child$ to $openlists[i+1]$\\ \label{SE:insertchild} \label{SE:nextopen}
}
}
}
} 
} 
\caption{Node selection \& expansion.}
\label{alg:selectexpand}
\end{algorithm}

\section{Rectangle Search}

One way to view rectangle search is as an iteratively widening and deepening monobead search that obviates the need to pre-specify $b$. In each iteration $i$ it is allowed to expand one new node at each previously explored depth level and to expand $i$ nodes at a new depth level. This last step makes the number of expansions at each allowed depth level equal, resulting in a rectangular shape of the state space being explored, as shown in Figure \ref{fig:nodes}. By default, the algorithm explores only one additional depth level at each iteration (resulting in a square), but this number of new levels can be adjusted by setting an \aspect\ parameter. We use rectangle($x$) to refer to rectangle search with $\mbox{\aspect} =x$.

Figure \ref{fig:nodes} demonstrates three iterations of rectangle(1). In the first iteration, node A is selected from the front of $open1$ (the open list for nodes at depth $1$) and expanded, with its children being inserted into $open2$. The second iteration begins with node B being selected from $open1$ and expanded and its children also inserted into $open2$. The algorithm is now allowed to expand two nodes at a new depth level, depth 2. At this level, nodes D and E are selected from the front of $open2$ and expanded, with their children being inserted into $open3$. The third iteration consisted of node C being selected at depth 1, node F being selected at depth 2, and nodes G, H, and I being selected at depth 3. If no solution has been found yet, the children of the nodes at depth 3 would be inserted into a new open list for depth level 4 (not shown).

Algorithms \ref{alg:rectanglecode} and \ref{alg:selectexpand} show the design of rectangle search. It creates an ordered collection of priority queues, called $openlists$, which initially contains only a single entry for storing nodes at depth level $1$ (line \ref{rectangle:openlists}). The main loop at line \ref{rectangle:mainloop} continues so long as there are open nodes at any depth level. For all depth levels except the deepest, a single node is selected and expanded, with its children being inserted into the next depth level's open list (line \ref{SE:insertchild}). Because of this structure, there will not have been any nodes previously expanded from the deepest depth level represented in $openlists$. New open lists are added to represent deeper levels (line \ref{rectangle:addnewopen}), and then up to $depth$ many nodes will be selected from all but the deepest depth level now represented in $openlists$. For rectangle(1), this will result in $depth$ many nodes being expanded at a single new depth level at each iteration of the main loop, and a single new priority queue being added to $openlists$ to hold the nodes at the new depth level to be explored in the next iteration. While rectangle search could order its priority queues on any criterion, we use $d(n)$ to encourage finding a first solution quickly.

\subsection{Properties of Rectangle Search}

\begin{theorem}
With an admissible heuristic, rectangle search is complete.
\end{theorem}
\begin{proof}
Rectangle search begins from the start state and terminates when no states remain at any depth level (line \ref{rectangle:mainloop}). Because it begins with an incumbent cost value of $\infty$, it will not prune any non-duplicate nodes until a solution is found (lines \ref{SE:pruneexp} and \ref{SE:prunegen}). It will search the entire reachable state space if necessary to find a solution.
\end{proof}

\begin{theorem}
With an admissible heuristic, the last solution returned by rectangle search is optimal.
\end{theorem}
\begin{proof}
Rectangle search prunes only nodes with $f$-values greater than or equal to the current incumbent (lines \ref{SE:pruneexp} and \ref{SE:prunegen}), which cannot lead to a better solution. It terminates when no nodes remain at any depth (line \ref{rectangle:mainloop}). Therefore, it expands all nodes with $f$-value less than or equal to the optimal solution and the last solution it returns will be optimal.
\end{proof}

Rectangle search is related to monobead search because both algorithms select nodes for expansion at each level in a well-defined order, with respect to beam slots. Rectangle does not hold this strict ordering when expanding nodes at a new depth level for the first time. It also performs duplicate elimination without regard to the slot at which a node was seen. However, once a depth level has had nodes selected and expanded, only a single node will be selected and expanded at a time for that level, preserving a monobead-like order from that point onward. Therefore, the solutions found by rectangle may not follow the monotonic ordering ensured by monobeam. To highlight the similarity, we define a variant of rectangle search called strict rectangle that expands nodes at the deepest current depth level one at a time, based on their slot index in the beam, and selects nodes to fill those slots at the next depth level immediately after the current node in that slot is expanded. Strict rectangle also only eliminates duplicate nodes if the previously seen version was expanded from the same slot as the duplicate or earlier. With these restrictions, the node selection for a specific slot and depth are exactly the same for monobead and strict rectangle search.

To estimate the overhead of rectangle search in relation to a fixed-width beam search, we analyze the number of nodes expanded by strict rectangle search relative to the number of nodes expanded by monobead search. For simplicity, we assume strict rectangle with $\mbox{\aspect} = 1$. Let the first solution found by monobead be generated when it expands a node $x$ in slot $s$ at depth $d$. To provide an upper bound on the extra work done by strict rectangle to find this solution, we assume the beam width for monobead is equal to $s$, no other solutions exist, and strict rectangle never runs out of nodes to fill its beam.  If $s=d$, strict rectangle will search exactly the same region of the search space as monobead and the number of nodes expanded by each will be equal. If $s > d$, then strict rectangle will expand $x$ during iteration $s$ and will have searched $s-d-1$ superfluous depth levels ($x$ will be expanded before any node at depth $s$ is expanded), each with $s-1$ nodes expanded. And if $s < d$, strict rectangle will expand $x$ during iteration $d$ and will have expanded $d-s$ superfluous nodes at depth levels $1$ to $d-1$.  This yields:
\[ overhead \leq 
\begin{cases} 
0 & s = d \\
(s-d-1)(s-1) & s > d \\
(d-1)(d-s) & s < d
\end{cases}
\]
In summary, the overhead of strict rectangle over monobead is at most quadratic in whichever of $d$ or $s$ is greater.

As we will see in our experimental evaluation, rectangle(1) is not well-suited for domains in which the heuristic is very accurate and solutions are very deep (e.g., $d \gg s$).  Rectangle search must do $O(d^2)$ work to reach depth $d$, whereas a best-first search might only need to do $O(d)$ if the heuristic is accurate.  In this sense, rectangle(1) errs too much on the side of exploration in such domains. Using a larger \aspect\ may help to reach deeper regions of the search space earlier, but will still require some extraneous exploration at higher levels.

On the other hand, rectangle search is not obliged to expand open nodes in order of their heuristic evaluation value (be it $f$, $h$, or $d$).  A {\em local minimum} (or {\em crater}) is a strongly connected set of states with deceptive heuristic values that are lower than the highest heuristic value necessary to encounter along a path from them to a goal \cite{wilt:svg,heusner:diss}.  Intuitively, they have low heuristic values but are not actually close to a goal and expanding these nodes does not necessarily represent progress.  Unlike many best-first searches, rectangle search does not need to expand all nodes in a local minimum before expanding a node with a higher heuristic value.  Furthermore, we conjecture that rectangle's similarity to monotonic beam search may aid in retaining diverse nodes during the search, preventing a large local minimum from displacing all other nodes in the beam and dominating the search.

Although it is based on breadth-first search, rectangle search's behavior has similarities to depth-first approaches.  When run with a large \aspect\ value (e.g., 500 is evaluated below), rectangle search's behavior is similar to limited-discrepancy search \cite{korf:ild}.  First, it repeatedly expands the best-looking child of the previously expanded node, forming a deep but thin probe into the state space.  This is similar to hill-climbing \cite{hoffmann:ff} and depth-first search.  If the heuristic is accurate or goals are plentiful, this will quickly find a goal.  When a goal is found, its cost establishes an $f$ pruning bound that limits the depth of subsequent search and forces the algorithm to widen more rapidly.  Rectangle search widens by expanding nodes at all previous depth levels, allowing it, for example, to escape a deep local minimum in $d$ without expanding all nodes in the minimum.  This is similar in spirit to limited-discrepancy search, a depth-first-search-based method in which alternatives to the child preferred by a node ordering heuristic are explored at every level of a tree, except that, by virtue of its {\em openlists}, rectangle search has the flexibility to explore alternatives anywhere in a depth layer.

\section{Experimental Evaluation}

To better understand rectangle search's behavior, we implemented it and other algorithms in C++ \footnote[1]{Code available at \url{https://github.com/snlemons/search}.} and compared their anytime behavior on seven classic search benchmarks (and for most, multiple variations in cost or action model).  All algorithms were given a 7.5GB memory limit and a 5 minute time limit on a 2.6 GHz Intel Xeon E5-2630v3.

\subsection{Relevant Comparators}

In addition to best-first algorithms like ARA* and AEES, we examined depth-first algorithms that can give anytime results. DFS* \cite{vempaty:dfvbf} is an anytime algorithm based on iterative deepening and depth-first branch-and-bound. It performs a depth-first search with an increasing cost bound (beginning with $h(start)$ and doubling at each iteration) until it finds an initial solution. It then performs a branch-and-bound search with the incumbent solution's cost as the new bound. We implemented DFS* both with and without child ordering.

We also implemented a similar variation of Improved Limited Discrepancy Search (ILDS) \cite{korf:ild} which we call ILDS*. Because ILDS requires a depth bound, ILDS* begins with a depth bound equal to $d(start)$ and performs ILDS with an increasing number of discrepancies until all nodes within the depth bound have been explored. When a goal is found, it is retained as an incumbent. The search continues until no node with $f(n) \leq g(incumbent)$ is pruned based on the depth bound or discrepancy bound.

Another anytime algorithm with similarities to rectangle search is Complete Anytime Beam Search (CABS) \cite{zhang:cab}. As adapted by \citeauthor{libralesso:tssop} \citeyearpar{libralesso:tssop}, CABS performs a sequence of increasing-width beam searches, doubling the beam width at each iteration, and retaining the best solution so far as an incumbent. It terminates when no node excluded from the beam has $f(n) \leq g(incumbent)$. While CABS has been primarily used in depth-bounded problems, many of our domains do not have fixed depths and therefore our implementation of CABS uses a closed list.

To keep our plots clear, we only include the competitive algorithms: ARA*, AEES, CABS, and rectangle search.  Plots for all algorithms in all domains are given by \citet{lemons:tar}. We ran ARA* in several previously-proposed configurations: initial weights of 10 and 2.5 with a decrement of 0.02 \cite{likhachev:aap} and a weight schedule of 5, 3, 2, 1.5, 1 \cite{thayer:aees}.  Rectangle search was tested with \aspect\ values of 1 and 500 (selected through preliminary tests).

\newcommand{\fixsubcapspace}{\vspace{-12pt}}

\begin{figure*}[t!]
\centering

\begin{subfigure}{\plotwidth{}}
\incfig{figures/plots/tiles/anytime-quality-unit}
\fixsubcapspace
\caption{15-puzzle (unit cost)}
\label{fig:tiles-quality-unit}
\end{subfigure}
\begin{subfigure}{\plotwidth{}}
\incfig{figures/plots/tiles/anytime-quality-inverse}
\fixsubcapspace
\caption{15-puzzle (inverse cost)}
\label{fig:tiles-quality-inverse}
\end{subfigure}
\begin{subfigure}{\plotwidth{}}
\incfig{figures/plots/plat2d/anytime-quality-unit}
\fixsubcapspace
\caption{platform game}
\label{fig:plat2d-quality-unit}
\end{subfigure}

\vspace{5pt}

\begin{subfigure}{\plotwidth{}}
\incfig{figures/plots/100pancake/anytime-quality-unit}
\fixsubcapspace
\caption{100 pancake (unit cost)}
\label{fig:100pancake-quality-unit}
\end{subfigure}
\begin{subfigure}{\plotwidth{}}
\incfig{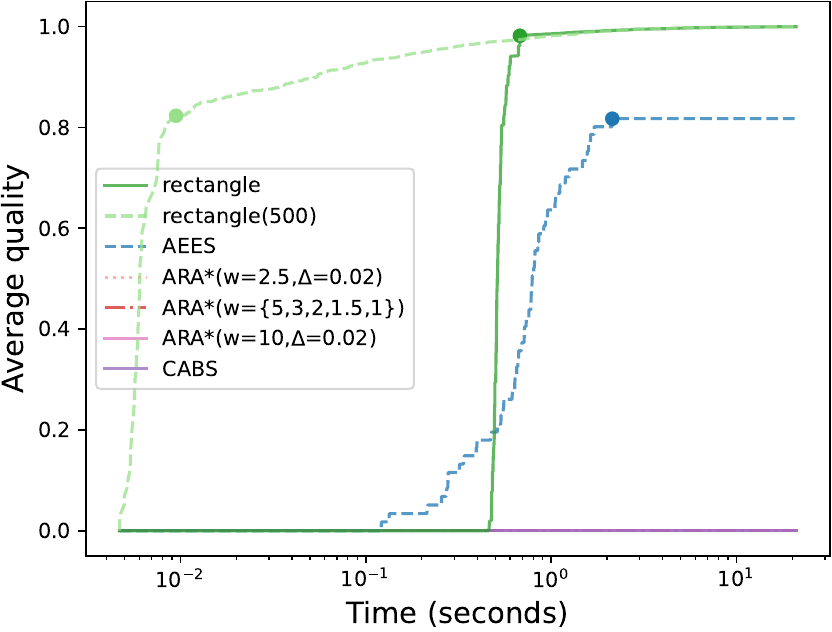}
\fixsubcapspace
\caption{100 pancake (heavy cost)}
\label{fig:100pancake-quality-heavy}
\end{subfigure}
\begin{subfigure}{\plotwidth{}}
\incfig{figures/plots/64room/anytime-quality-octile}
\fixsubcapspace
\caption{grid pathfinding, rooms}
\label{fig:64room-quality-octile}
\end{subfigure}

\vspace{5pt}

\begin{subfigure}{\plotwidth{}}
\incfig{figures/plots/vacuum/anytime-quality-unit}
\fixsubcapspace
\caption{vacuum, 10 dirts (unit cost)}
\label{fig:vacuum-quality-unit}
\end{subfigure}
\begin{subfigure}{\plotwidth{}}
\incfig{figures/plots/vacuum/anytime-quality-heavy}
\fixsubcapspace
\caption{vacuum, 10 dirts (heavy cost)}
\label{fig:vacuum-quality-heavy}
\end{subfigure}
\begin{subfigure}{\plotwidth{}}
\incfig{figures/plots/24tiles/rectangle-beam-quality-unit}
\fixsubcapspace
\caption{24-puzzle (unit cost)}
\label{fig:24tiles-beam-quality-unit}
\end{subfigure}

\caption{Anytime algorithms (\ref{fig:tiles-quality-unit}-\ref{fig:64room-quality-octile}, legend in \ref{fig:100pancake-quality-heavy}).  Rectangle (anytime) versus fixed-width bead (one shot, \ref{fig:24tiles-beam-quality-unit}).}
\label{fig:plots}
\end{figure*}

\subsection{The Sliding Tile Puzzle}

Six different cost models of sliding tiles were used: unit cost; heavy cost, where moving tile number $t$ costs $t$; sqrt cost, moving $t$ costs $\sqrt{t}$; inverse cost, $1/t$; reverse cost, $\#spaces - t$, and reverse inverse cost, $1/(\#spaces-t)$. 
We tested both the 15-puzzle (4x4) and the 24-puzzle (5x5), using a cost-weighted Manhattan distance heuristic.

Selected results are shown in Figures \ref{fig:tiles-quality-unit} and \ref{fig:tiles-quality-inverse} (the legend is in \ref{fig:100pancake-quality-heavy}). We plot anytime  performance over time, including before all algorithms have solved all instances. The primary metric we use is average {\em quality}: the cost of the best known solution (optimal if known, otherwise the best solution cost given by any of the algorithms at any time), divided by the cost of the incumbent solution (or $\infty$ if no solution has been found yet.) Quality can range from 0 for unsolved instances to a maximum of 1. A dot is included to show when an algorithm has solved all instances. Before this dot, an increase in average quality could be due to either a new instance being solved or a better solution being found for an instance that was already solved.

In both the 15- and 24-puzzle, and across all cost models, rectangle(1) achieves full coverage before other algorithms and provides better solution cost from that point onward. More specifically, both configurations of rectangle reach full coverage first, with rectangle(1) having lower solution costs at that point and rectangle(500) quickly catching up. In particular, in the inverse cost model rectangle(1), rectangle(500), and AEES are the only algorithms to solve all instances within our time bound.

\subsection{Blocks World}

We tested on 100 random blocks world instances with 20 blocks. We included two different action models: `blocks world', where blocks are directly moved to a stack as an action, and `deep blocks world' \cite{lelis:sts}, where picking up and putting down blocks each use an action, leading to longer solutions. The heuristic used was the number of blocks out of place (any block which is not in a sequence of blocks from the table up which matches the goal state). This heuristic value is doubled for deep blocks.  For both variants, rectangle search reaches full coverage earliest and provides the best solution costs after.

\subsection{The Pancake Problem}

Two cost models were used: unit cost, and heavy cost \cite{hatem:2014}, in which each pancake is given an ID number from 1 through $N$ (the number of pancakes), and the cost of a flip is the ID of the pancake above the spatula. 50, 70, and 100 pancakes were used, with the gap heuristic \cite{helmert:2010}, with modifications to include cost per pancake in the heavy cost model.

Selected results are shown in Figures \ref{fig:100pancake-quality-unit} and \ref{fig:100pancake-quality-heavy}. In all the sizes of unit pancake, rectangle(1) lags behind the other algorithms in terms of solving. Once it has solved the problems, it provides better solution cost than the other algorithms, but this is not desirable behavior for an anytime algorithm. However, increasing \aspect\ improves the performance: rectangle(500) solves unit instances at around the same time as the ARA* variants, and provides better solution cost than all algorithms shortly after. CABS reaches full coverage slightly earlier than other algorithms, but lacks in terms of quality for much of the time. In heavy cost pancakes, only rectangle(500), rectangle(1) and AEES solve problems reliably, with rectangle(500) being a clear winner.

\subsection{Vacuum World}

In vacuum world \cite{russell:aima} a robot must vacuum up dirt in a grid world. We tested both unit costs and heavy costs, where the cost of movement is equal to the number of dirts that the robot has already vacuumed \cite{thayer:bss}. The heuristic sums the number of remaining dirts, the edges of the minimum spanning tree (MST) of the Manhattan distances among the dirts, and the minimum Manhattan distance from the agent to one of the dirts. For the heavy cost model, the distance components were multiplied by the number of dirts cleaned so far. We tested three problem sizes: 200$\times$200 grid with 10 dirts, 500$\times$500 grid with 20 dirts, and 500$\times$500 grid with 60 dirts.

Selected results are shown in Figures \ref{fig:vacuum-quality-unit} and \ref{fig:vacuum-quality-heavy}. ARA* leads in performance in smaller unit cost problems, though rectangle is competitive, reaching full coverage only slightly after ARA*. Rectangle also improves quality beyond ARA* shortly after reaching full coverage. As problem size increases, ARA* loses its lead. In the 500x500 problems with 60 dirts rectangle(1) lags, but rectangle(500) dominates in terms of coverage. CABS provides better quality for some of the time in these largest vacuum problems. In heavy cost, only rectangle and AEES are able to solve reliably in all problem sizes tested, with rectangle(500) clearly superior.

\subsection{Grid Pathfinding}

For 2D grid pathfinding problems we used 5000x5000 maps with random obstacles (grid-random) with 4-way movement (unit cost and life cost \cite{ruml:bfu}), and structured maps of room64 and orz100d \cite{sturtevant:bgp} with 8-way movement. For 4-way movement, we use cost-adjusted Manhattan distance; for 8-way movement, octile distance.

Results in the rooms map are shown in Figure \ref{fig:64room-quality-octile}. Overall, ARA* gives best results in all of the maps used. Rectangle provides comparable solution costs for the instances which it solves, but tends to lag significantly on coverage, even more than AEES and CABS. The performance of rectangle(1) is relatively consistent across the maps tested. However, the performance of rectangle(500) is worse in orz100d than in random grids, and worse in 64 room maps than in orz100d. We will return to this phenomenon below.

\subsection{A Platform Game}

The platform domain \cite{burns:hsw} is based on 2D video games in which a character jumps between platforms to reach a goal location. The heuristic uses visibility graphs and actions have unit cost. Not all actions in this domain are reversible, as gravity has an effect on available movements while jumping. We used 100 problem instances defined as 50x50 grids. Results are shown in Figure \ref{fig:plat2d-quality-unit}. Most algorithms give comparable performance on this domain, though ARA* provides slightly better quality and is first to reach full coverage. The other algorithms approach full coverage but do not quite reach it in the time given. It is noteworthy that in this domain, rectangle(500) performs worse than rectangle(1), indicating that reconsideration of earlier steps may matter more than quick, deep exploration.

\subsection{Dynamic Traffic}

The traffic domain \cite{kiesel:agq} has both fixed and moving obstacles that the agent must navigate around to reach a goal. The agent knows the future positions of all obstacles. When the agent collides with an obstacle, it is returned to the start location. This search space is directed, because the choices available to the agent change as obstacles move and the same configuration of obstacles may never occur again. We used 100 grid maps of size 100x100 with 5000 total obstacles.

In this domain, all algorithms give comparable performance, except rectangle(1), which solves much later than all of the other algorithms. ARA* reaches full coverage slightly before the other algorithms, but rectangle(500) and CABS improve on quality more quickly than the other algorithms after reaching full coverage.

\subsection{Strengths and Weaknesses of Rectangle Search}

\begin{table}[t]
\centering
\begin{tabular}{l|ll}
& {\bf worse} & {\bf better} \\ \hline
rect(500) & non-random grids    & pan, hvy vac \\
AEES     & 24, pan, bw & \\
rect(1)  & pan, vac, traffic, grid & tiles \\
CABS     & tiles, hvy pan, bw, grid &  \\
ARA* 5   & 15 inv, 24, pan, bw, hvy vac & plat, room, orz \\
ARA* 2.5 & 15 inv, 24, pan, bw, hvy vac & plat, room, orz \\
ARA* 10  & tiles, 24, pan, bw, hvy vac & plat, room, orz \\
\end{tabular}
\caption{Where each method is worse or better than others.}
\label{tab:bests-fails}
\end{table}

Table~\ref{tab:bests-fails} summarizes the domains in which each algorithm performs exceptionally well or poorly relative to the others.  When a general domain like `tiles' is listed instead of a specific variant like `hvy vac', this indicates that most variants gave similar relative performance.  Rectangle search provides competitive results across a wide variety of domains. While it is not the best algorithm in all domains, rectangle(1) gives superior performance in all sliding tiles configurations, and rectangle(500) in pancakes and the heavy vacuum problems.  Perhaps more importantly, rectangle(500) is not entirely inferior to other algorithms in any domain. Its only clear failures are on the non-random map variants of grid pathfinding (orz100d and 64room) --- we discuss this below. On the other hand, all other algorithms tested are markedly inferior (either failing to solve or providing poor quality) in at least two of the domains tested. For these reasons, rectangle(500) appears promising as a method of first resort on new problem domains.

In cases where the heuristic gives extremely accurate guidance, the additional exploration of rectangle(1) at higher depth levels will be unnecessary and increasing \aspect\ can lead to rectangle finding solutions faster. However, not all the settings in which rectangle search struggles to compete are ones in which the heuristic is extremely accurate, and adjusting \aspect\ does not adequately solve performance problems in all domains.

\def \cupwidth {0.8in}

\begin{figure*}[t]
\centering
\begin{adjustbox}{angle=90}
\begin{tabular}{cc}
    rectangle & GBFS\\
    \hline
    \multicolumn{2}{c}{ goal outside} \\
    \includegraphics[width=\cupwidth{}]{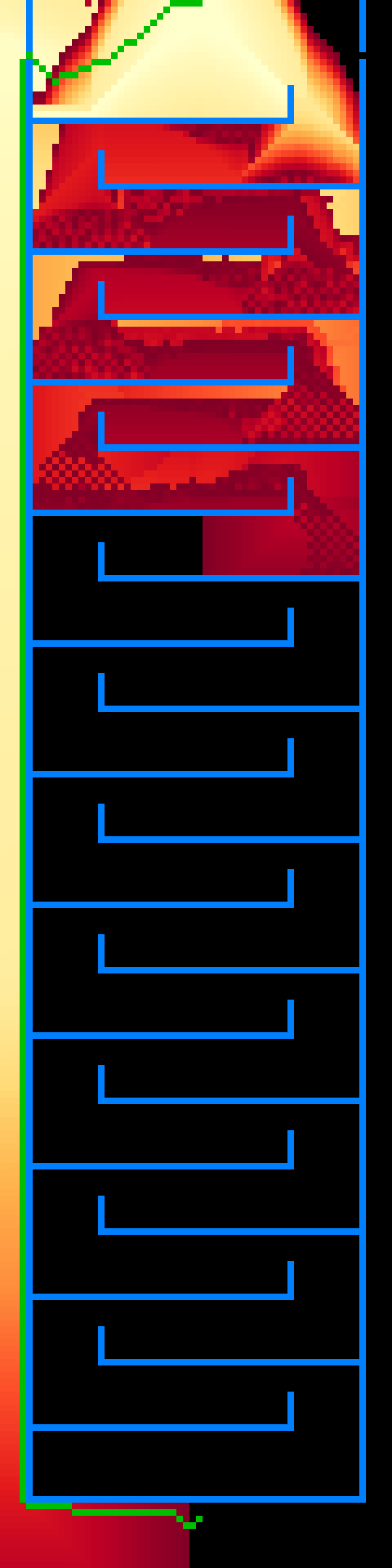} &
    \includegraphics[width=\cupwidth{}]{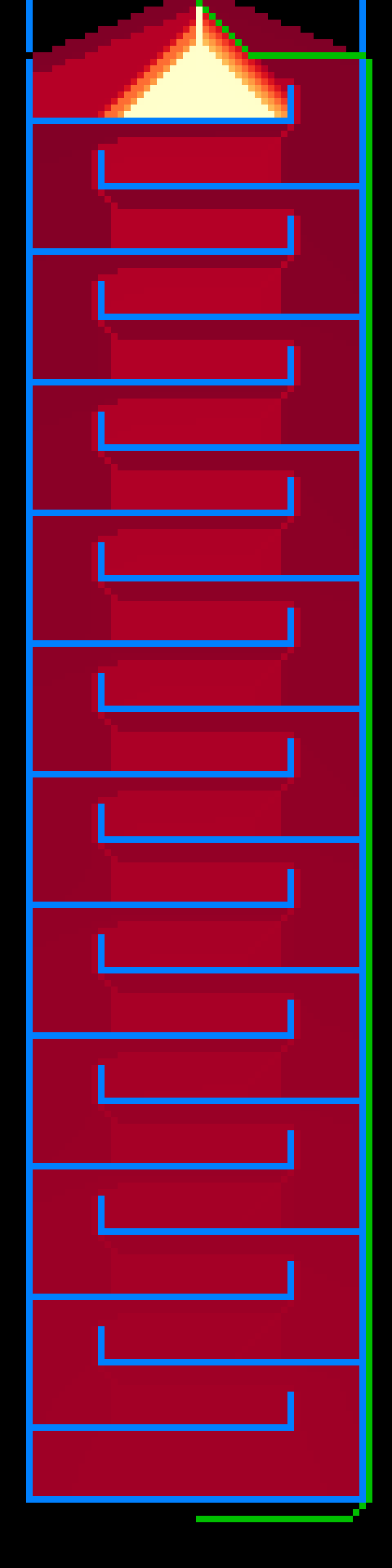}\\
    \hline
   \multicolumn{2}{c}{ goal inside} \\
    \includegraphics[width=\cupwidth{}]{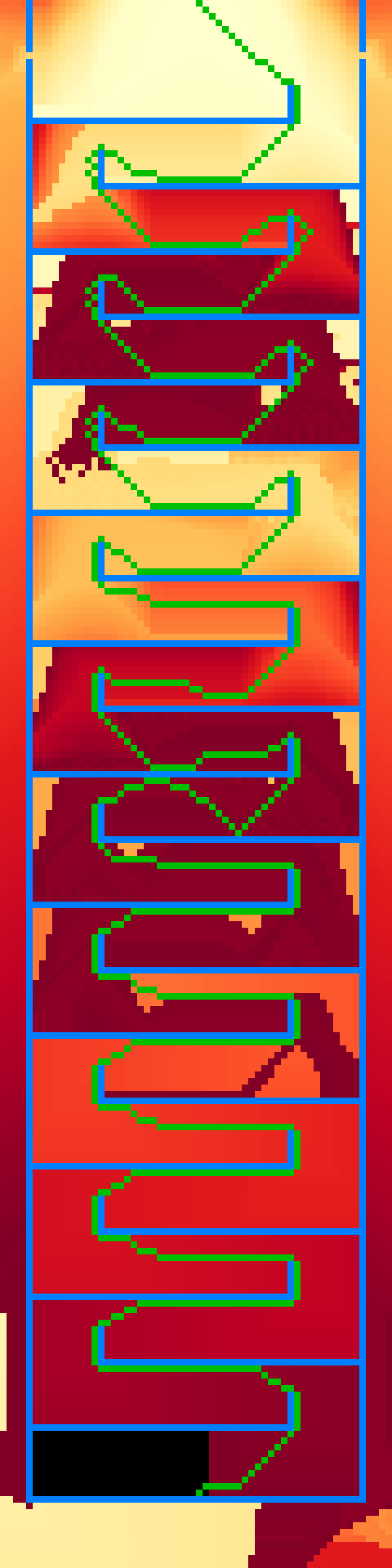} &
    \includegraphics[width=\cupwidth{}]{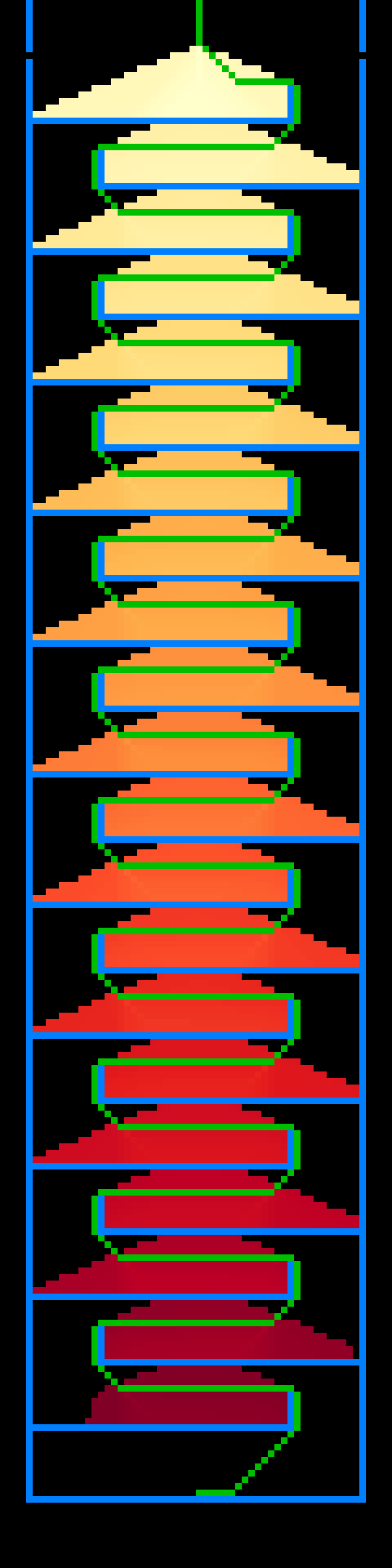} 
\end{tabular}
\end{adjustbox}
\caption{Good and bad scenarios for rectangle search vs best-first search.}
\label{fig:multi-cups}
\end{figure*}


To understand how rectangle search can perform better than best-first searches and when it can be expected to perform worse, we focus on heuristic error---in particular, local minima. 
Consider the grid pathfinding problem shown in Figure \ref{fig:multi-cups}, with a large room containing a slalom of cup-like obstacles and small exits from the room near the start.  States are colored green along the solution path and a gradient from yellow to red shows expansion order, with yellow denoting earliest and dark red latest. If the goal location is outside the room (left panels), most of the room forms one large local minimum. The best path is to move through one of the nearby exits but the heuristic tempts the search inward. A best-first search (GBFS in the figure, as it is the limiting case of ARA* and AEES) will explore the entire local minimum until it has exhausted all states that appear better than the exits. Rectangle search (shown here with \aspect=1), however, widens its exploration of the early regions of the search space and finds a path through an exit and to the goal quickly without exploring most of the local minimum.
On the other hand, when the goal is inside the obstacle (right panels), the cups form many small local minima and it is best for a search to focus its efforts inside the room. GBFS quickly fills each minimum and reaches the goal, while rectangle's skeptical widening causes it to search both inside and outside the room, even re-expanding useless states when it encounters better paths to them.

Rectangle search assumes that it is helpful to re-consider early decisions. This allows it to escape problematic regions of the search space such as a large minimum. But in problems where progress is periodically made and it is better to commit to paths than to reconsider early choices, rectangle search will waste effort. Therefore, rectangle search is strongest in problems where there are deep local minima that it can avoid searching, but is not the best choice when local minima are small and the search can profit from committing to nodes with low heuristic values. For example, in the 64room map of grid pathfinding (Figure \ref{fig:64room-quality-octile}), the rooms create a sequence of local minima, causing rectangle search to waste effort exploring adjacent rooms near the start while ARA* and AEES can forget about a room once they escape it. In contrast, the tiny local minima created by randomly placed single-cell obstacles are easy to escape by generating a few successors (considering alternative fringe nodes earlier in the search is not necessary), so rectangle(500) is no worse than other algorithms in such problems.

\subsection{Comparison with Bead Search}

We also performed a comparison of rectangle search with fixed-width bead search to understand how the additional overhead of rectangle search compares to the results obtainable by a well-selected beam width.
In the sliding tile puzzle, rectangle search provides competitive quality to bead at a variety of widths, and is able to continue improving its solution quality where bead cannot.  Figure \ref{fig:24tiles-beam-quality-unit} shows results for the unit cost 24-puzzle. While rectangle(1) is clearly superior in this setting, rectangle(500) still gives comparable results to many of the fixed-width beam searches. Similar results were found for the rest of the sliding tiles, blocks world, and the platform game. In the pancake problem, the smaller vacuum problems, grid pathfinding, and traffic, it is rectangle(500) which performs approximately as well as fixed width bead searches or better, with rectangle(1) doing as well as many but not all of the fixed-width bead searches. In only the larger vacuum problems with unit cost do we see both configurations of rectangle performing significantly more poorly than bead searches.

Overall, rectangle search is able to provide solutions about as quickly and with comparable quality to bead search with fixed widths across most of the domains tested. This means that even in settings where
anytime behavior is not required,
rectangle search can serve as the algorithm of choice.

\section{Discussion}

BULB \cite{furcy:ldb} behaves like regular beam search until a given depth limit is reached, at which point it uses backtracking to continue the search.  In contrast, rectangle search represents a new alternative to conventional beam search and does not require a depth limit.  For huge problems where memory capacity is an issue, it would be interesting to integrate BULB-like backtracking with rectangle search.

We have investigated rectangle search's performance with \aspect=1 and \aspect=500.  Additional research will be necessary to understand how to tune this parameter.  Nonlinear increases of width and depth are also a possibility.
We leave exploration of these variants to future work.


\section{Conclusions}

Rectangle search is an effective anytime algorithm with a simple design.  Unlike previous anytime algorithms, which are based on best-first search, rectangle search is instead based on breadth-first search. It has similarities to both monotonic beam search and depth-first algorithms like ILDS. It enforces exploration at a variety of depths in the search tree, which allows it to escape large local minima.  Rectangle search is often an effective replacement for fixed-width beam searches. Overall, rectangle search's promising performance suggests that suboptimal non-best-first heuristic search deserves further exploration.

\section{Acknowledgments}

We are grateful for support from the NSF-BSF program (via NSF grant 2008594) and Earlham College (via the Lemann Student/Faculty Collaborative Research Fund).

\onecolumn

\section*{Appendix}

This appendix contains all the experimental results, as there was insufficient space to present them in the paper. The results were summarized in the paper, but these plots provide additional detail.

First, in Section~\ref{sec:anytime}, we focus on anytime performance. Section \ref{sec:competitive} presents the results for all algorithms that provide competitive performance on a majority of the domains tested: ARA*, AEES, and rectangle search. In Section \ref{sec:dfs}, we present results for rectangle search versus several depth-first-based anytime approaches: ILDS*, DFS*, and DFS* with child ordering (DFS*-co). Finally, in Section \ref{sec:beam}, we present comprehensive results for the comparison of rectangle search to fixed-width bead search.

\section{Anytime Results} \label{sec:anytime}

We summarize results for anytime algorithms in three different ways. First, we use {\em quality}: the cost of the best known solution (optimal if known, otherwise the best solution cost given by any of the algorithms at any time), divided by the cost of the incumbent solution (or $\infty$ if no solution has been found yet). Quality can range from 0 for unsolved instances to a maximum of 1. A dot is included in the quality plots to show when an algorithm has solved all instances. Before this dot, an increase in quality could be due to either a new instance being solved or a better solution being found for an instance that was already solved. Coverage is the number of instances solved by a given algorithm. Solution cost is also given, though it is averaged across only those instances which all algorithms shown have solved at the time displayed.  Average cost may increase as more difficult problems with costlier solutions are solved by all algorithms shown.  Algorithms that solve fewer than 20\% of total problem instances by the end of the allowed time are omitted from the cost plots.

\subsection{Competitive Algorithms} \label{sec:competitive}

\begin{figure*}[h!]%
\centering%
\includegraphics[width=\plotwidth{}]{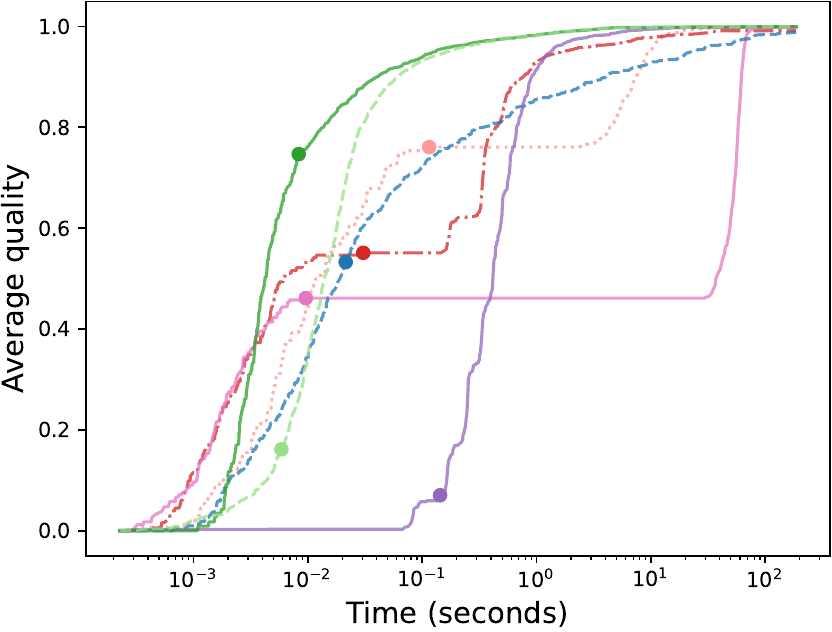}%
\centering%
\includegraphics[width=\plotwidth{}]{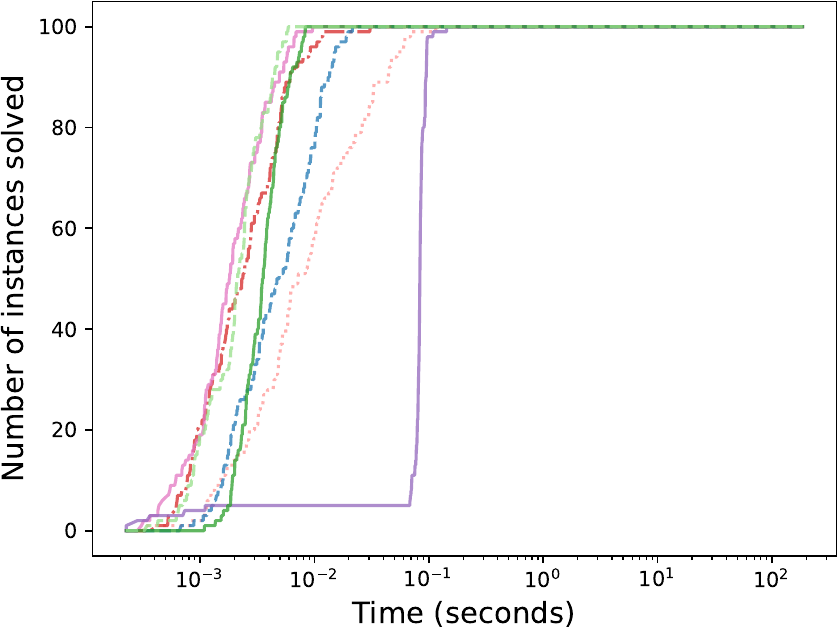}%
\centering%
\includegraphics[width=\plotwidth{}]{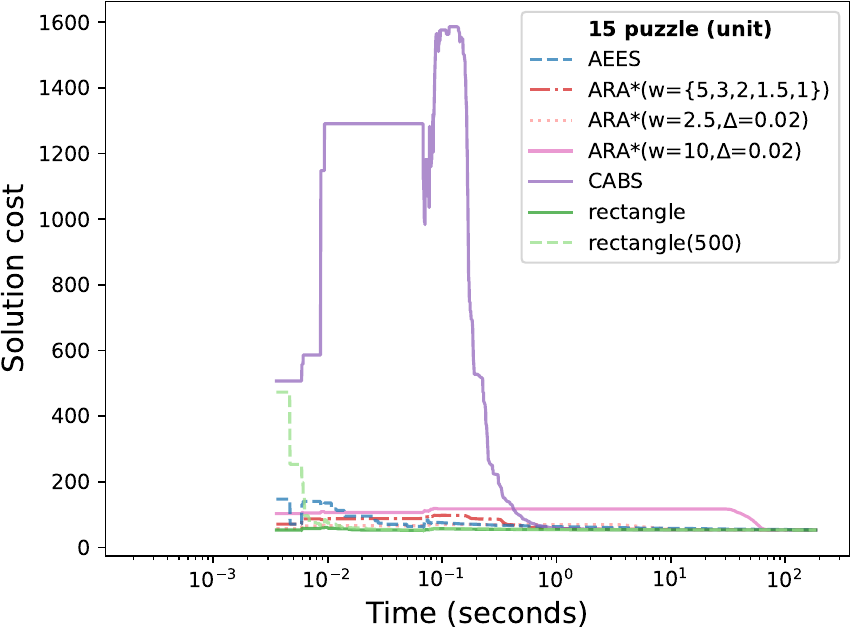}%
\caption{15 puzzle (unit cost)}%
\end{figure*}

\begin{figure*}[h!]%
\centering%
\includegraphics[width=\plotwidth{}]{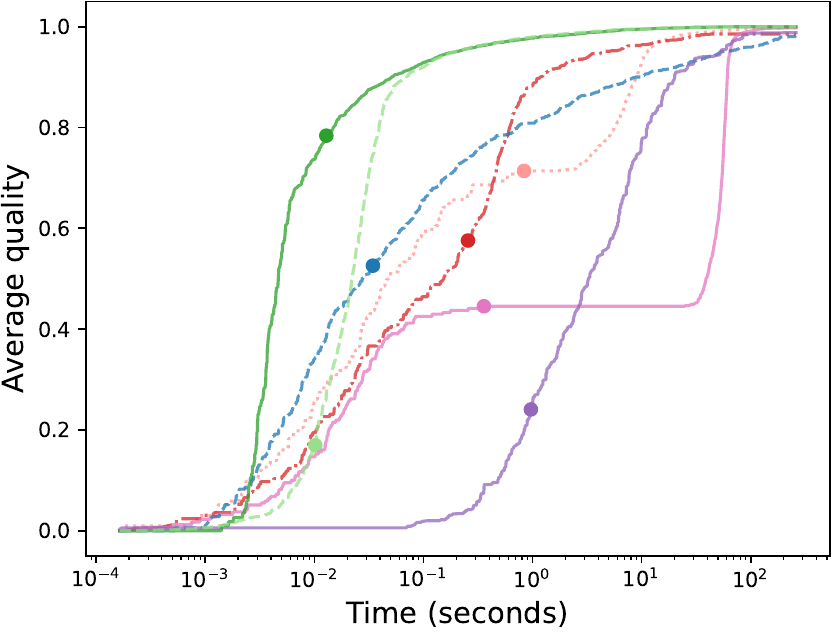}%
\centering%
\includegraphics[width=\plotwidth{}]{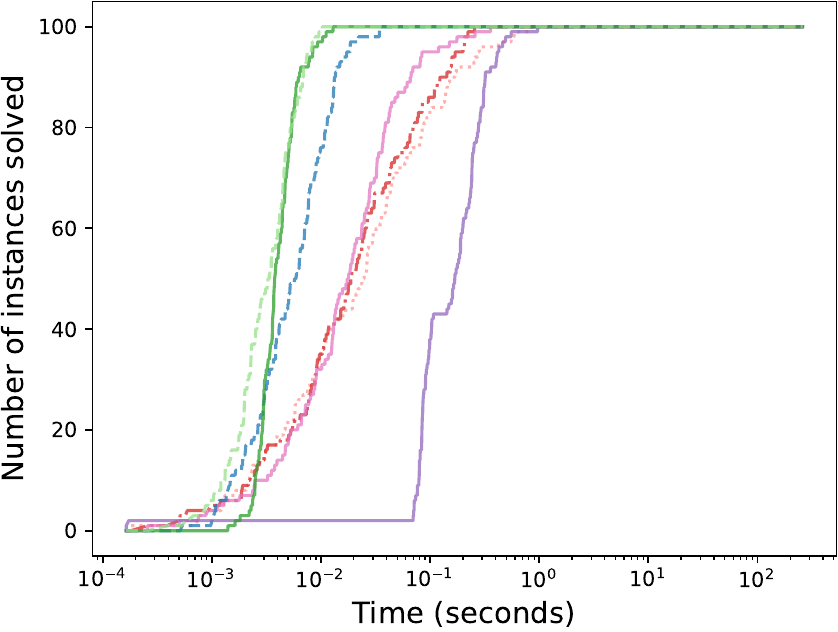}%
\centering%
\includegraphics[width=\plotwidth{}]{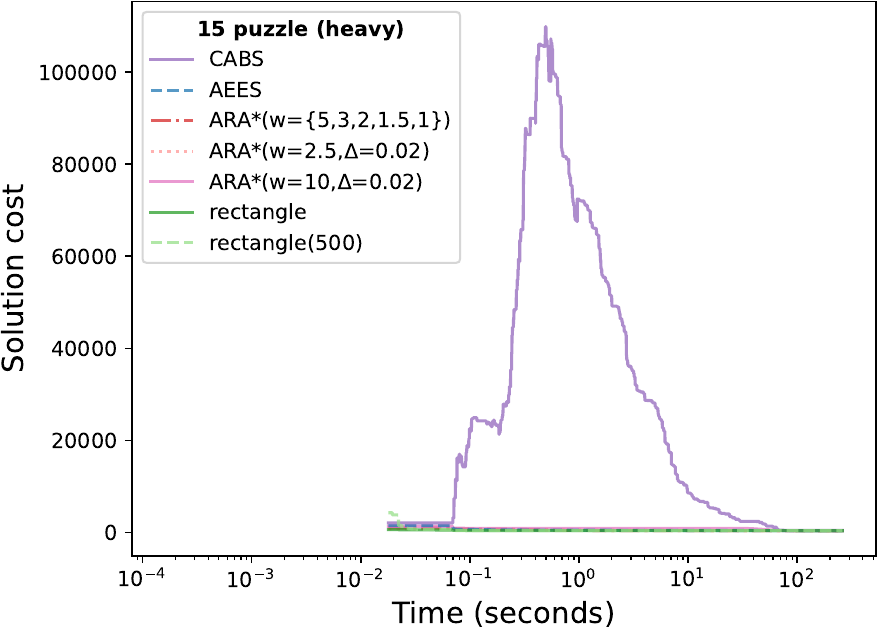}%
\caption{15 puzzle (heavy cost)}%
\end{figure*}

\begin{figure*}[h!]%
\centering%
\includegraphics[width=\plotwidth{}]{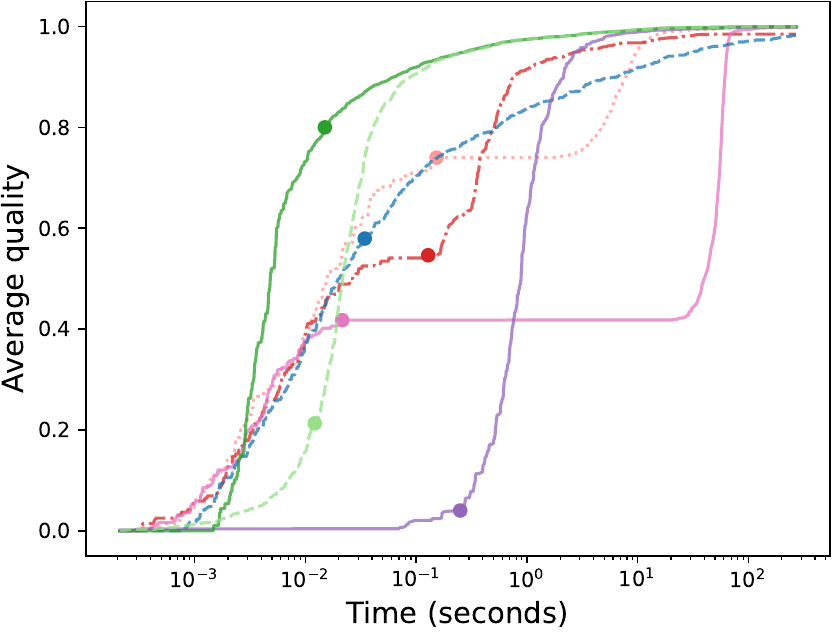}%
\centering%
\includegraphics[width=\plotwidth{}]{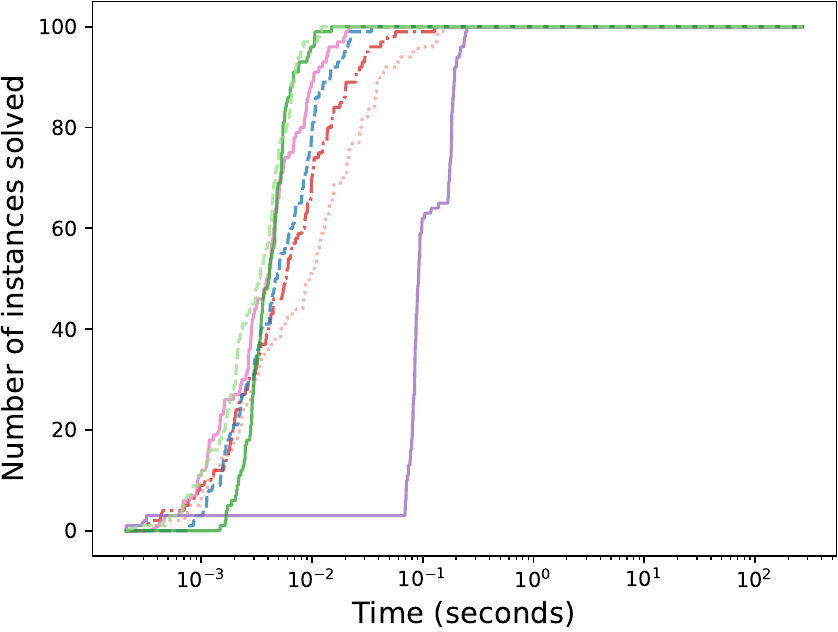}%
\centering%
\includegraphics[width=\plotwidth{}]{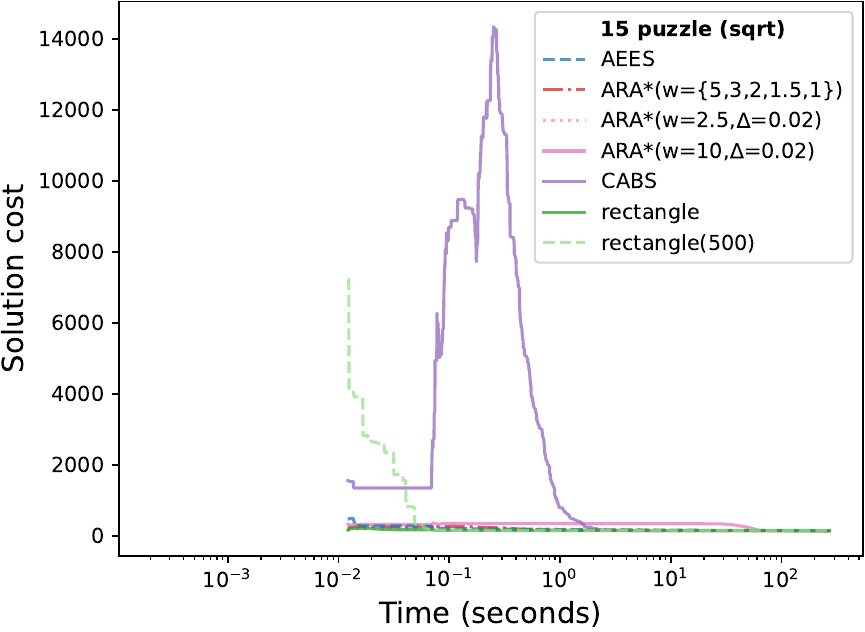}%
\caption{15 puzzle (sqrt cost)}%
\end{figure*}

\begin{figure*}[h!]%
\centering%
\includegraphics[width=\plotwidth{}]{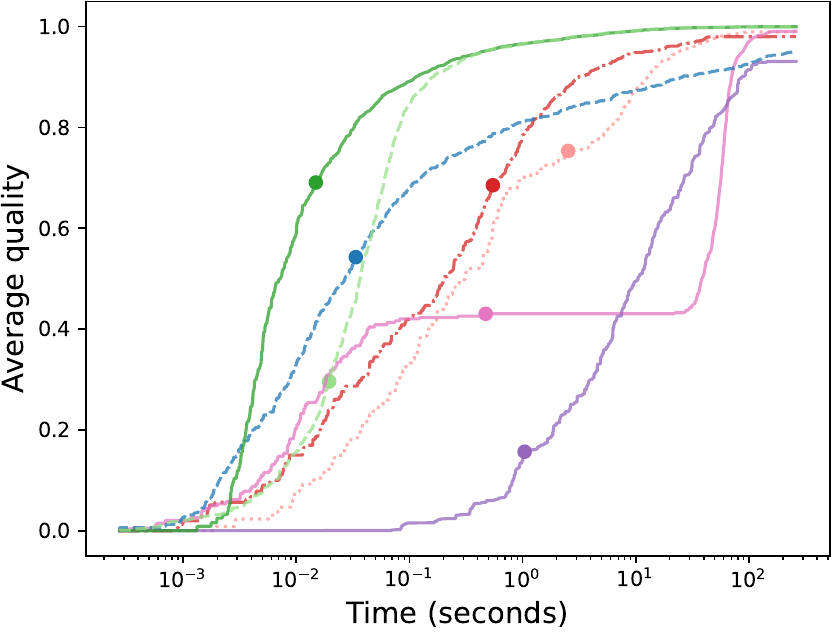}%
\centering%
\includegraphics[width=\plotwidth{}]{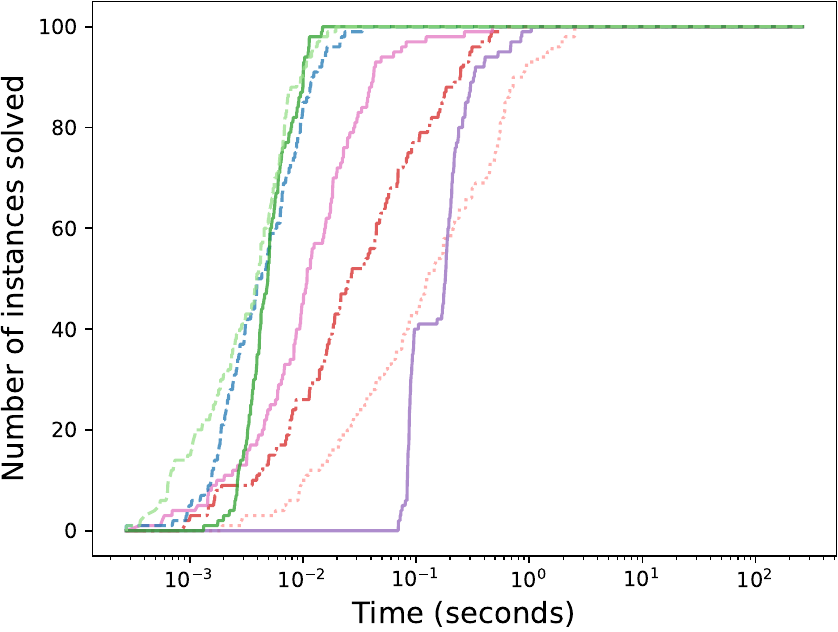}%
\centering%
\includegraphics[width=\plotwidth{}]{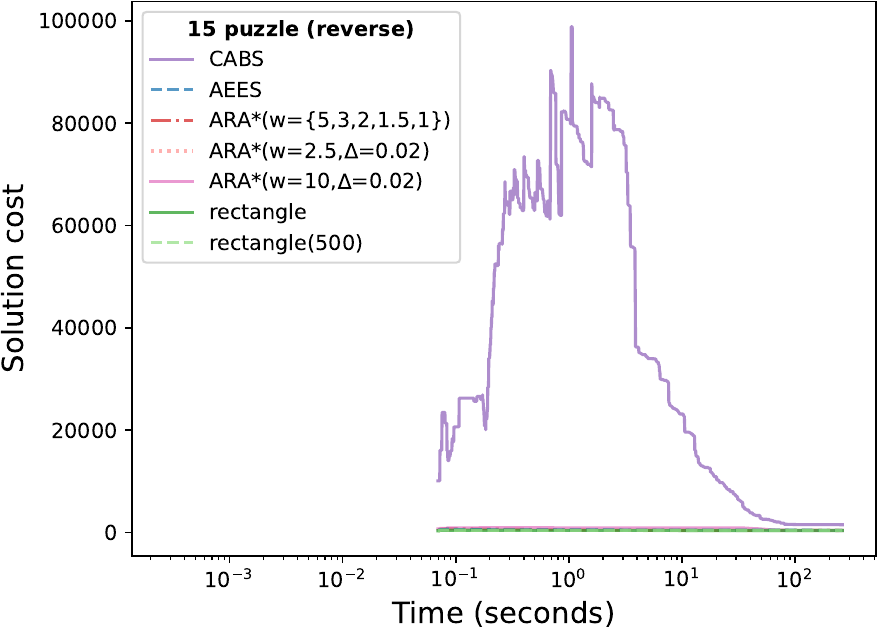}%
\caption{15 puzzle (reverse cost)}%
\end{figure*}

\begin{figure*}[h!]%
\centering%
\includegraphics[width=\plotwidth{}]{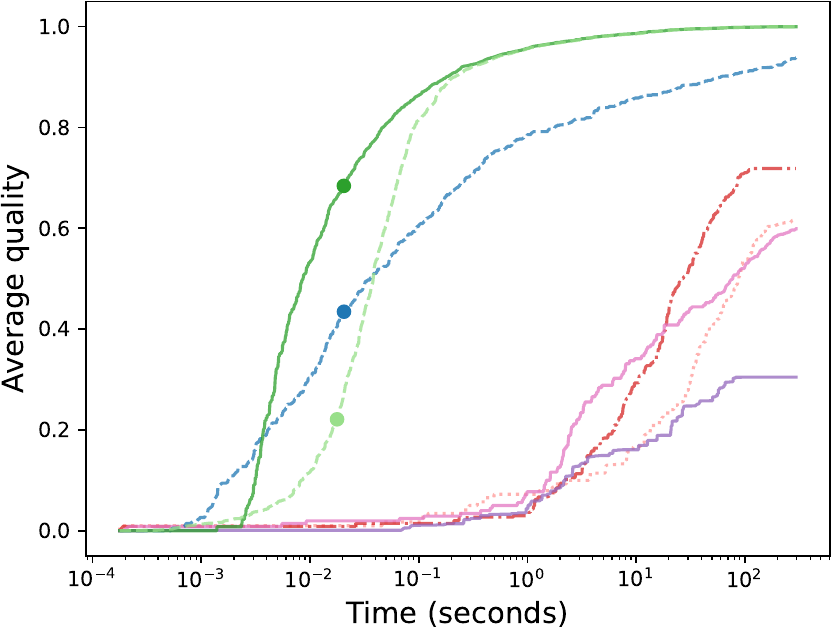}%
\centering%
\includegraphics[width=\plotwidth{}]{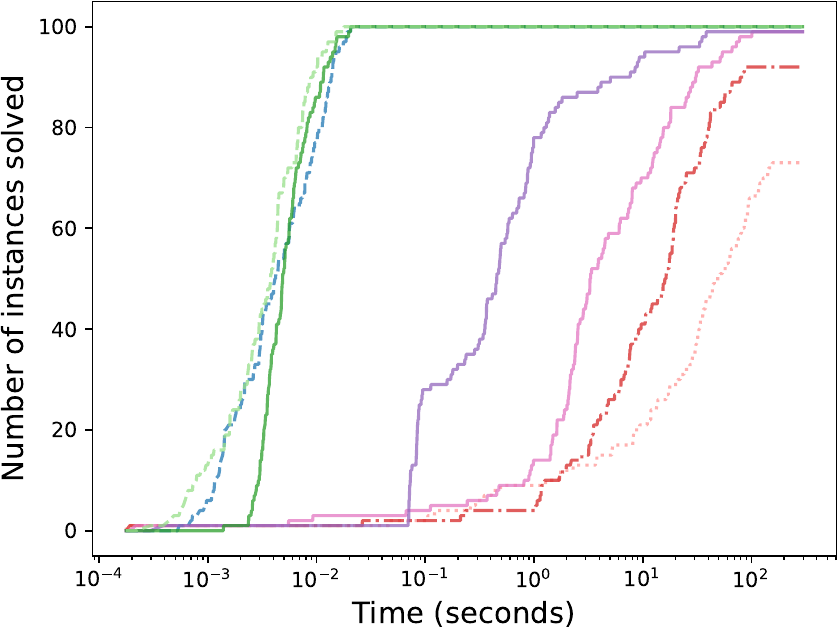}%
\centering%
\includegraphics[width=\plotwidth{}]{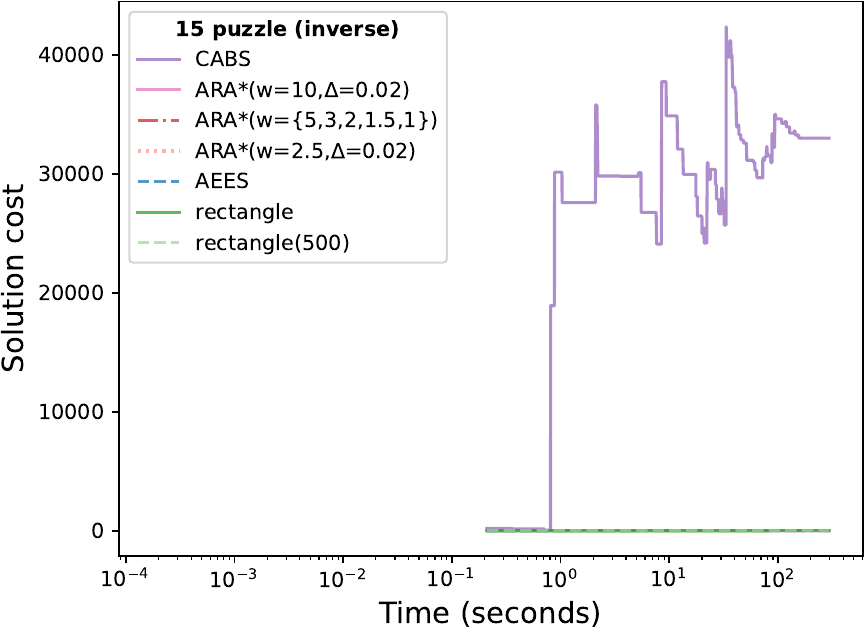}%
\caption{15 puzzle (inverse cost)}%
\end{figure*}

\begin{figure*}[h!]%
\centering%
\includegraphics[width=\plotwidth{}]{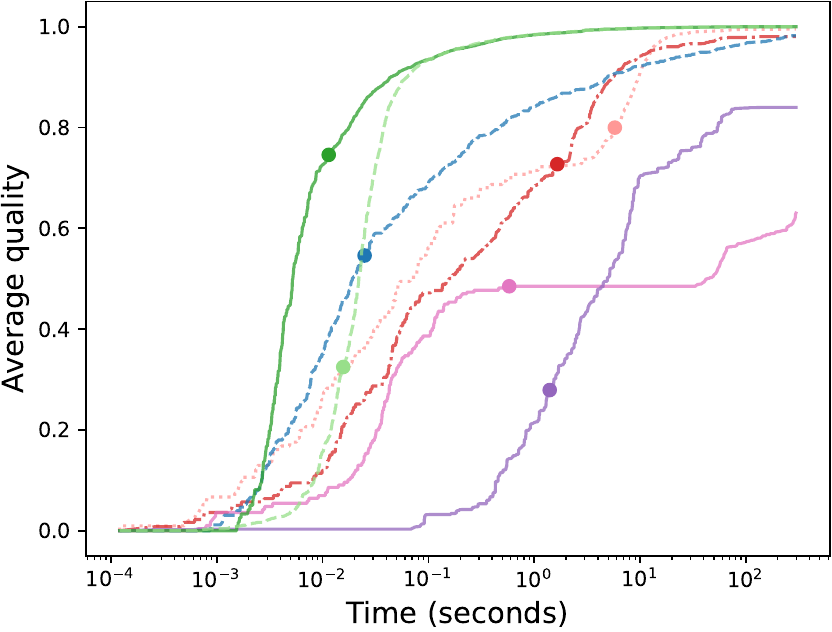}%
\centering%
\includegraphics[width=\plotwidth{}]{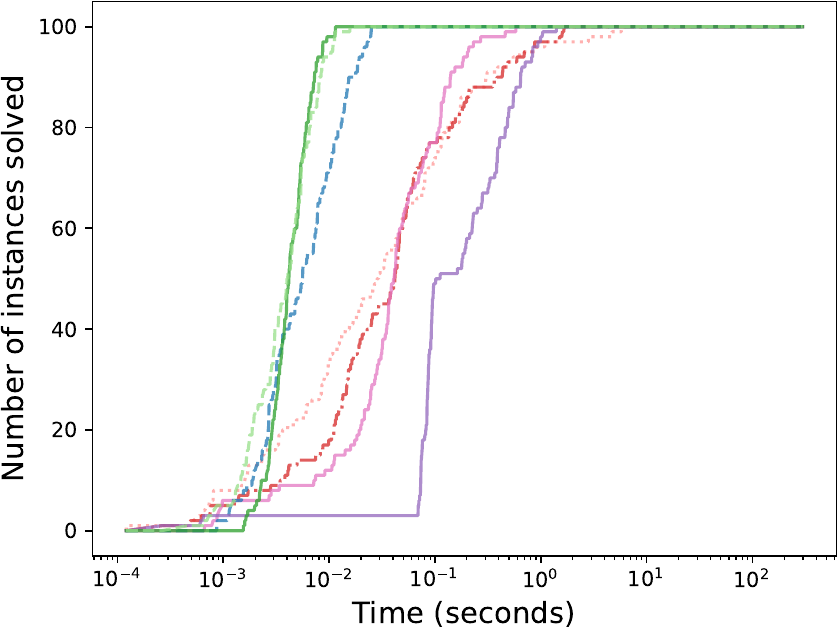}%
\centering%
\includegraphics[width=\plotwidth{}]{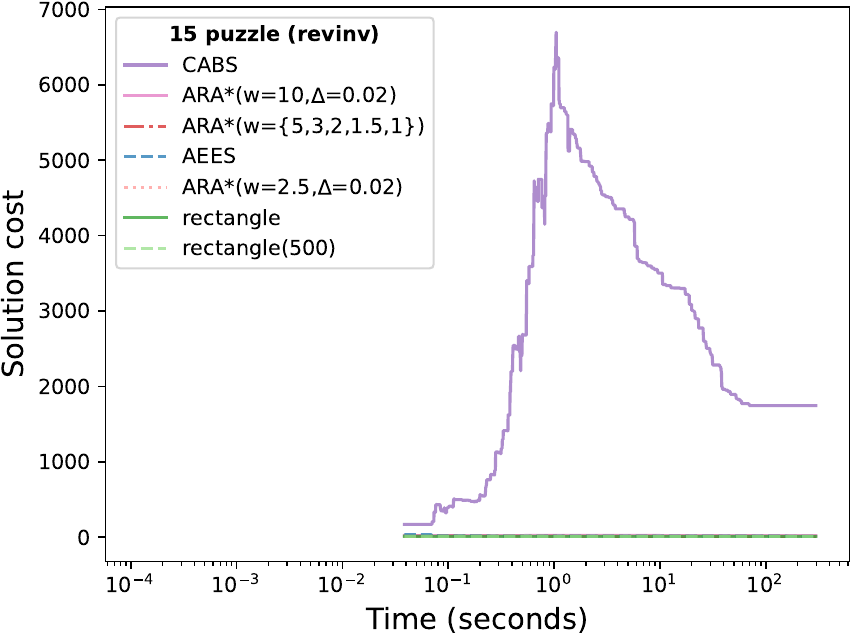}%
\caption{15 puzzle (reverse inverse cost)}%
\end{figure*}

\begin{figure*}[h!]%
\centering%
\includegraphics[width=\plotwidth{}]{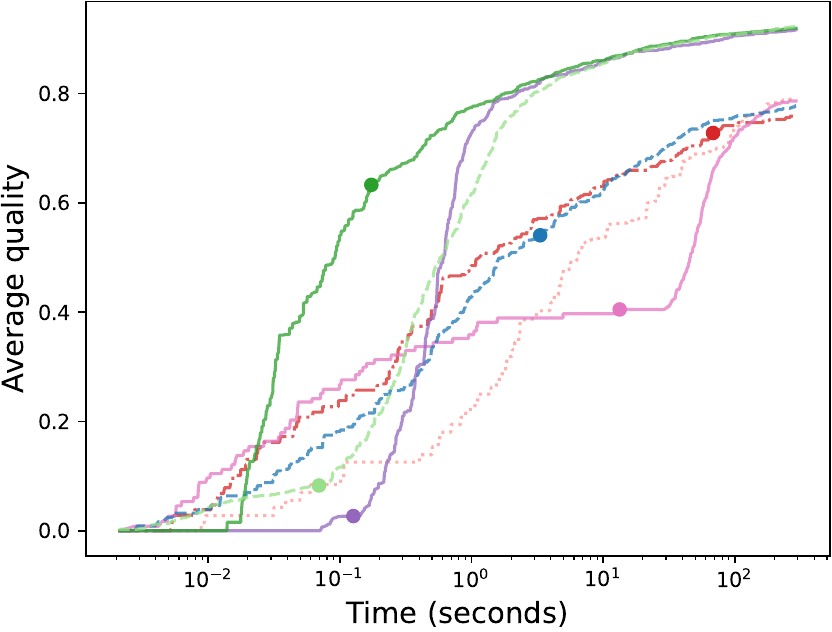}%
\centering%
\includegraphics[width=\plotwidth{}]{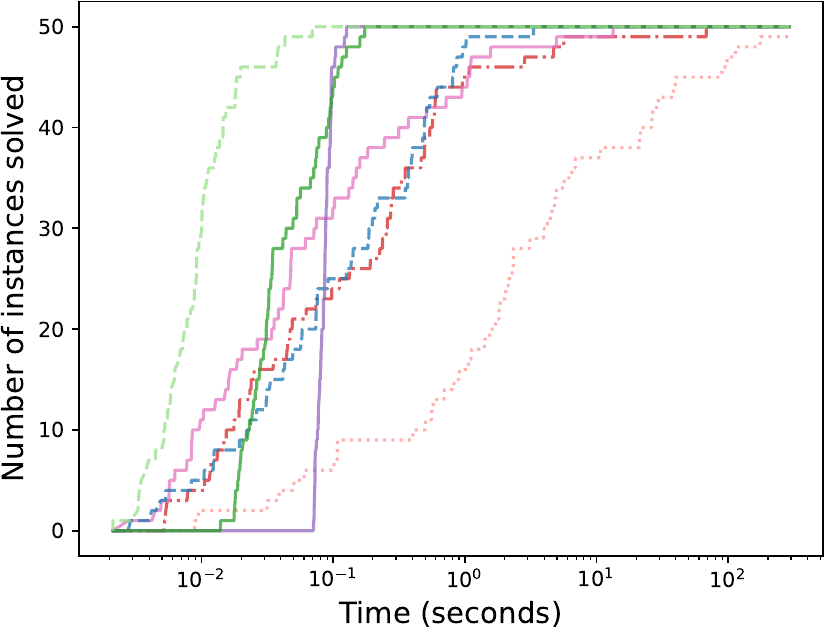}%
\centering%
\includegraphics[width=\plotwidth{}]{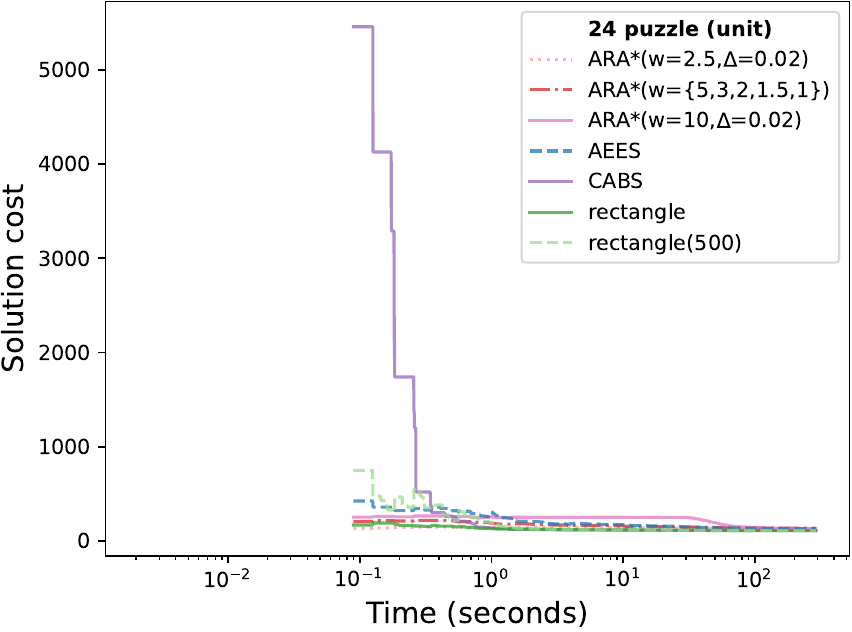}%
\caption{24 puzzle (unit cost)}%
\end{figure*}

\begin{figure*}[h!]%
\centering%
\includegraphics[width=\plotwidth{}]{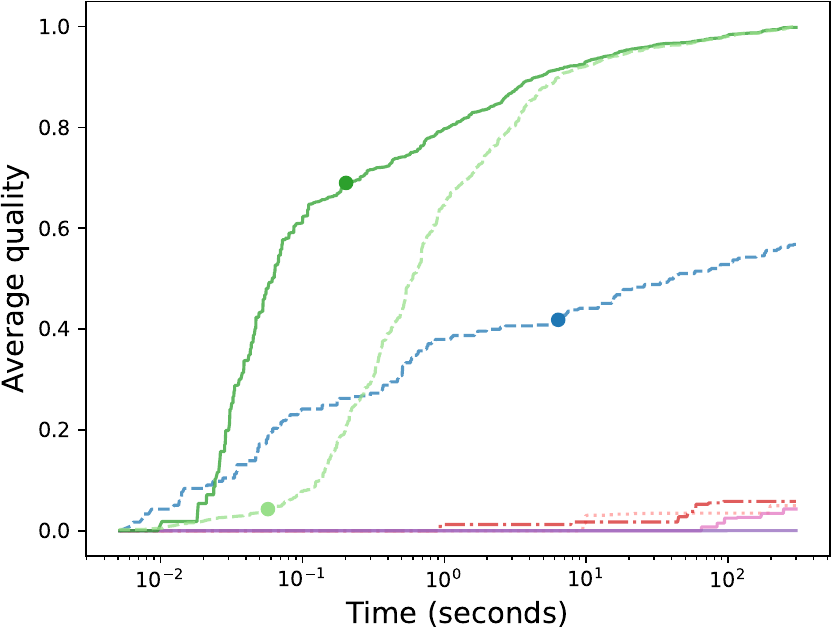}%
\centering%
\includegraphics[width=\plotwidth{}]{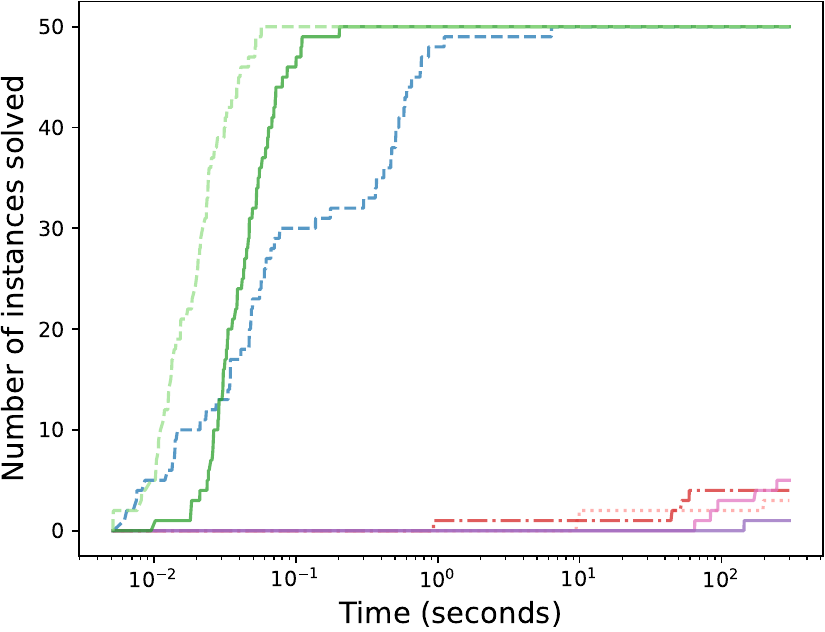}%
\centering%
\includegraphics[width=\plotwidth{}]{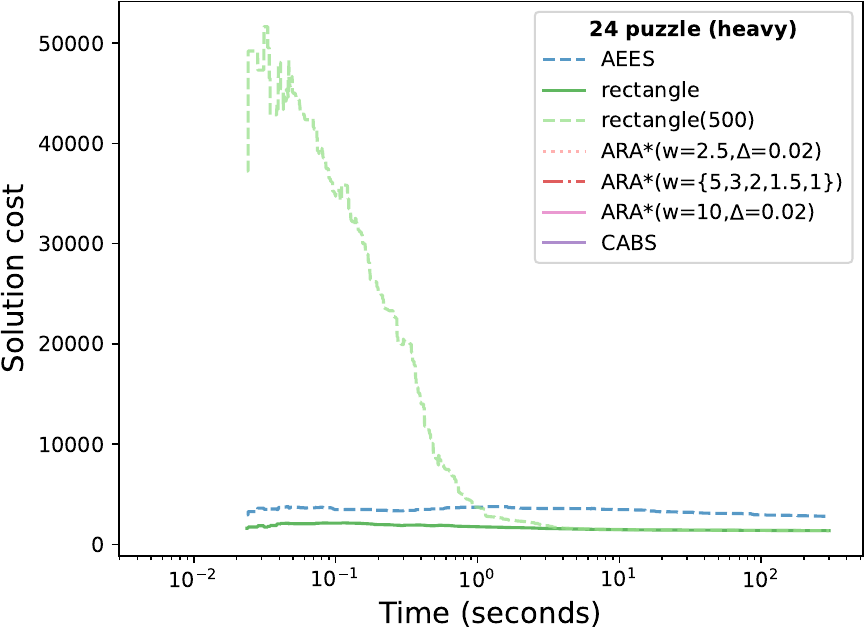}%
\caption{24 puzzle (heavy cost)}%
\end{figure*}

\begin{figure*}[h!]%
\centering%
\includegraphics[width=\plotwidth{}]{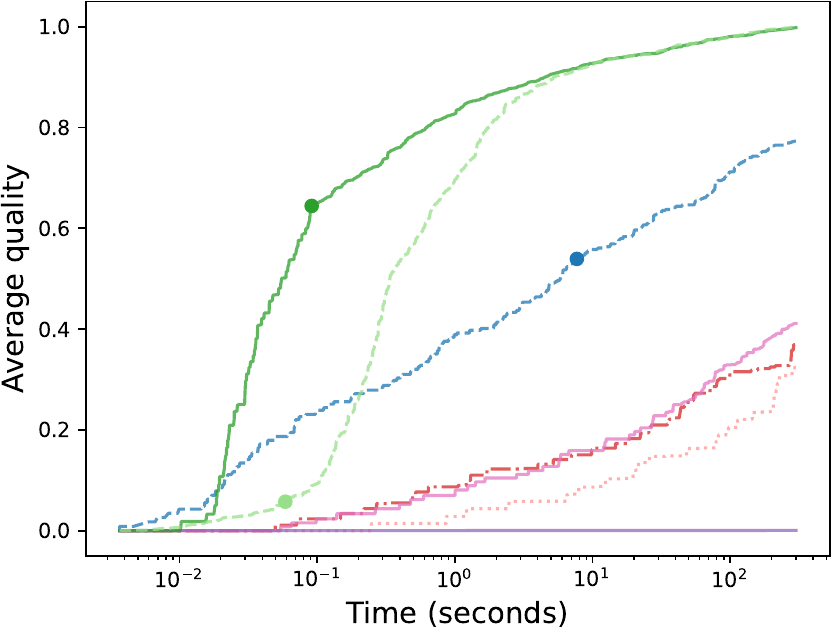}%
\centering%
\includegraphics[width=\plotwidth{}]{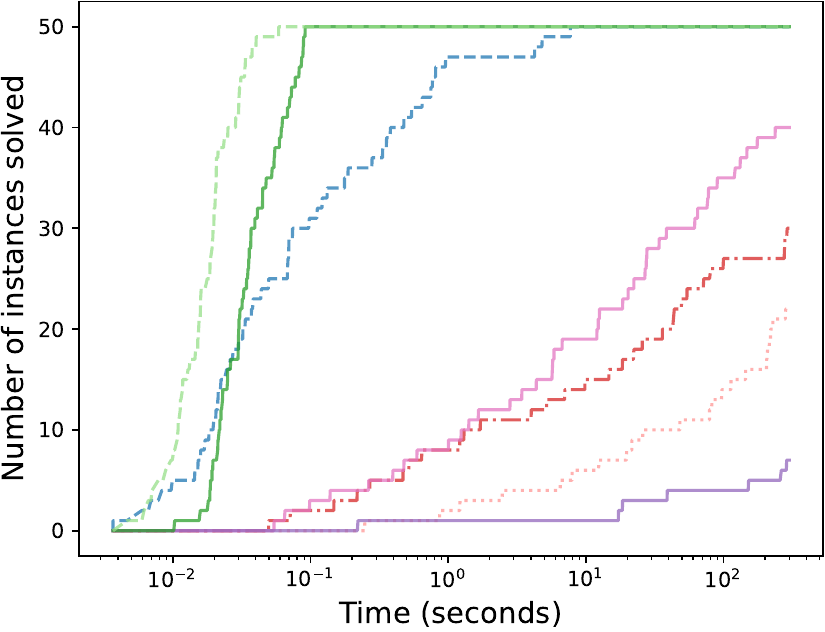}%
\centering%
\includegraphics[width=\plotwidth{}]{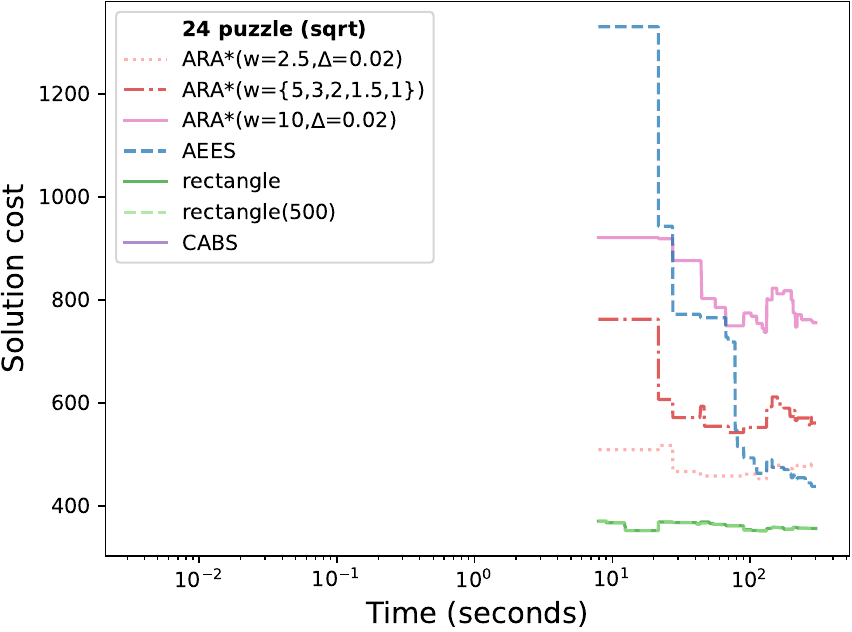}%
\caption{24 puzzle (sqrt cost)}%
\end{figure*}

\begin{figure*}[h!]%
\centering%
\includegraphics[width=\plotwidth{}]{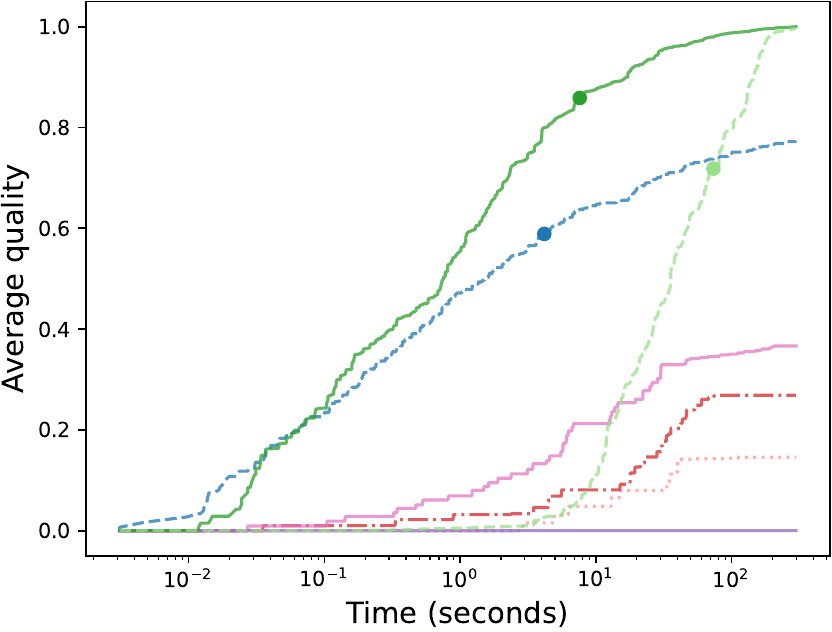}%
\centering%
\includegraphics[width=\plotwidth{}]{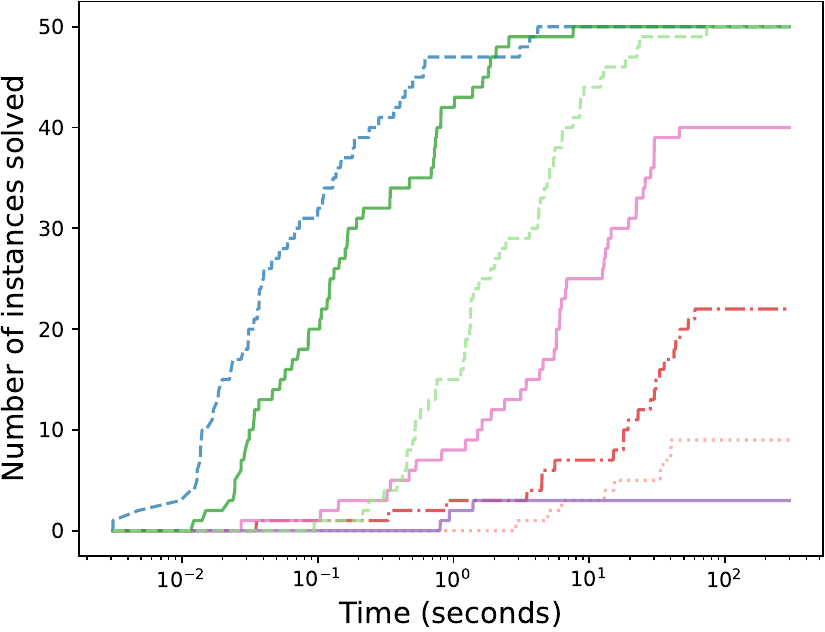}%
\centering%
\includegraphics[width=\plotwidth{}]{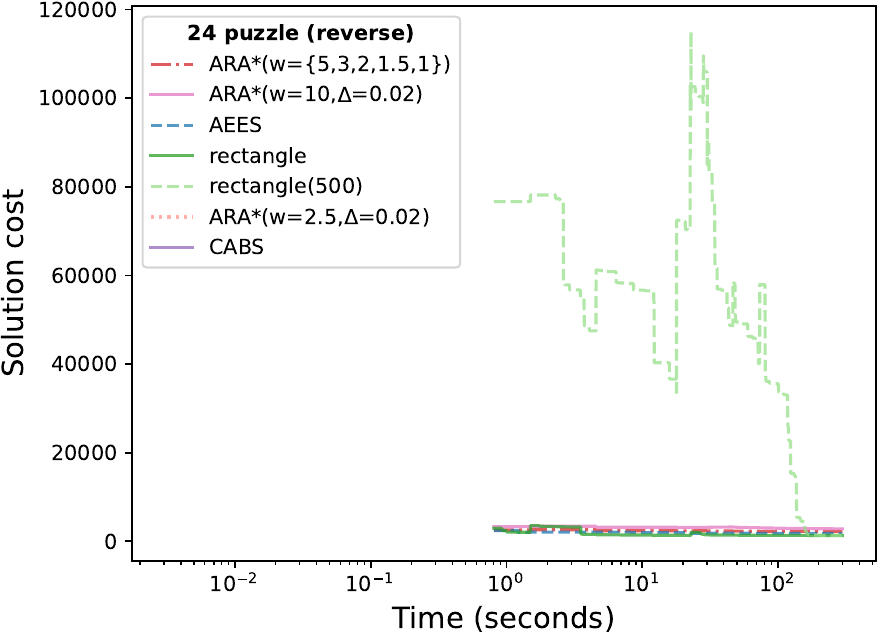}%
\caption{24 puzzle (reverse cost)}%
\end{figure*}

\begin{figure*}[h!]%
\centering%
\includegraphics[width=\plotwidth{}]{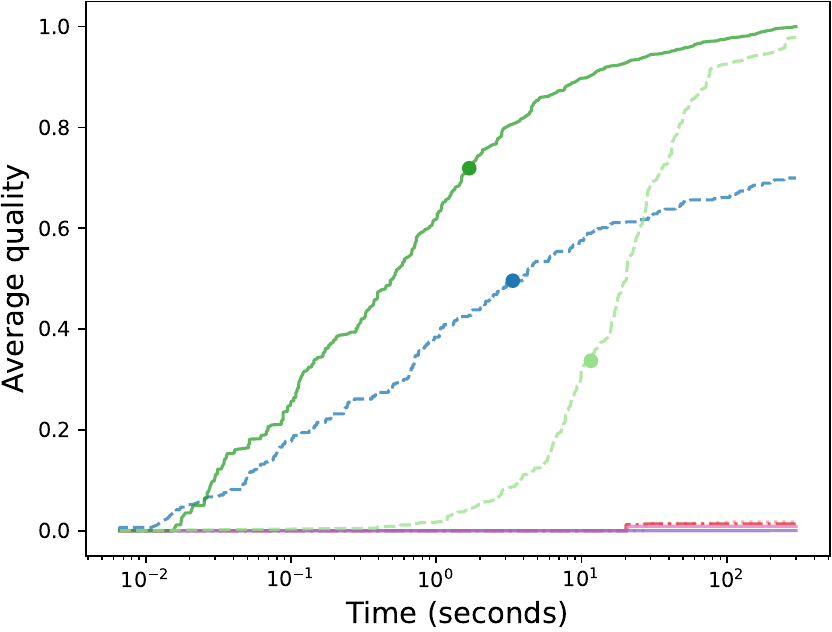}%
\centering%
\includegraphics[width=\plotwidth{}]{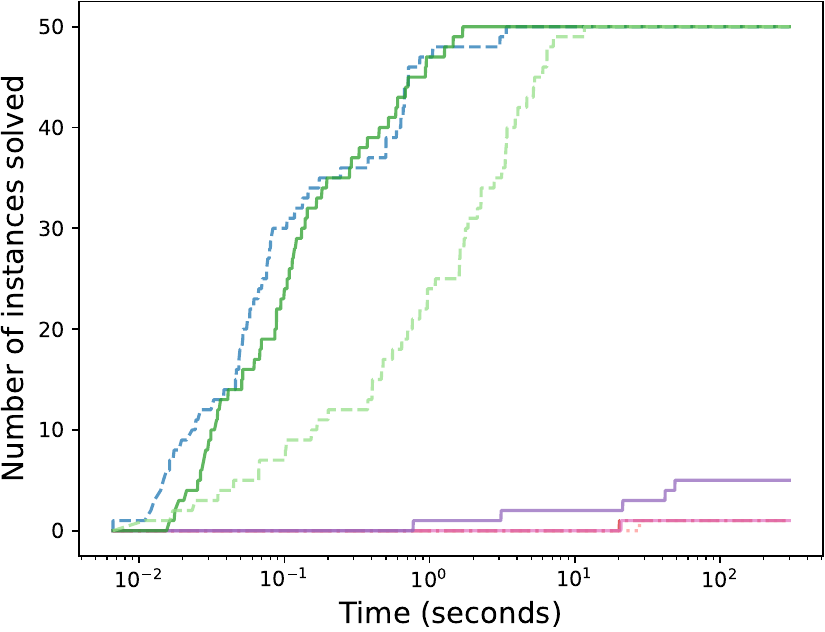}%
\centering%
\includegraphics[width=\plotwidth{}]{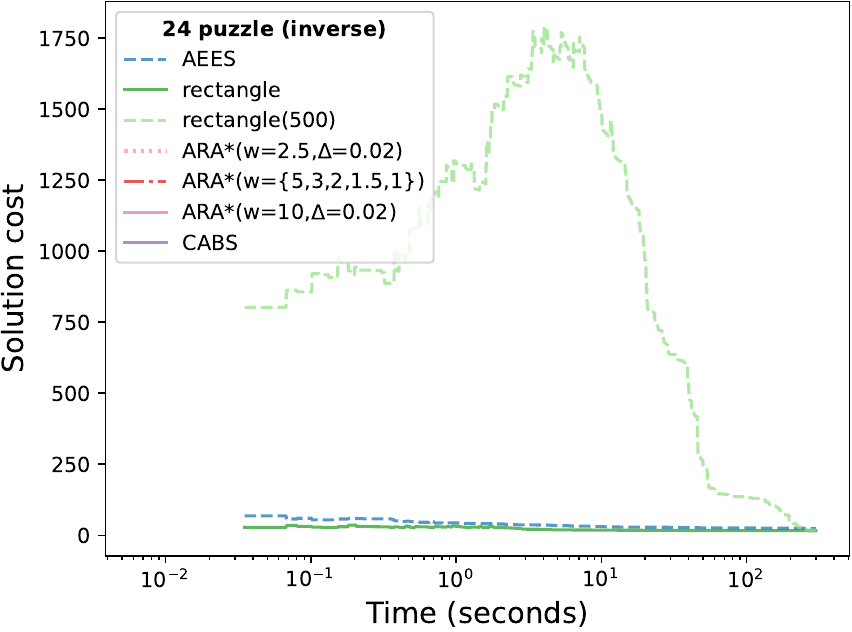}%
\caption{24 puzzle (inverse cost)}%
\end{figure*}

\begin{figure*}[h!]%
\centering%
\includegraphics[width=\plotwidth{}]{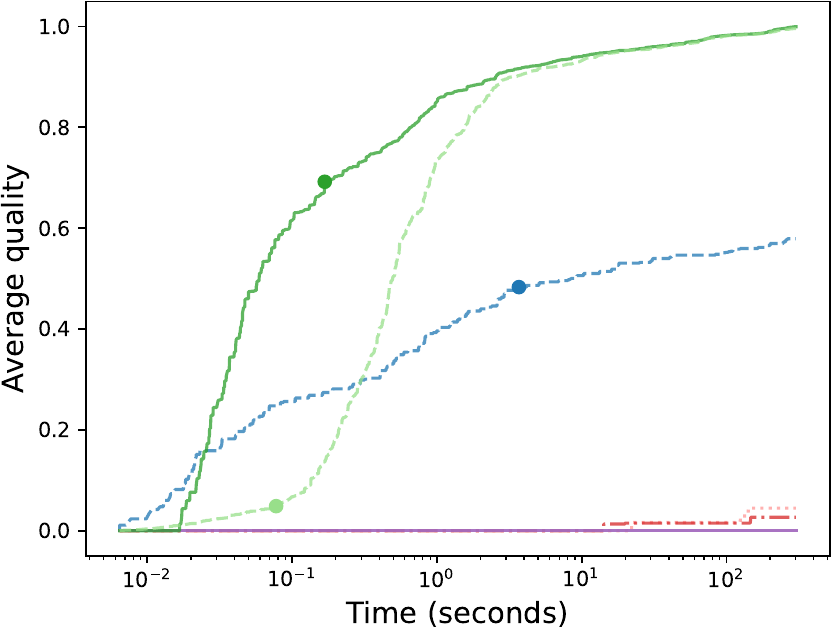}%
\centering%
\includegraphics[width=\plotwidth{}]{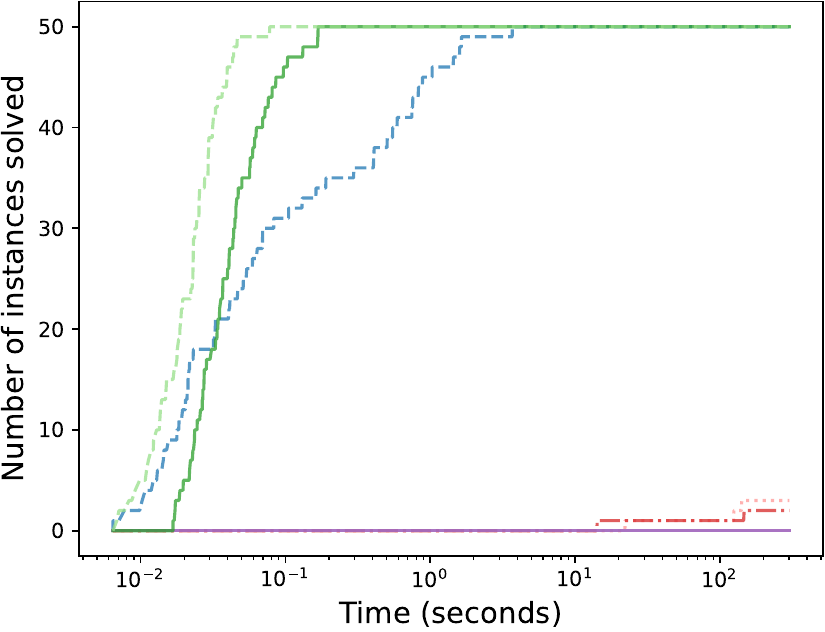}%
\centering%
\includegraphics[width=\plotwidth{}]{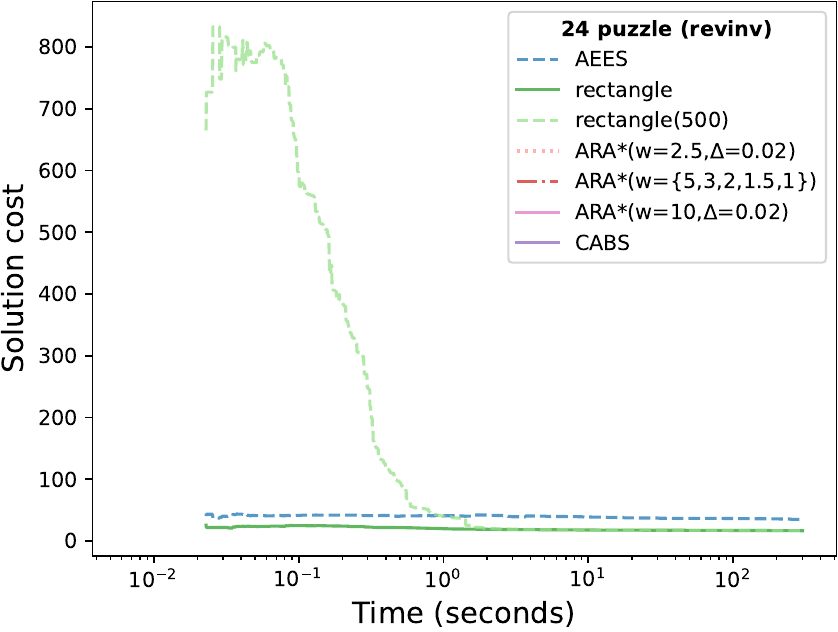}%
\caption{24 puzzle (reverse inverse cost)}%
\end{figure*}

\begin{figure*}[h!]%
\centering%
\includegraphics[width=\plotwidth{}]{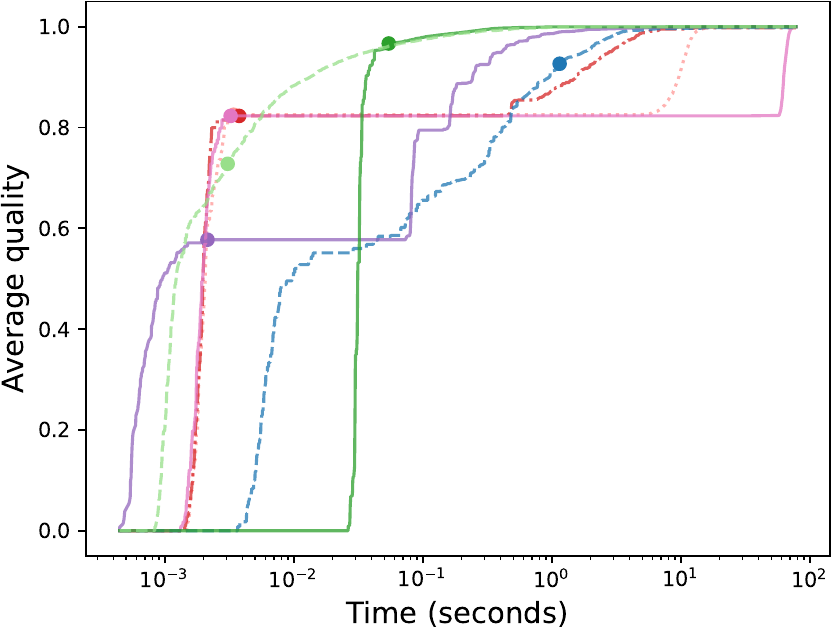}%
\centering%
\includegraphics[width=\plotwidth{}]{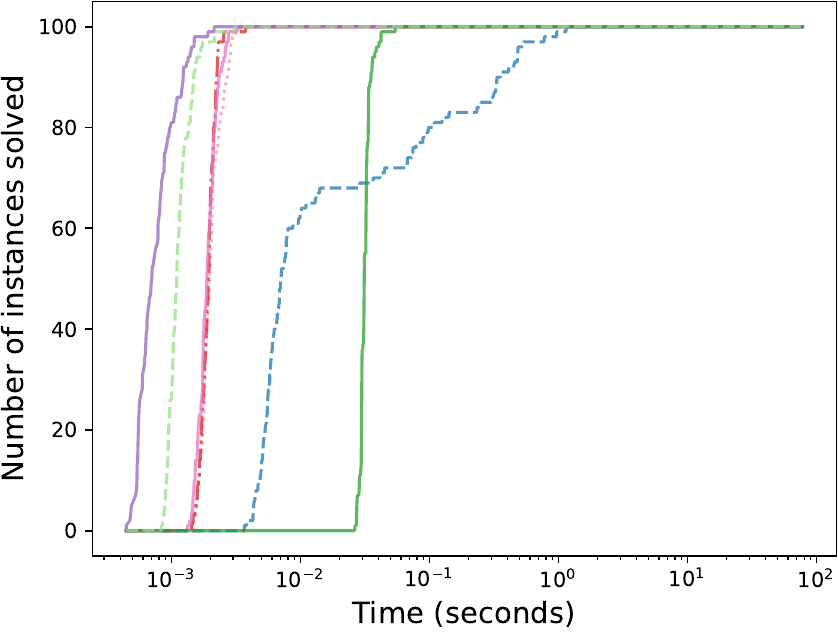}%
\centering%
\includegraphics[width=\plotwidth{}]{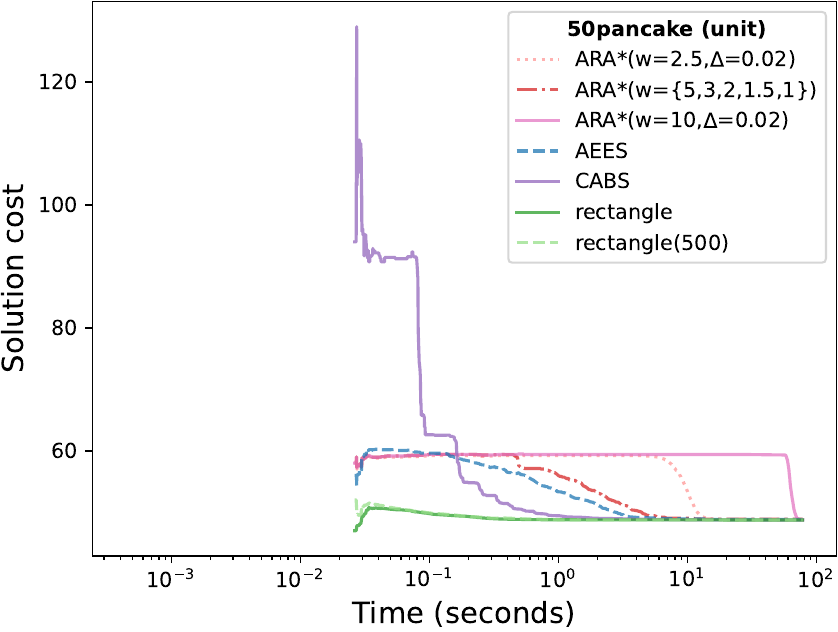}%
\caption{50pancake (unit cost)}%
\end{figure*}

\begin{figure*}[h!]%
\centering%
\includegraphics[width=\plotwidth{}]{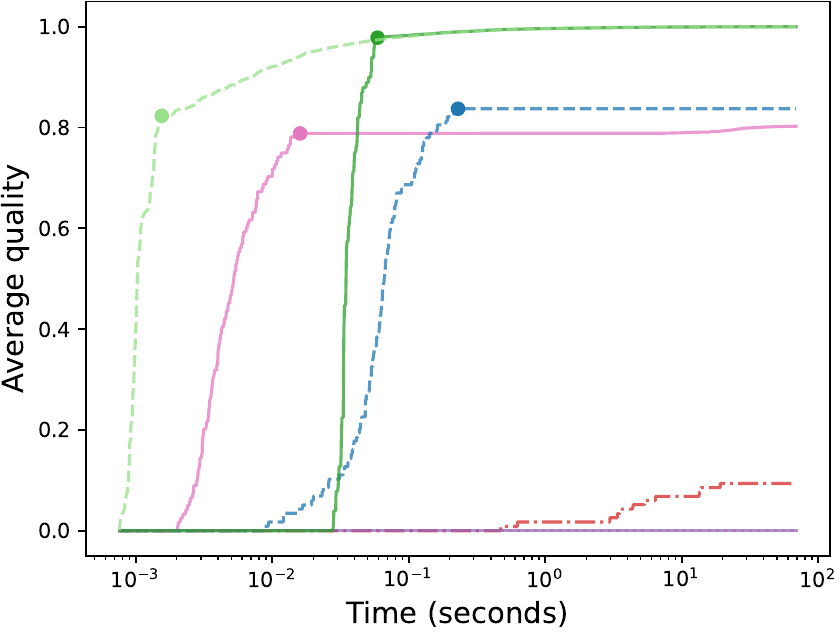}%
\centering%
\includegraphics[width=\plotwidth{}]{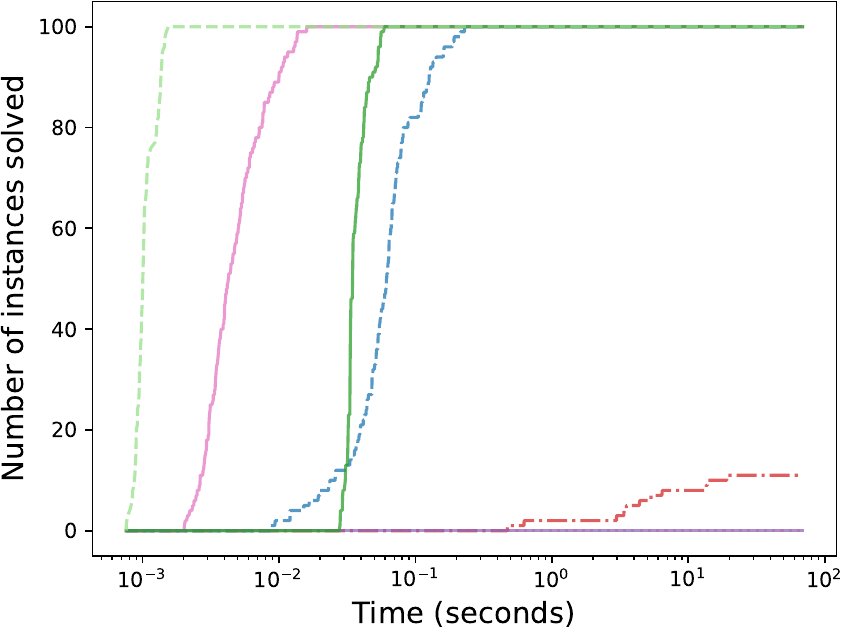}%
\centering%
\includegraphics[width=\plotwidth{}]{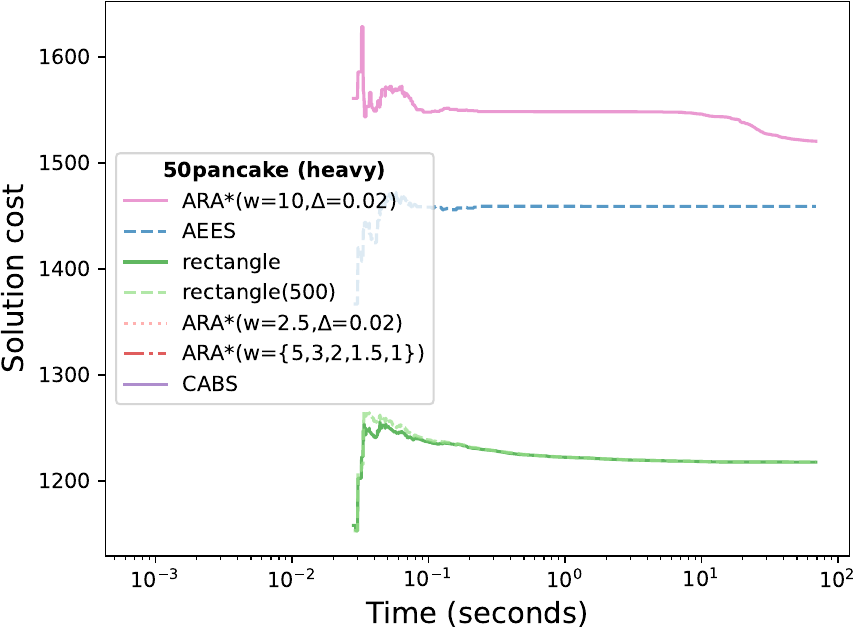}%
\caption{50pancake (heavy cost)}%
\end{figure*}

\begin{figure*}[h!]%
\centering%
\includegraphics[width=\plotwidth{}]{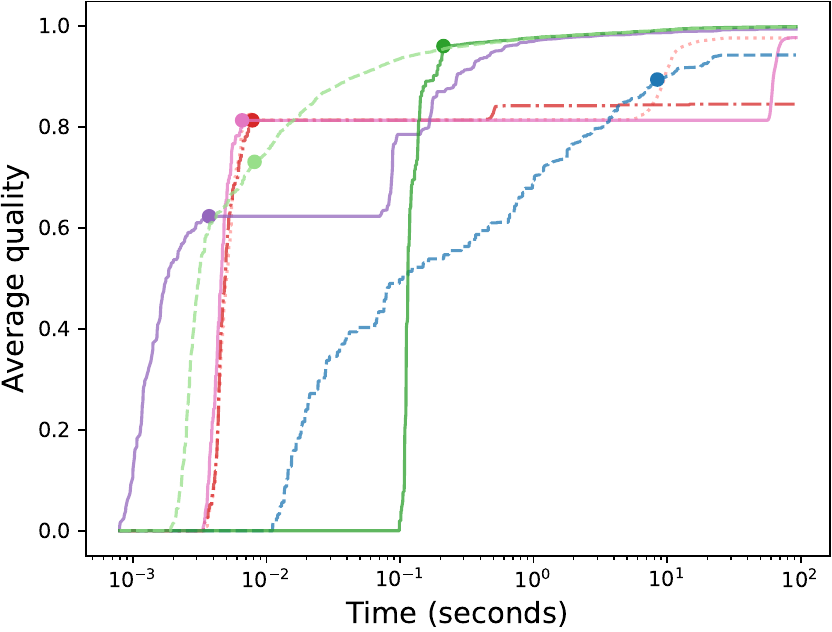}%
\centering%
\includegraphics[width=\plotwidth{}]{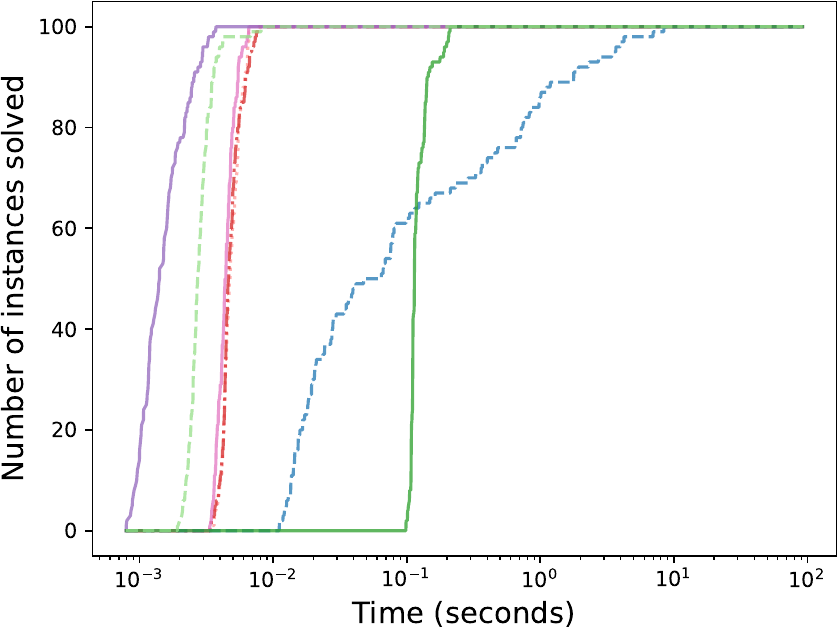}%
\centering%
\includegraphics[width=\plotwidth{}]{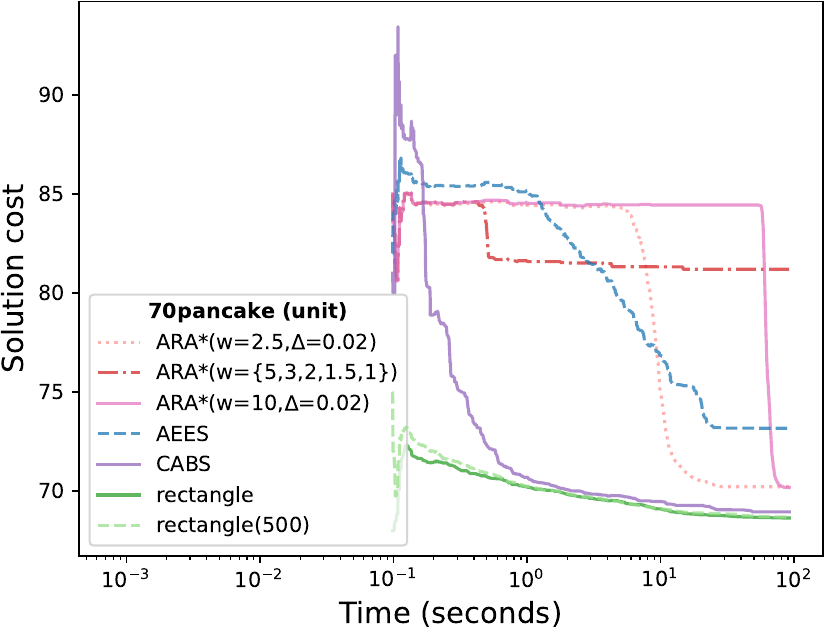}%
\caption{70pancake (unit cost)}%
\end{figure*}

\begin{figure*}[h!]%
\centering%
\includegraphics[width=\plotwidth{}]{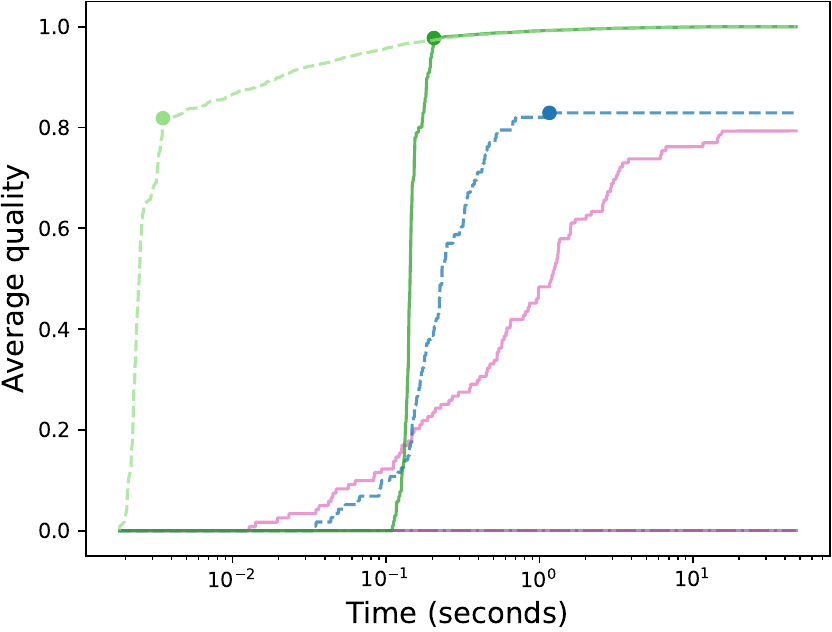}%
\centering%
\includegraphics[width=\plotwidth{}]{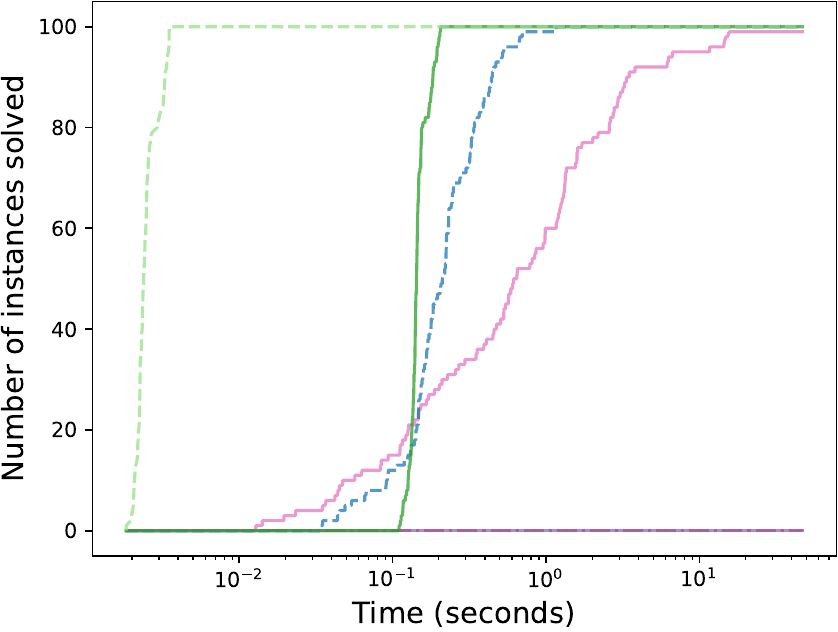}%
\centering%
\includegraphics[width=\plotwidth{}]{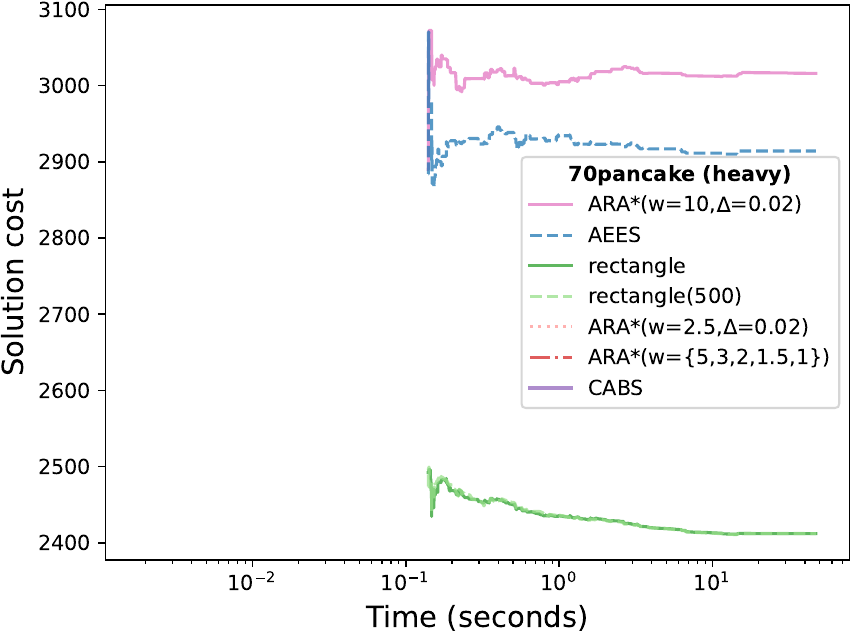}%
\caption{70pancake (heavy cost)}%
\end{figure*}

\begin{figure*}[h!]%
\centering%
\includegraphics[width=\plotwidth{}]{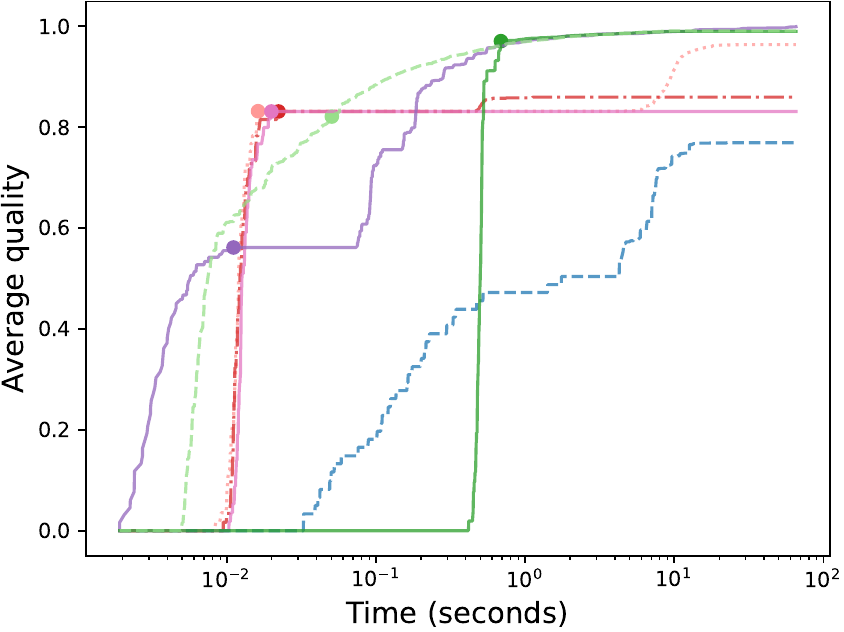}%
\centering%
\includegraphics[width=\plotwidth{}]{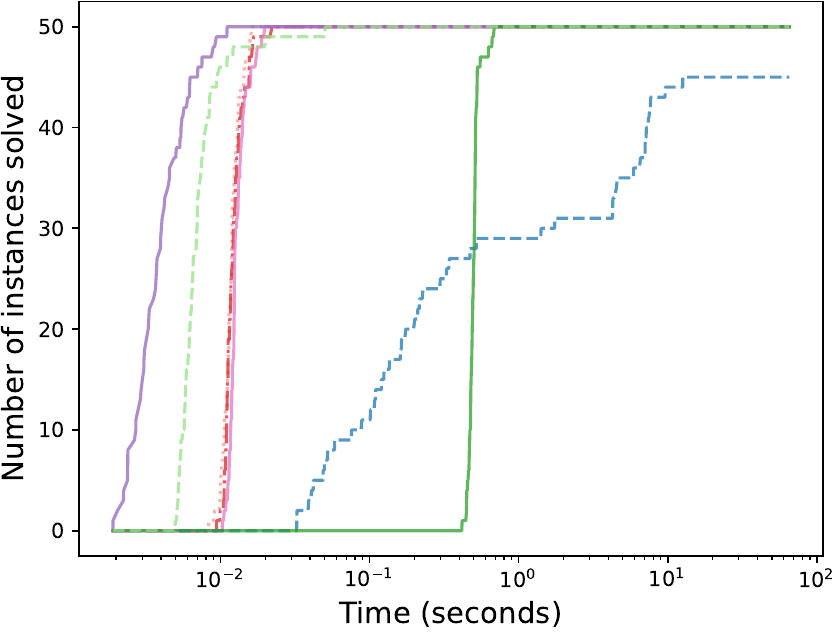}%
\centering%
\includegraphics[width=\plotwidth{}]{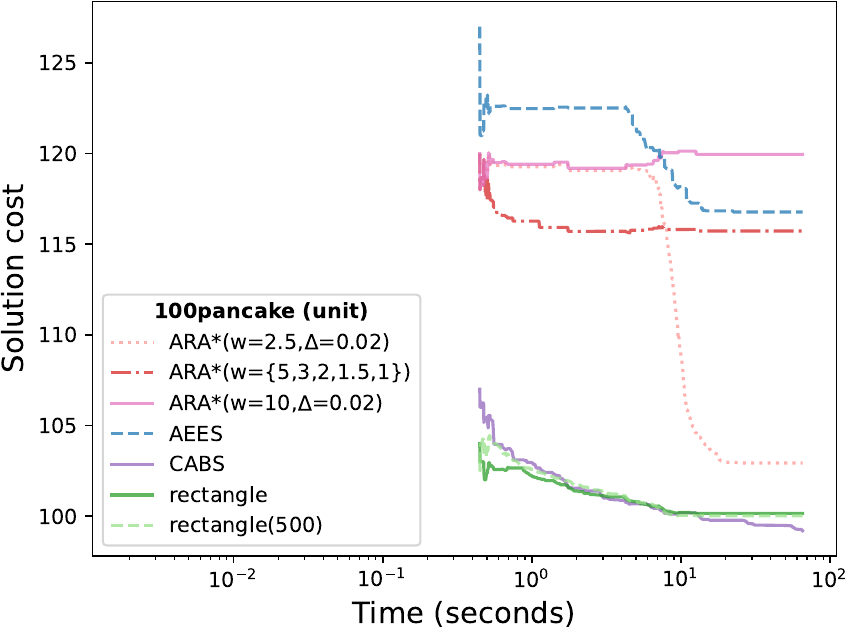}%
\caption{100pancake (unit cost)}%
\end{figure*}

\begin{figure*}[h!]%
\centering%
\includegraphics[width=\plotwidth{}]{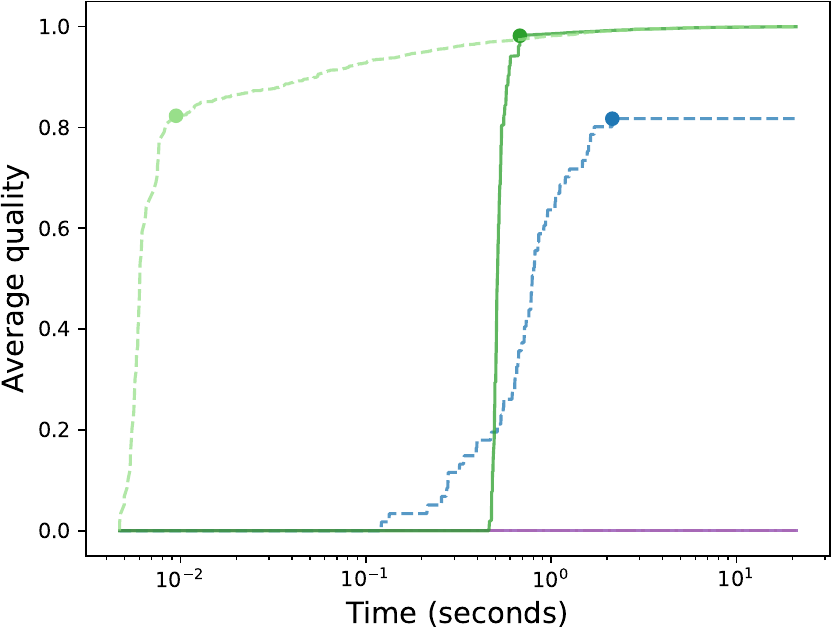}%
\centering%
\includegraphics[width=\plotwidth{}]{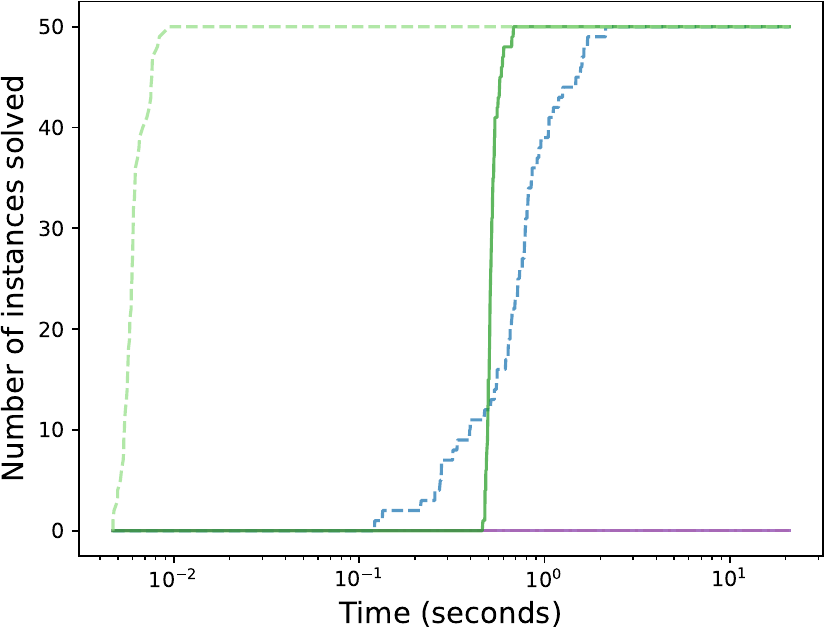}%
\centering%
\includegraphics[width=\plotwidth{}]{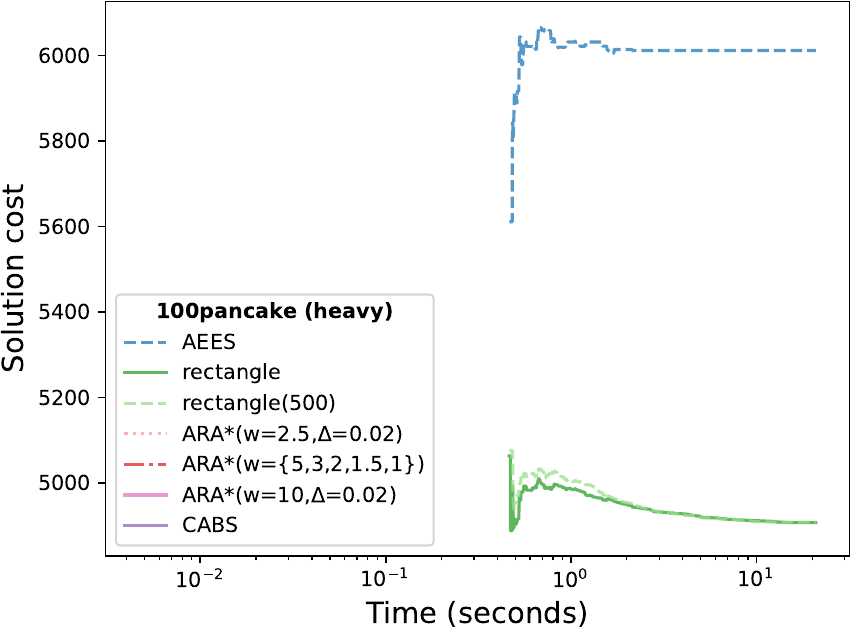}%
\caption{100pancake (heavy cost)}%
\end{figure*}

\begin{figure*}[h!]%
\centering%
\includegraphics[width=\plotwidth{}]{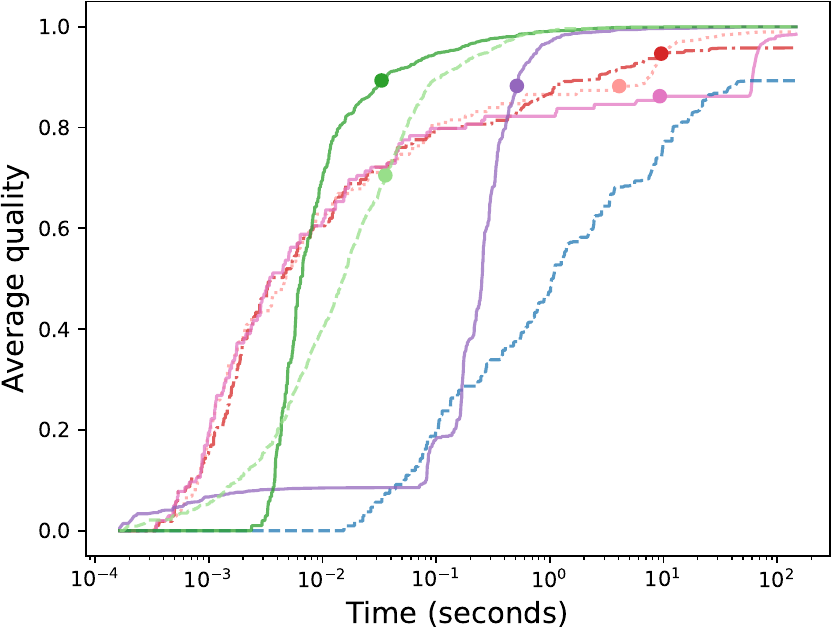}%
\centering%
\includegraphics[width=\plotwidth{}]{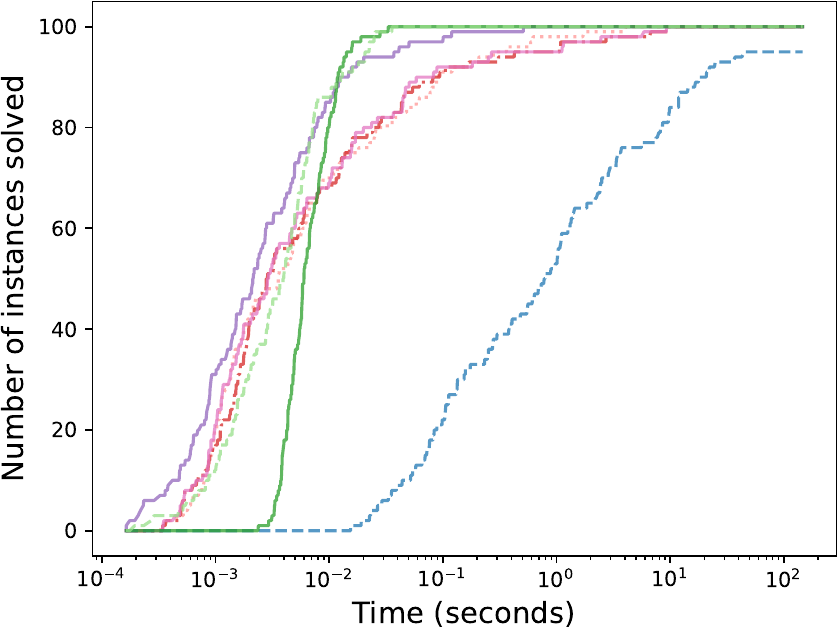}%
\centering%
\includegraphics[width=\plotwidth{}]{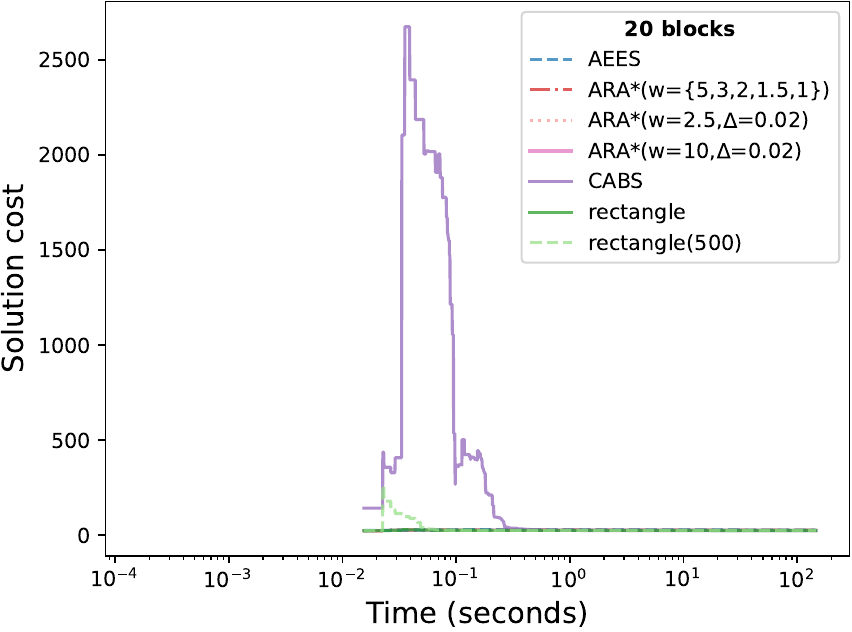}%
\caption{20 blocks}%
\end{figure*}

\begin{figure*}[h!]%
\centering%
\includegraphics[width=\plotwidth{}]{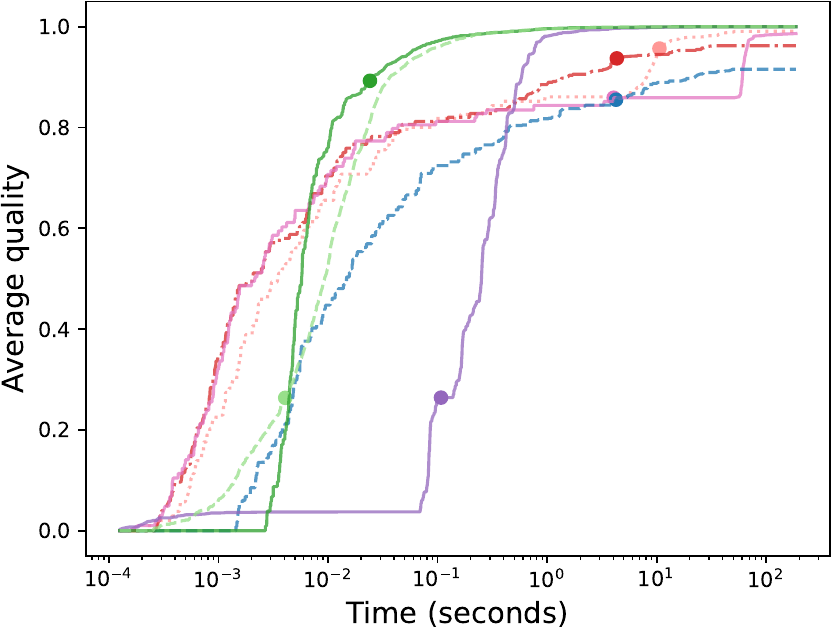}%
\centering%
\includegraphics[width=\plotwidth{}]{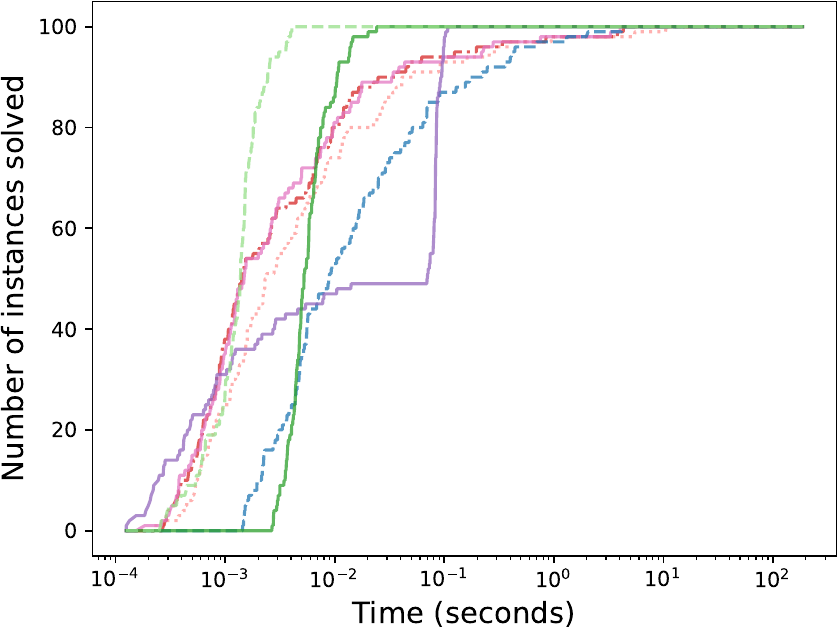}%
\centering%
\includegraphics[width=\plotwidth{}]{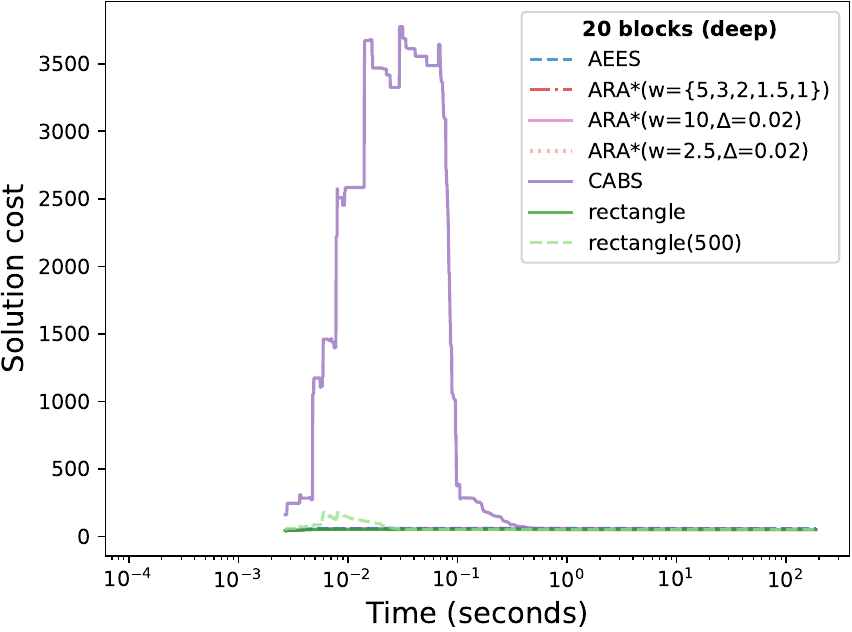}%
\caption{20 blocks (deep)}%
\end{figure*}

\begin{figure*}[h!]%
\centering%
\includegraphics[width=\plotwidth{}]{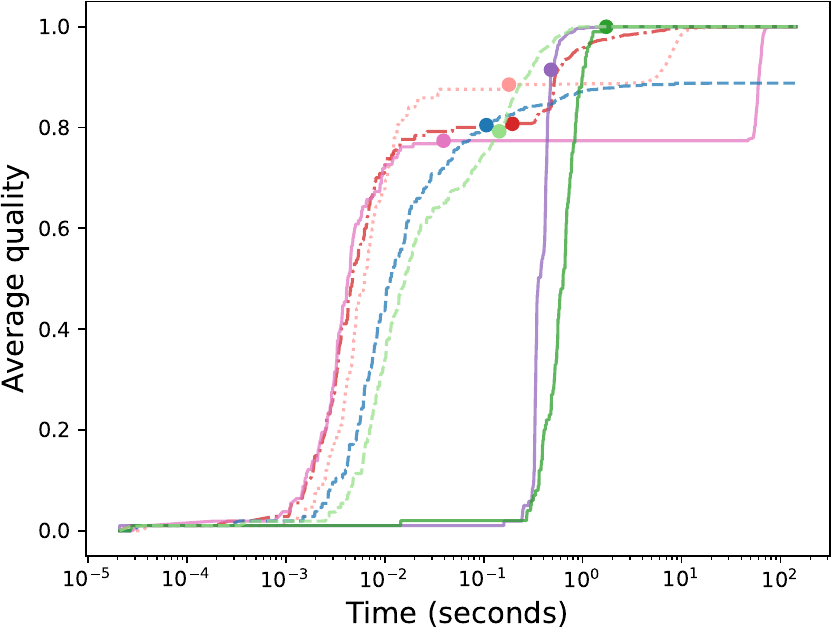}%
\centering%
\includegraphics[width=\plotwidth{}]{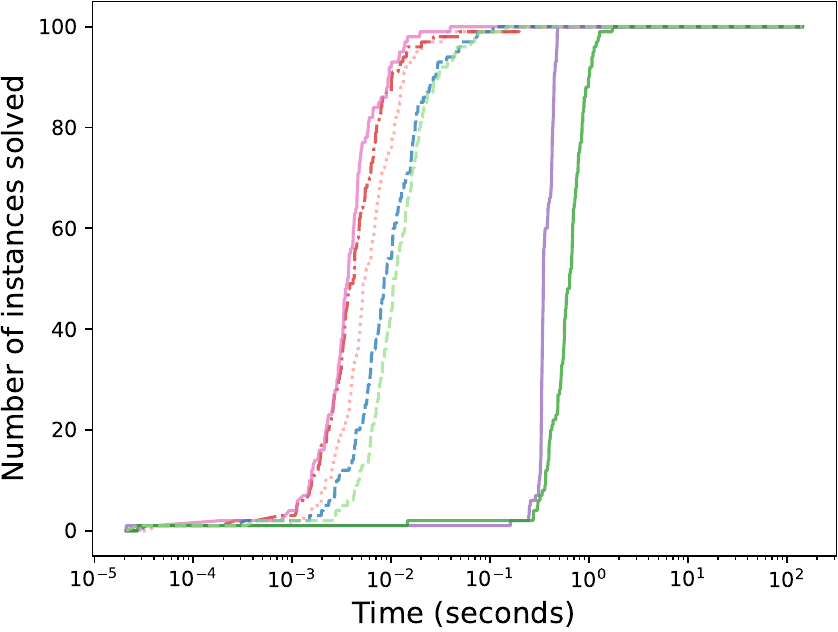}%
\centering%
\includegraphics[width=\plotwidth{}]{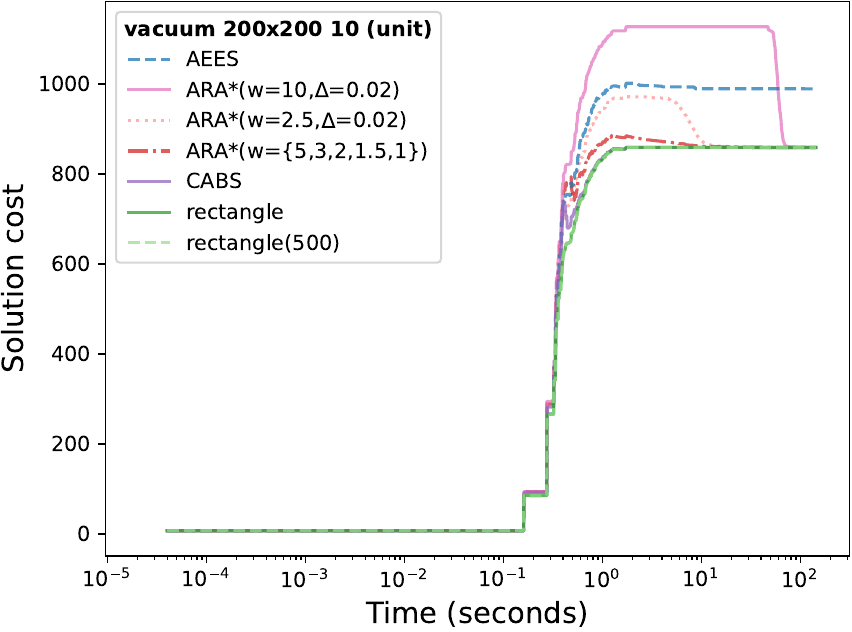}%
\caption{vacuum 200x200, 10 dirts (unit cost)}%
\end{figure*}

\begin{figure*}[h!]%
\centering%
\includegraphics[width=\plotwidth{}]{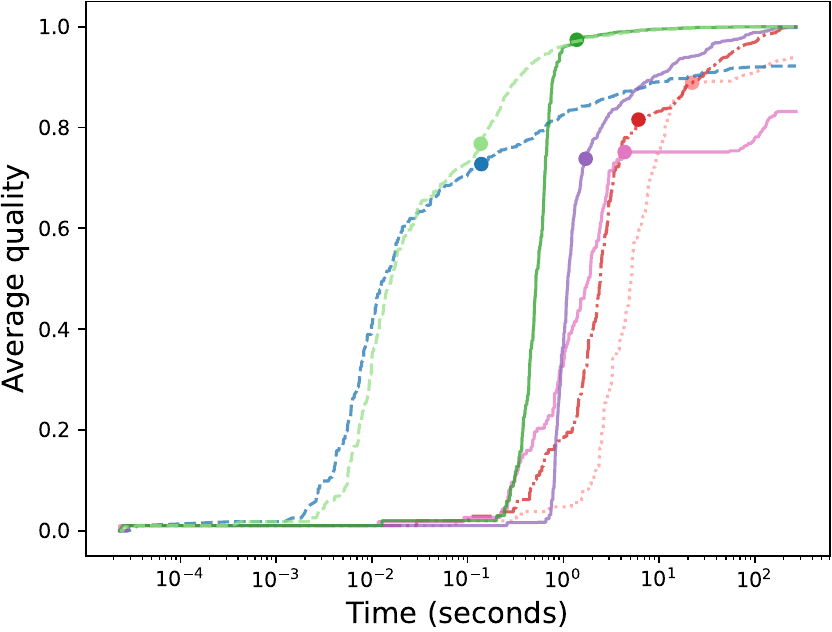}%
\centering%
\includegraphics[width=\plotwidth{}]{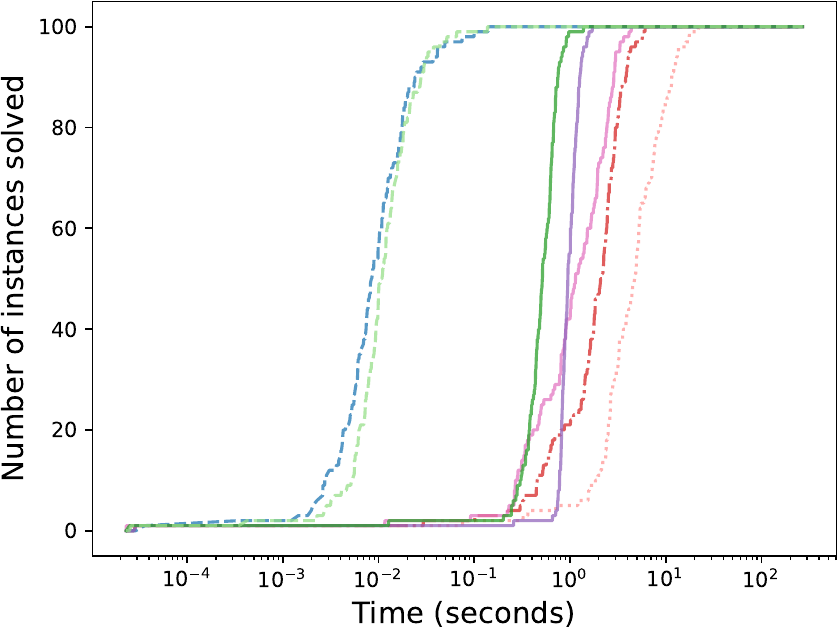}%
\centering%
\includegraphics[width=\plotwidth{}]{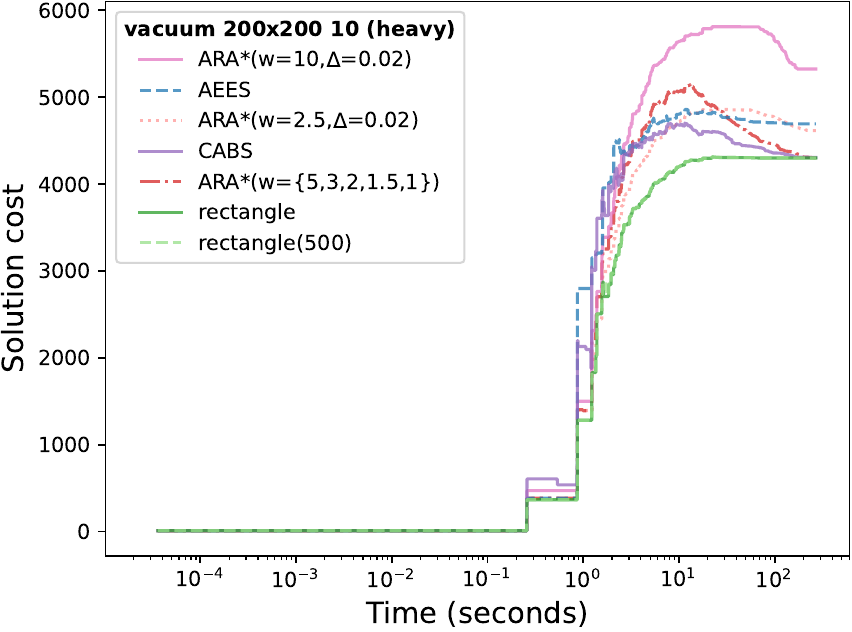}%
\caption{vacuum 200x200, 10 dirts (heavy cost)}%
\end{figure*}

\begin{figure*}[h!]%
\centering%
\includegraphics[width=\plotwidth{}]{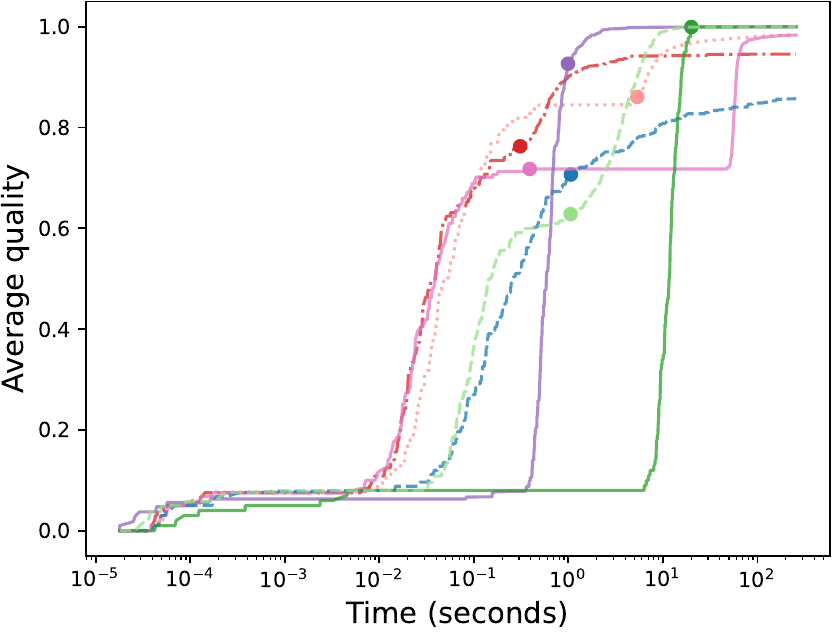}%
\centering%
\includegraphics[width=\plotwidth{}]{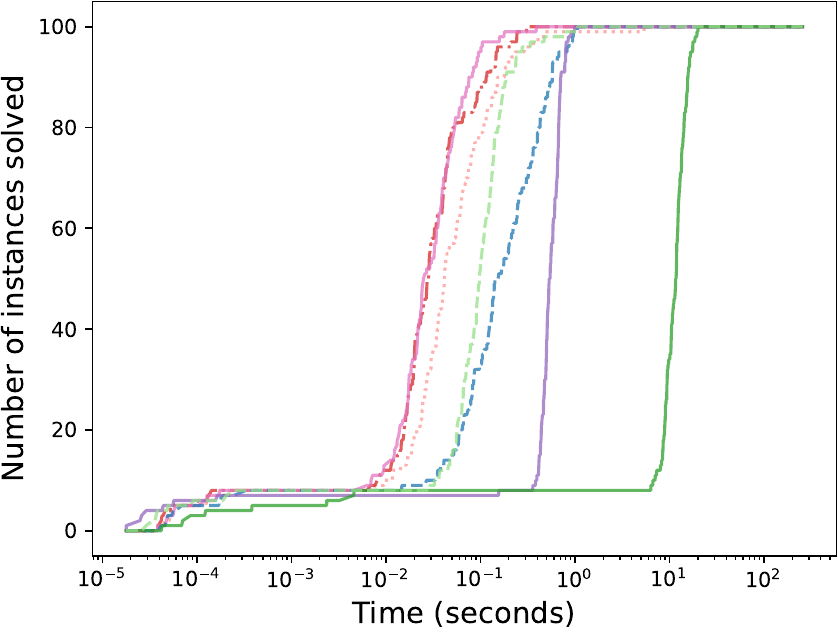}%
\centering%
\includegraphics[width=\plotwidth{}]{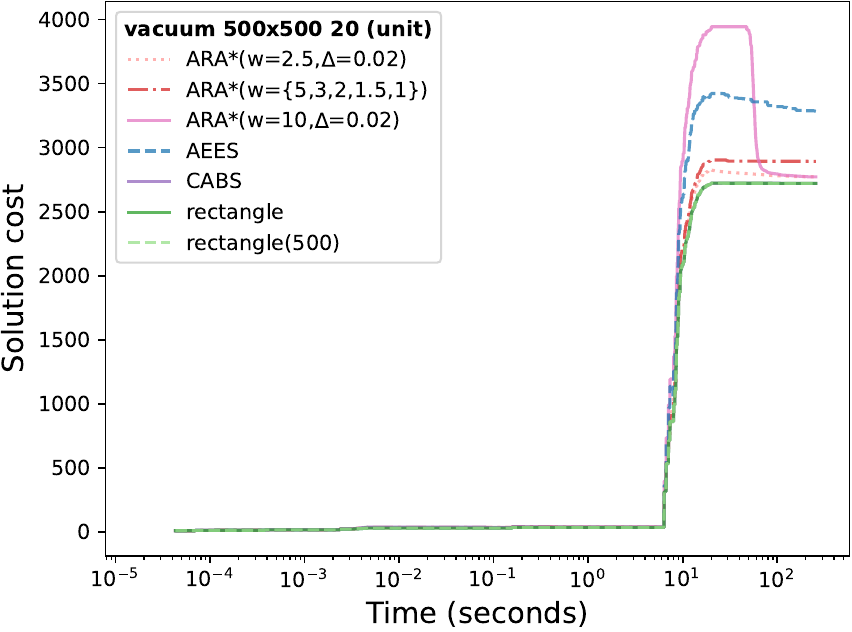}%
\caption{vacuum 500x500, 20 dirts (unit cost)}%
\end{figure*}

\begin{figure*}[h!]%
\centering%
\includegraphics[width=\plotwidth{}]{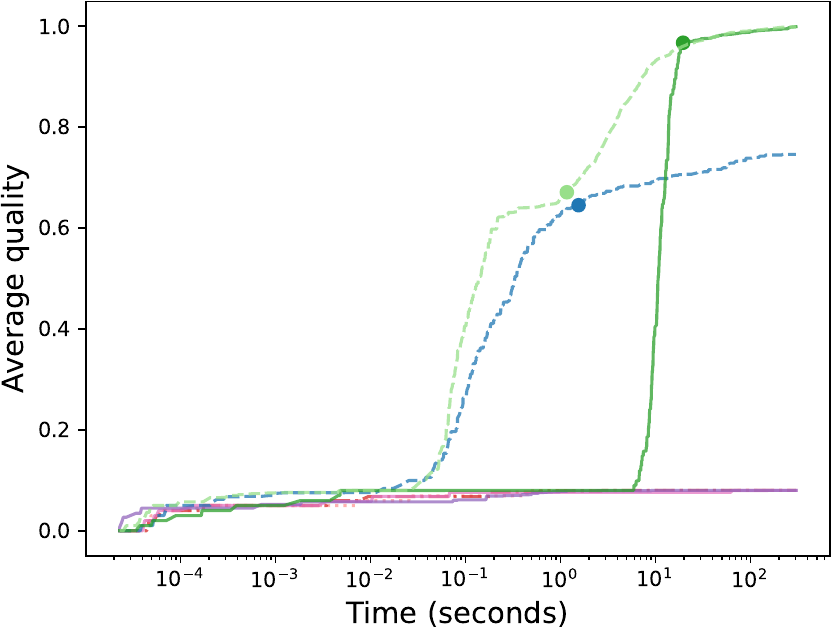}%
\centering%
\includegraphics[width=\plotwidth{}]{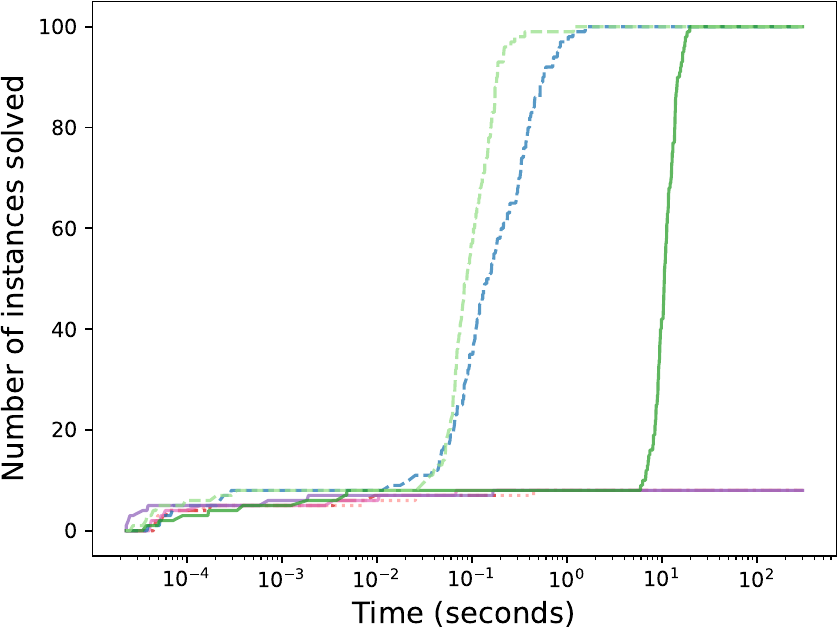}%
\centering%
\includegraphics[width=\plotwidth{}]{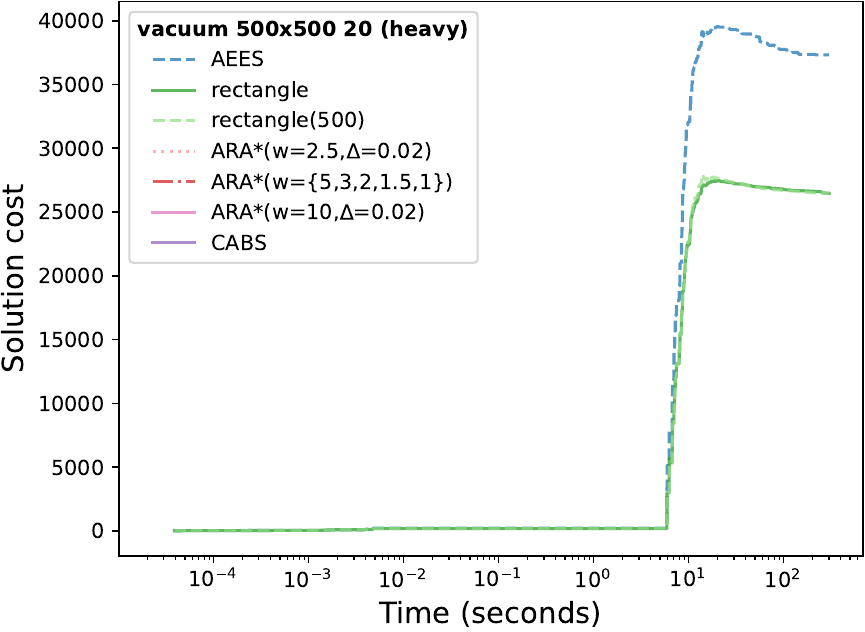}%
\caption{vacuum 500x500, 20 dirts (heavy cost)}%
\end{figure*}

\begin{figure*}[h!]%
\centering%
\includegraphics[width=\plotwidth{}]{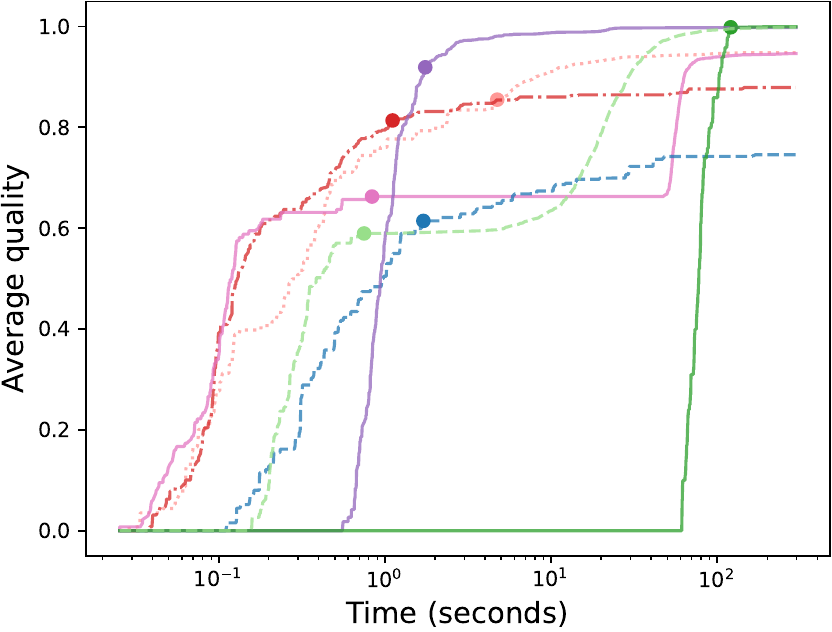}%
\centering%
\includegraphics[width=\plotwidth{}]{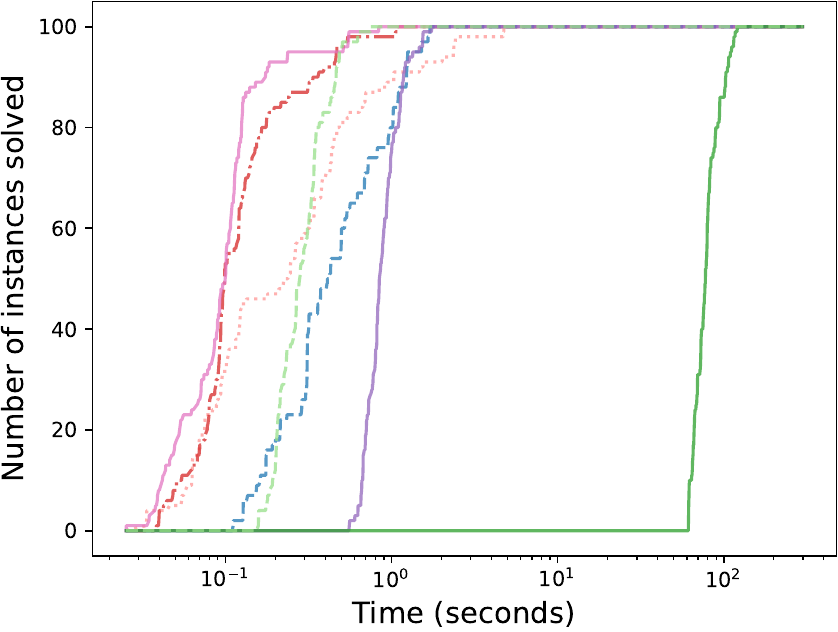}%
\centering%
\includegraphics[width=\plotwidth{}]{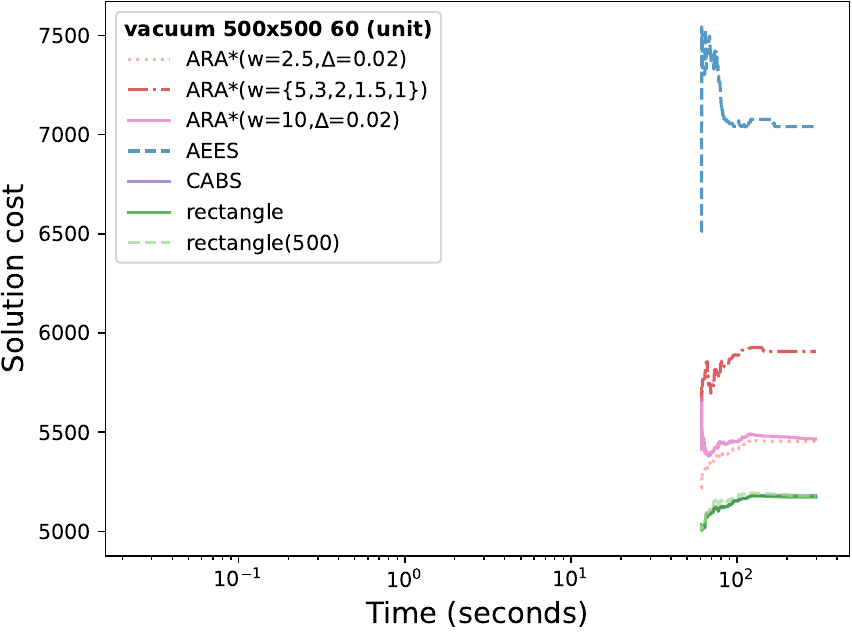}%
\caption{vacuum 500x500, 60 dirts (unit cost)}%
\end{figure*}

\begin{figure*}[h!]%
\centering%
\includegraphics[width=\plotwidth{}]{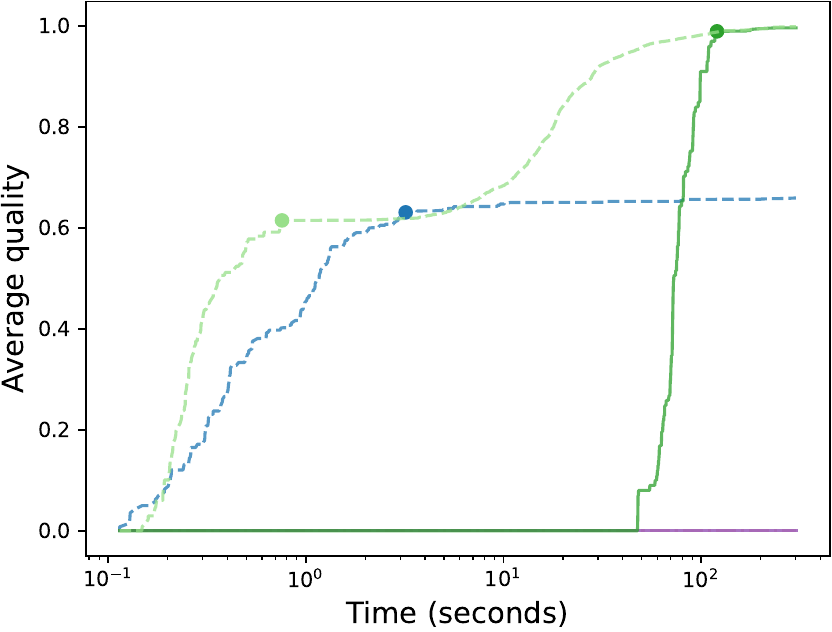}%
\centering%
\includegraphics[width=\plotwidth{}]{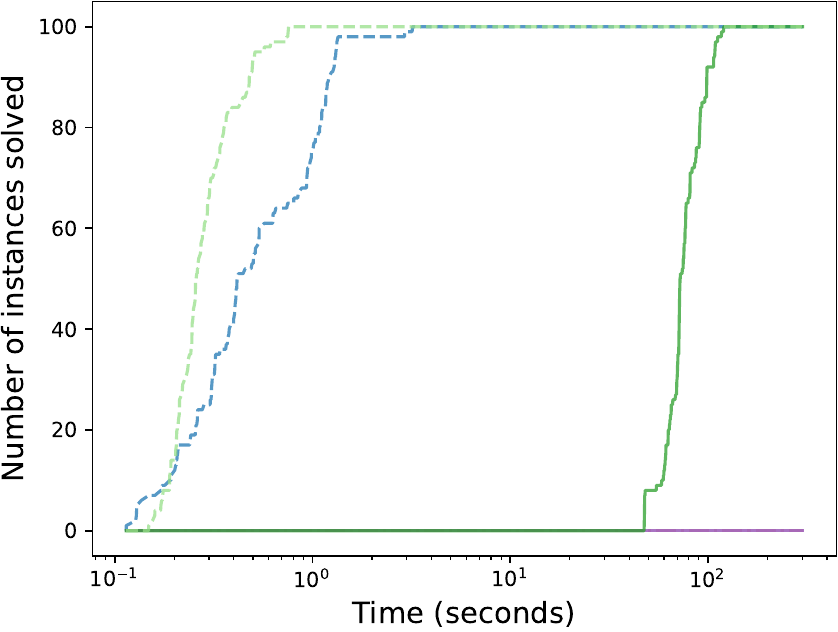}%
\centering%
\includegraphics[width=\plotwidth{}]{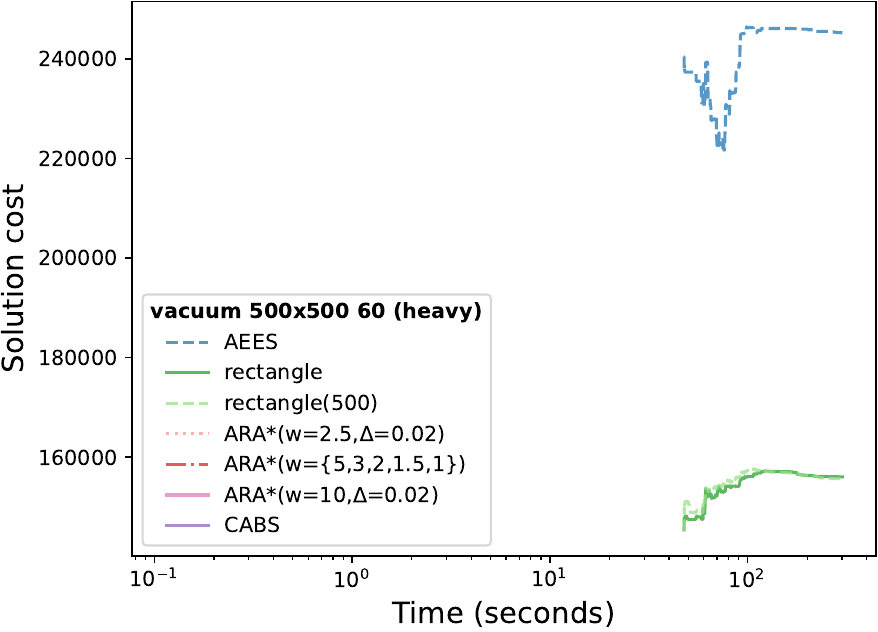}%
\caption{vacuum 500x500, 60 dirts (heavy cost)}%
\end{figure*}

\begin{figure*}[h!]%
\centering%
\includegraphics[width=\plotwidth{}]{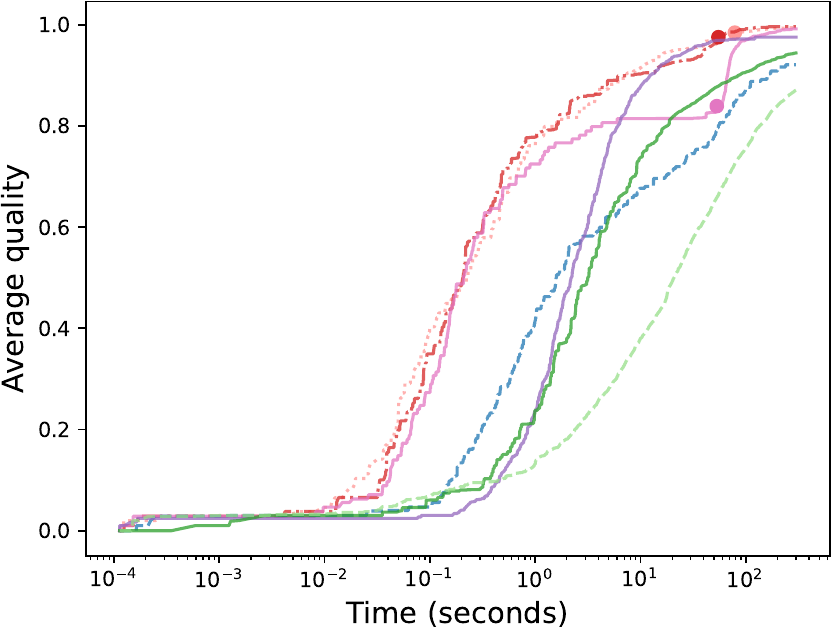}%
\centering%
\includegraphics[width=\plotwidth{}]{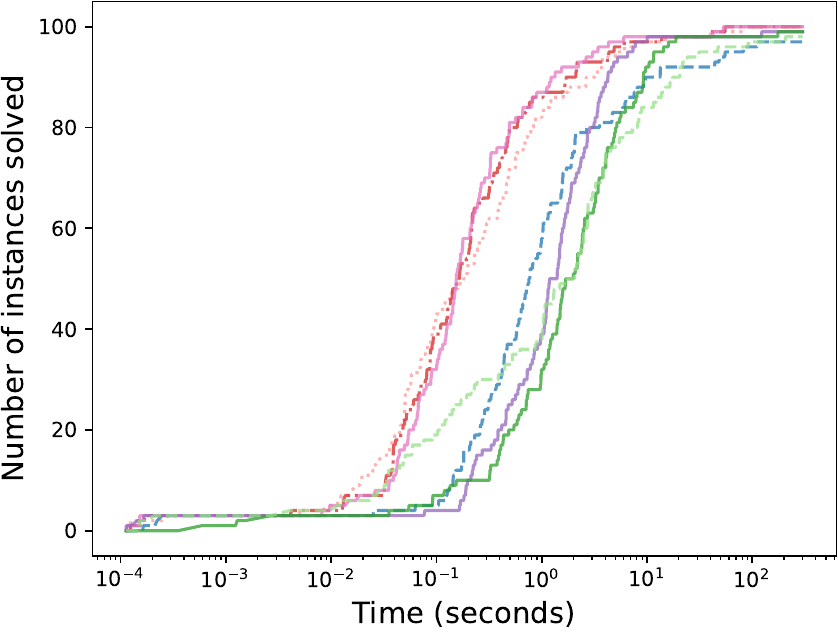}%
\centering%
\includegraphics[width=\plotwidth{}]{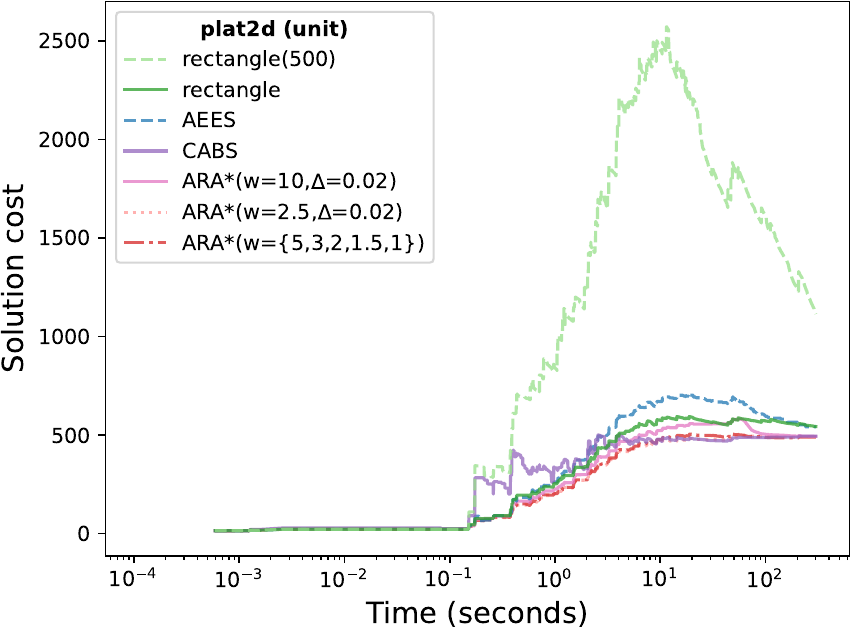}%
\caption{platform game (unit cost)}%
\end{figure*}

\begin{figure*}[h!]%
\centering%
\includegraphics[width=\plotwidth{}]{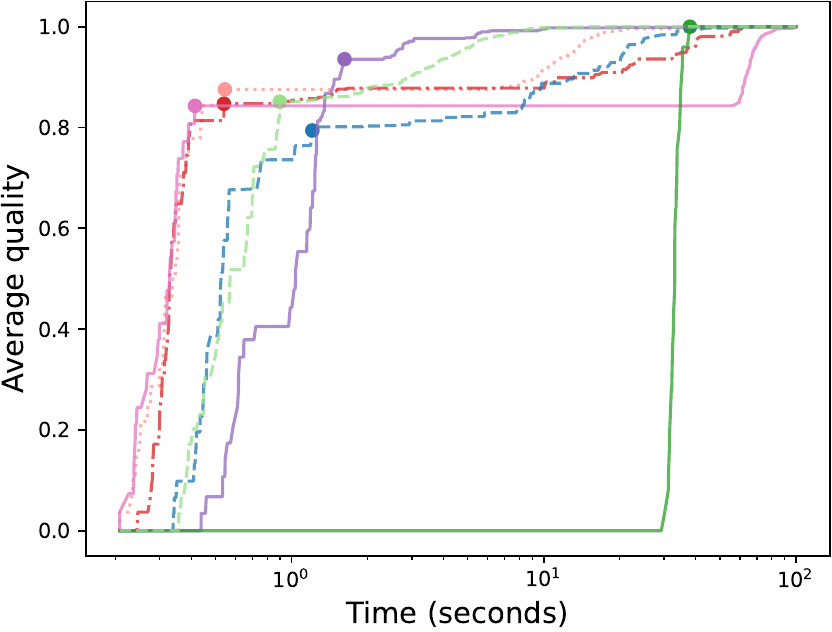}%
\centering%
\includegraphics[width=\plotwidth{}]{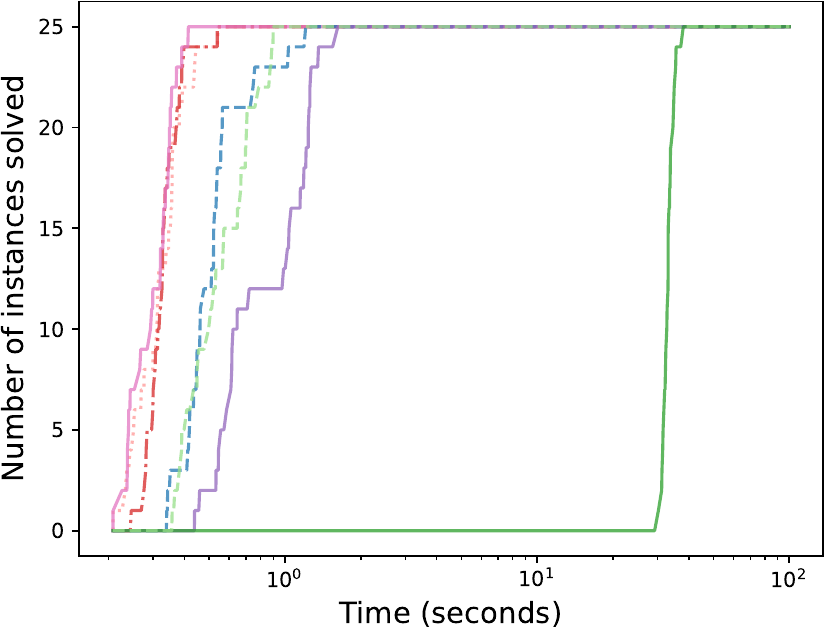}%
\centering%
\includegraphics[width=\plotwidth{}]{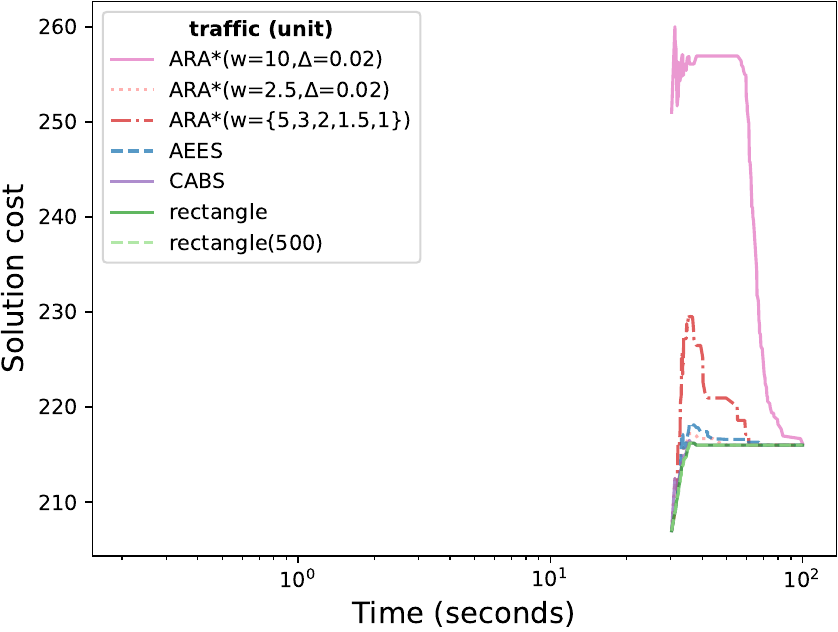}%
\caption{traffic (unit cost)}%
\end{figure*}

\begin{figure*}[h!]%
\centering%
\includegraphics[width=\plotwidth{}]{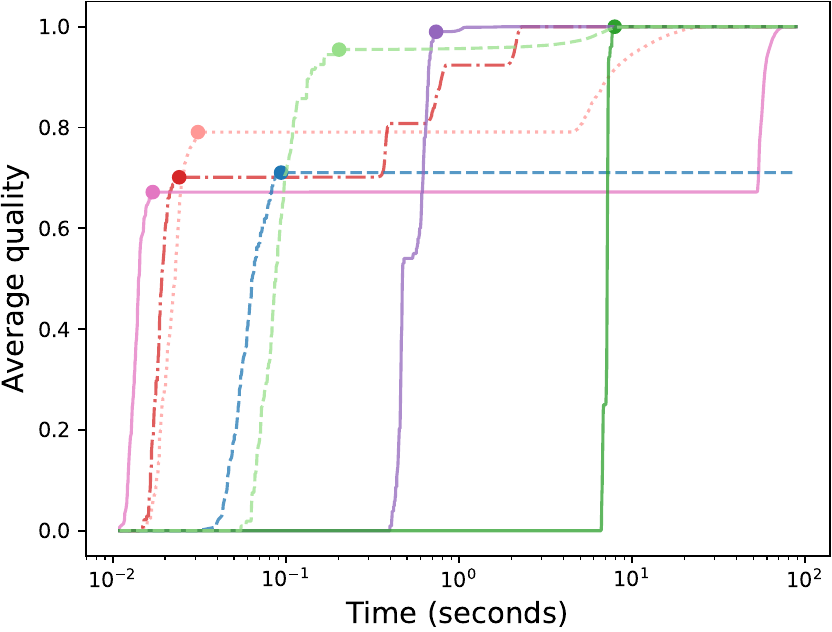}%
\centering%
\includegraphics[width=\plotwidth{}]{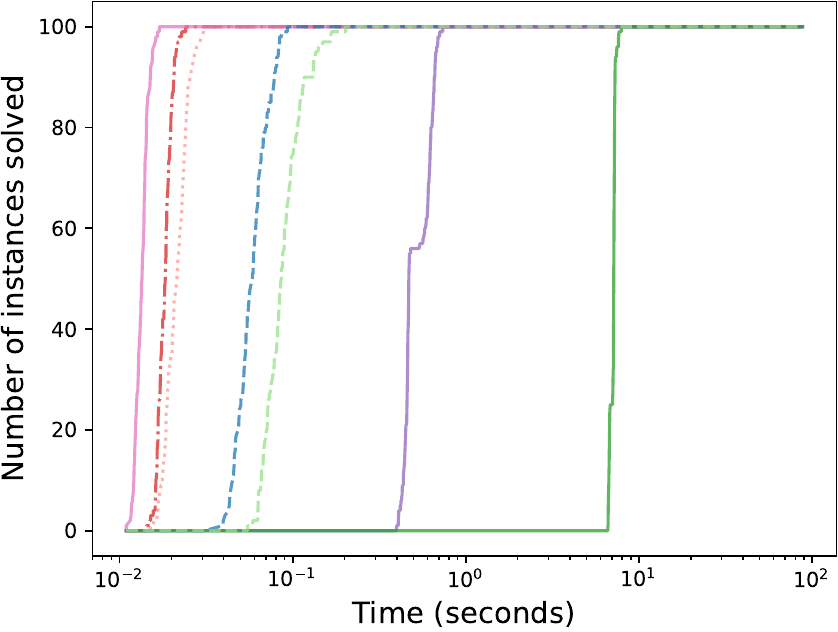}%
\centering%
\includegraphics[width=\plotwidth{}]{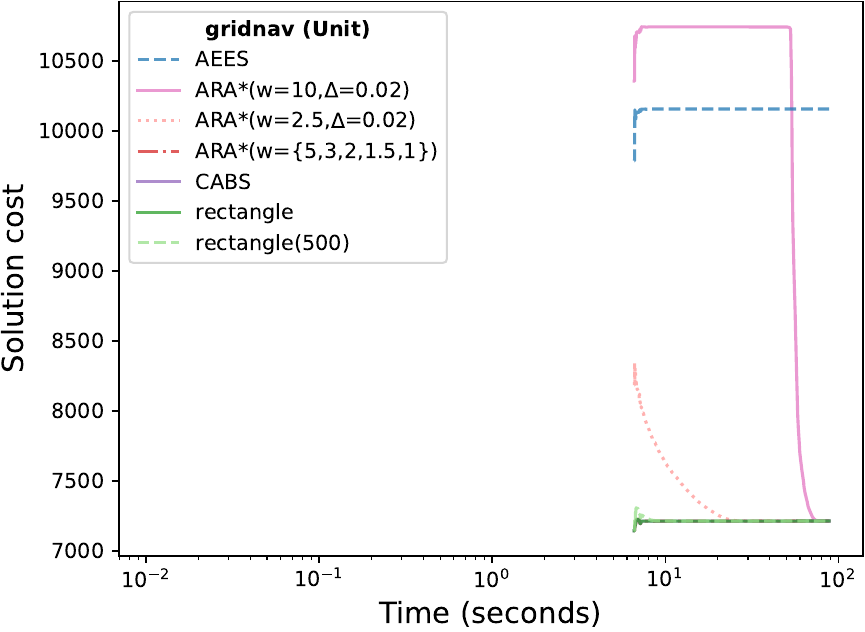}%
\caption{random grids (unit cost)}%
\end{figure*}

\begin{figure*}[h!]%
\centering%
\includegraphics[width=\plotwidth{}]{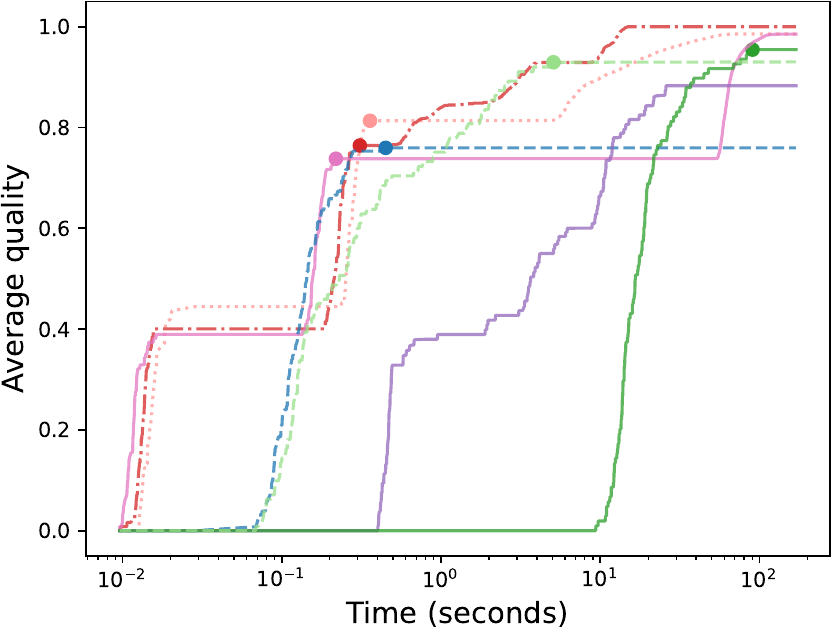}%
\centering%
\includegraphics[width=\plotwidth{}]{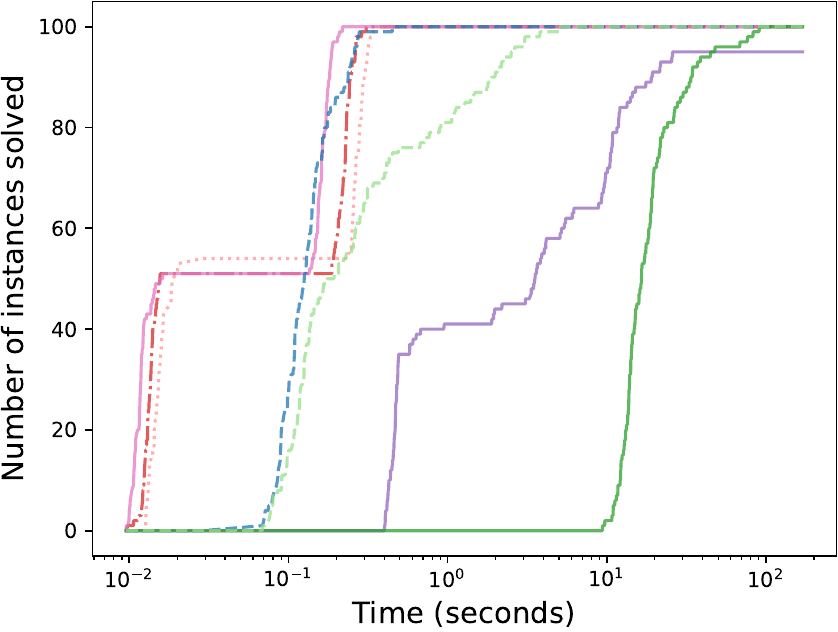}%
\centering%
\includegraphics[width=\plotwidth{}]{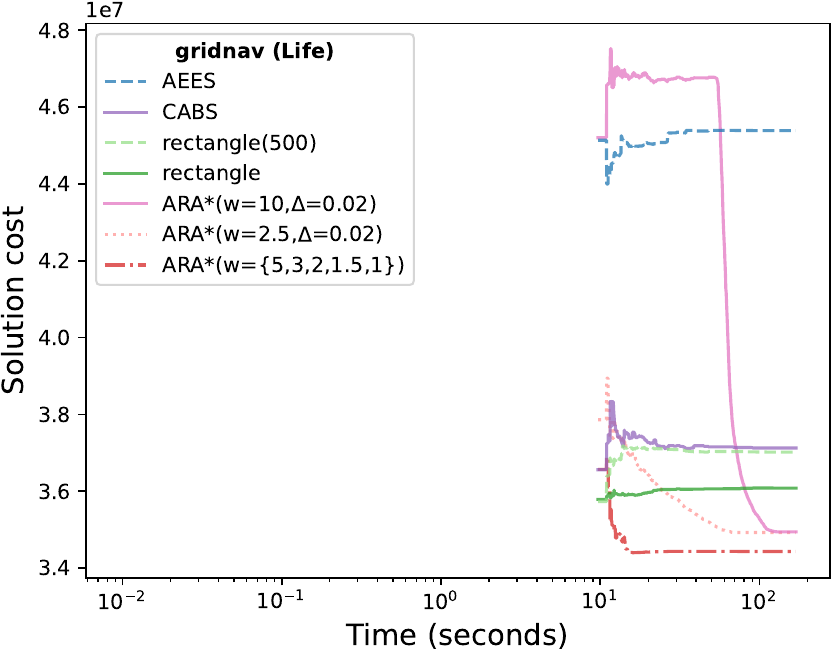}%
\caption{random grids (life cost)}%
\end{figure*}

\begin{figure*}[h!]%
\centering%
\includegraphics[width=\plotwidth{}]{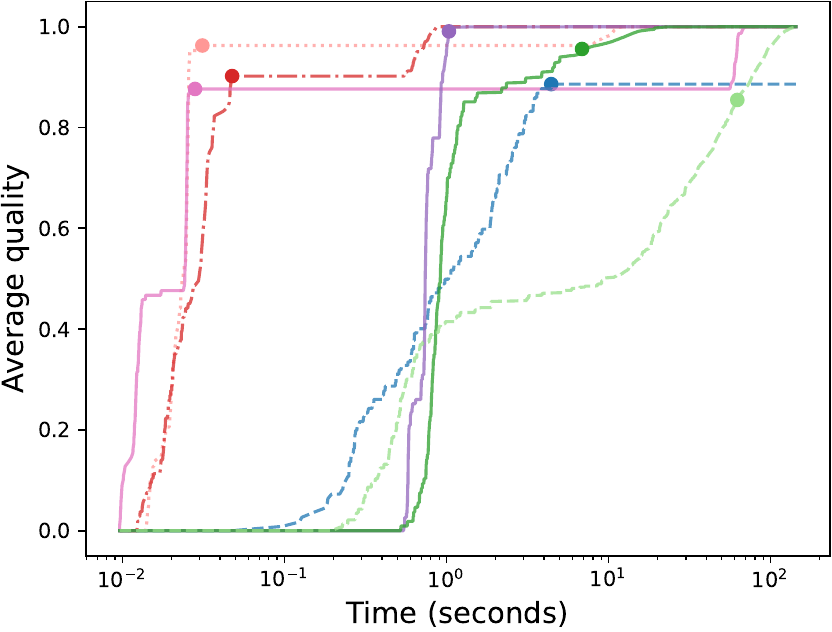}%
\centering%
\includegraphics[width=\plotwidth{}]{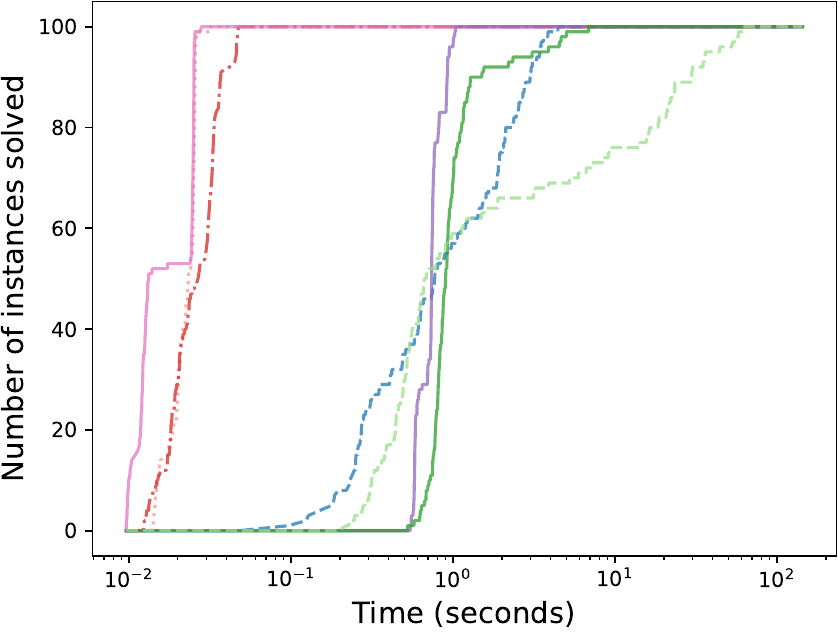}%
\centering%
\includegraphics[width=\plotwidth{}]{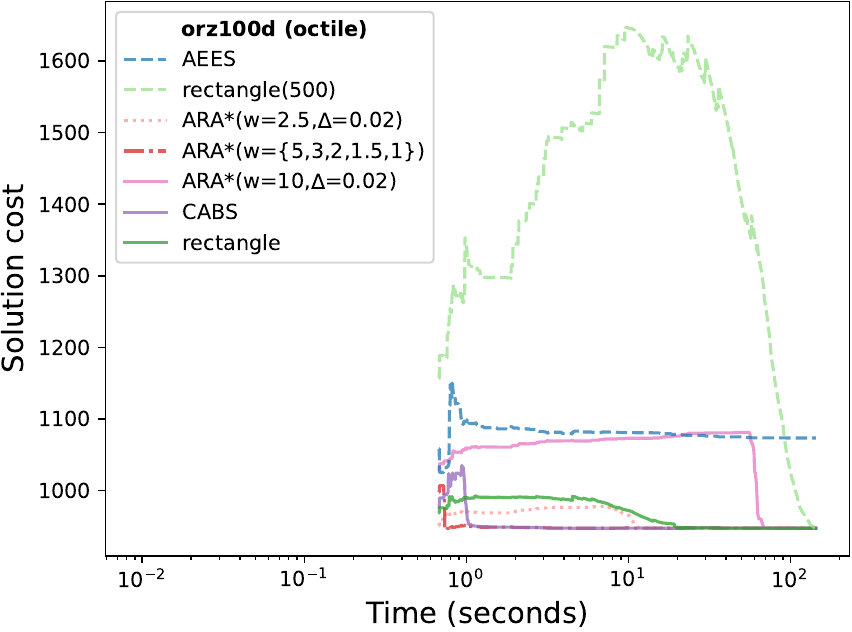}%
\caption{orz100d (octile cost)}%
\end{figure*}

\begin{figure*}[h!]%
\centering%
\includegraphics[width=\plotwidth{}]{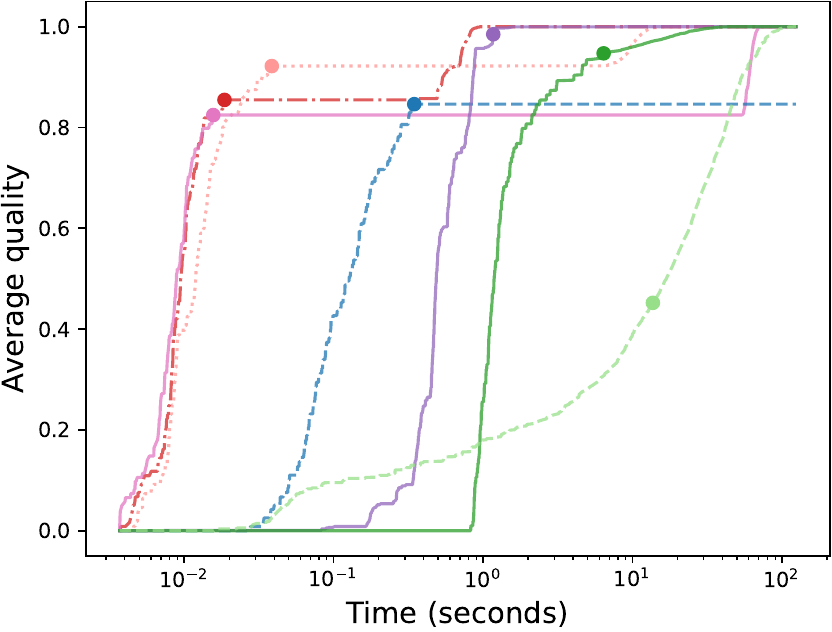}%
\centering%
\includegraphics[width=\plotwidth{}]{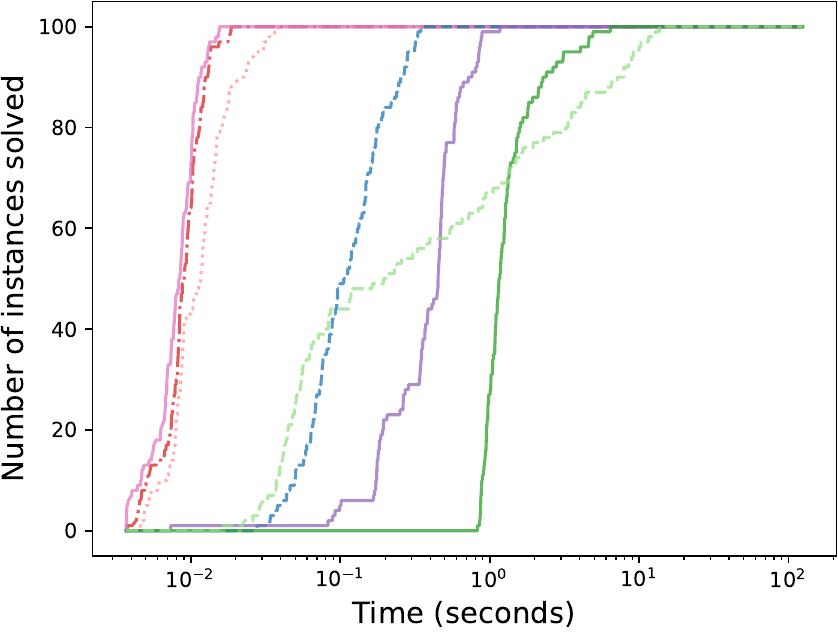}%
\centering%
\includegraphics[width=\plotwidth{}]{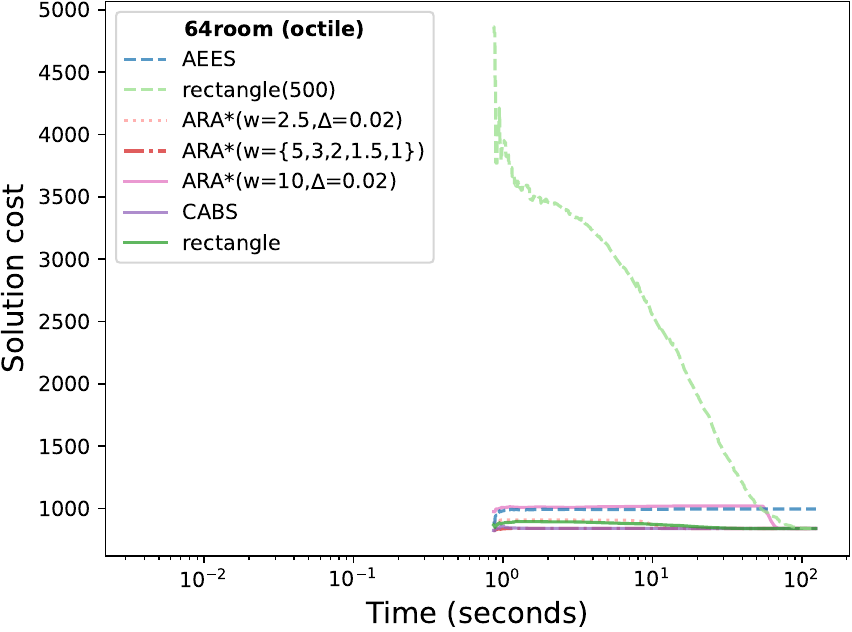}%
\caption{64 room (octile cost)}%
\end{figure*}

\clearpage

\subsection{Depth-first-style Algorithms} \label{sec:dfs}

\begin{figure*}[h!]%
\centering%
\includegraphics[width=\plotwidth{}]{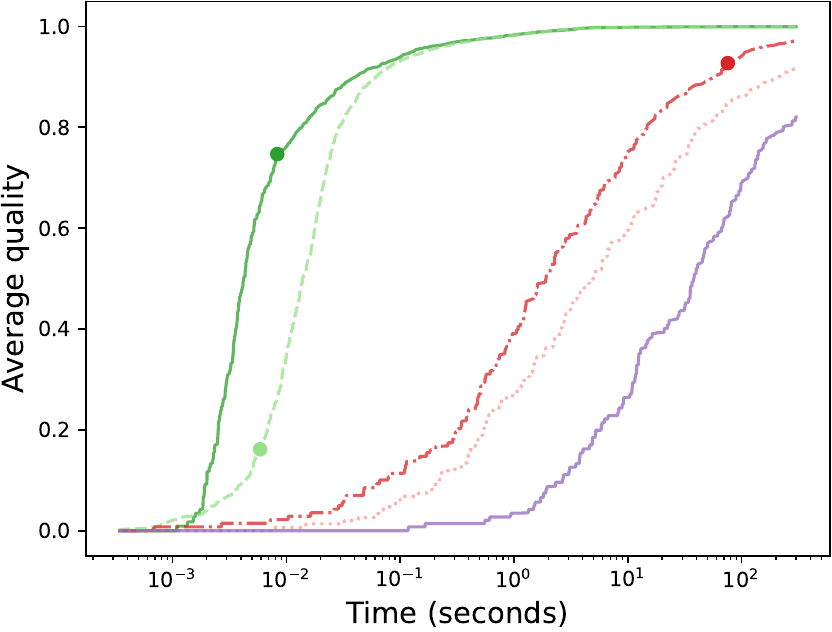}%
\centering%
\includegraphics[width=\plotwidth{}]{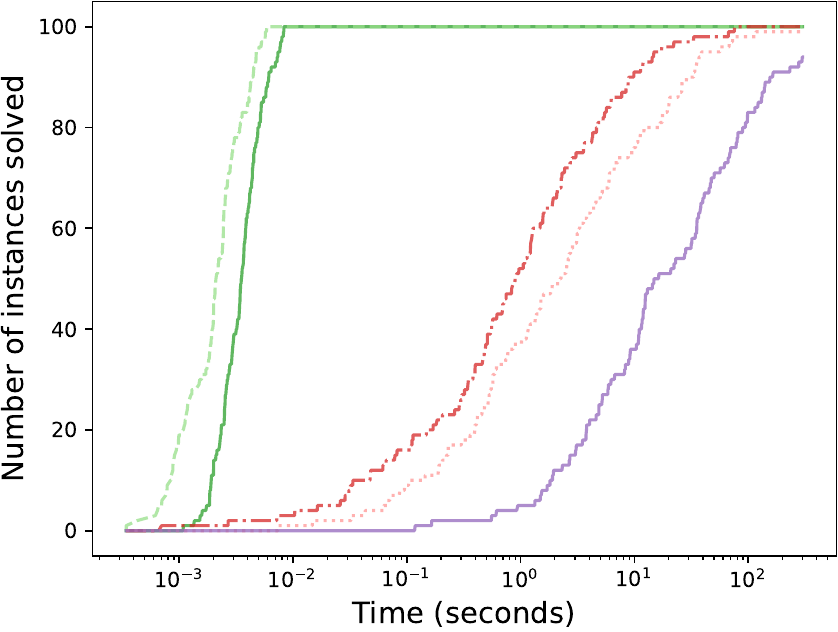}%
\centering%
\includegraphics[width=\plotwidth{}]{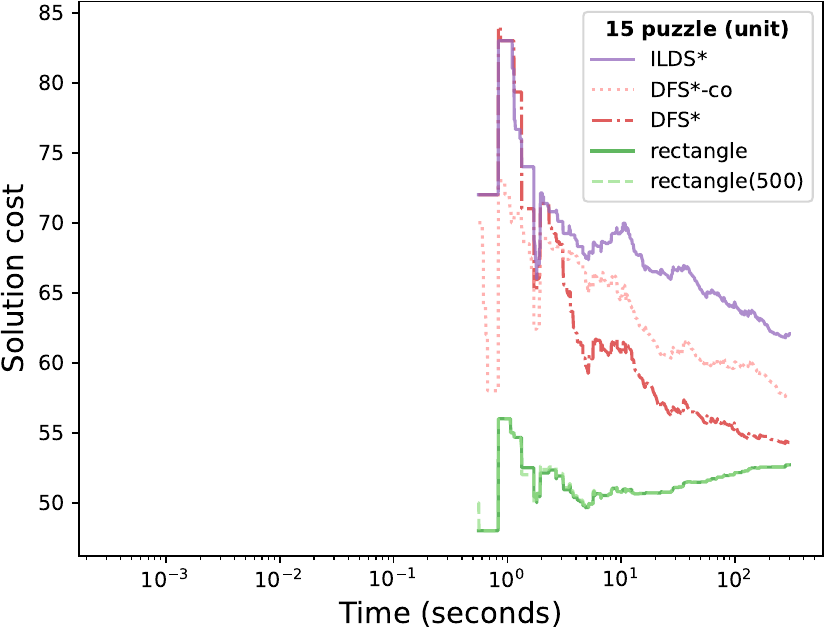}%
\caption{15 puzzle (unit cost)}%
\end{figure*}

\begin{figure*}[h!]%
\centering%
\includegraphics[width=\plotwidth{}]{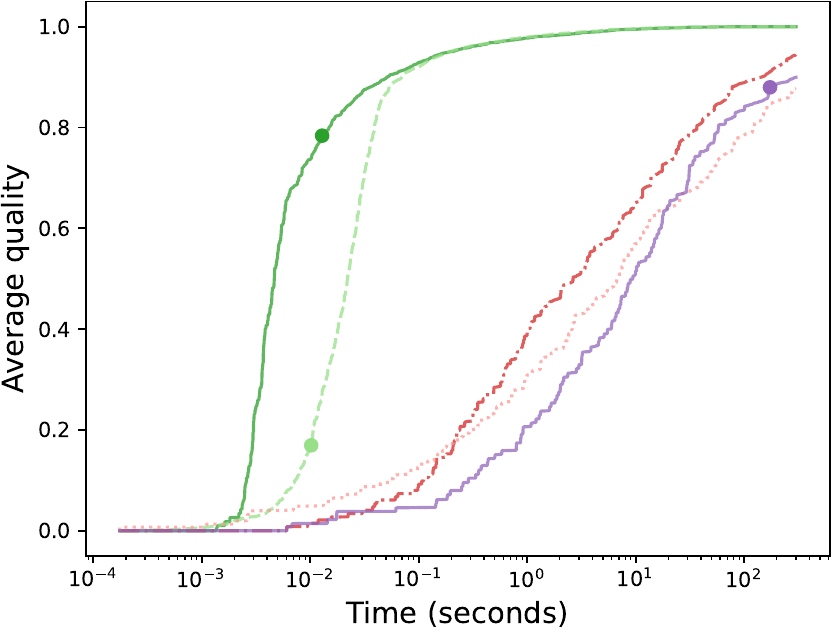}%
\centering%
\includegraphics[width=\plotwidth{}]{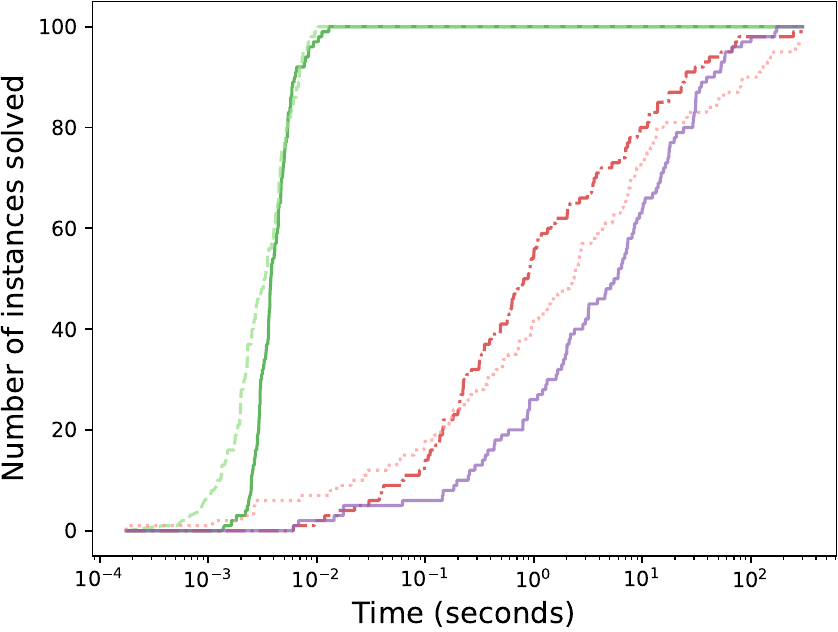}%
\centering%
\includegraphics[width=\plotwidth{}]{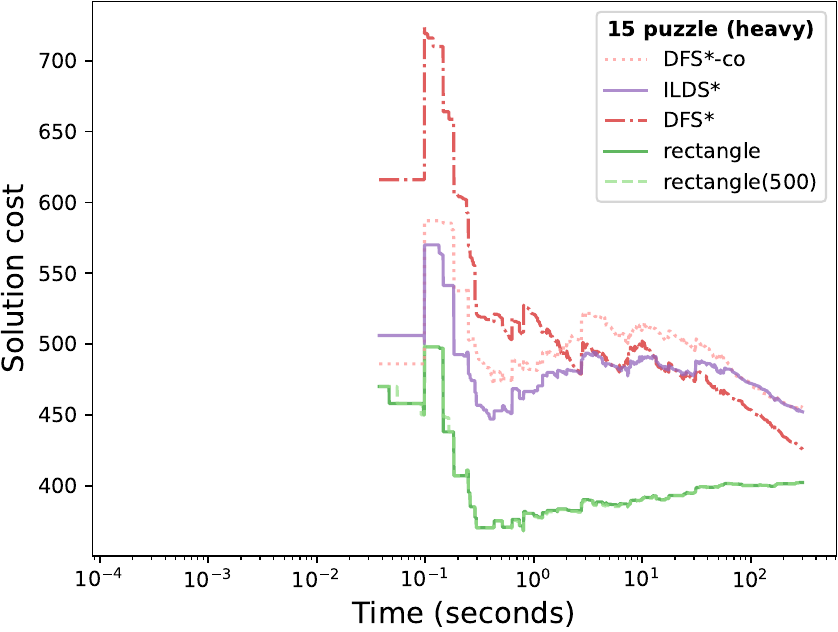}%
\caption{15 puzzle (heavy cost)}%
\end{figure*}

\begin{figure*}[h!]%
\centering%
\includegraphics[width=\plotwidth{}]{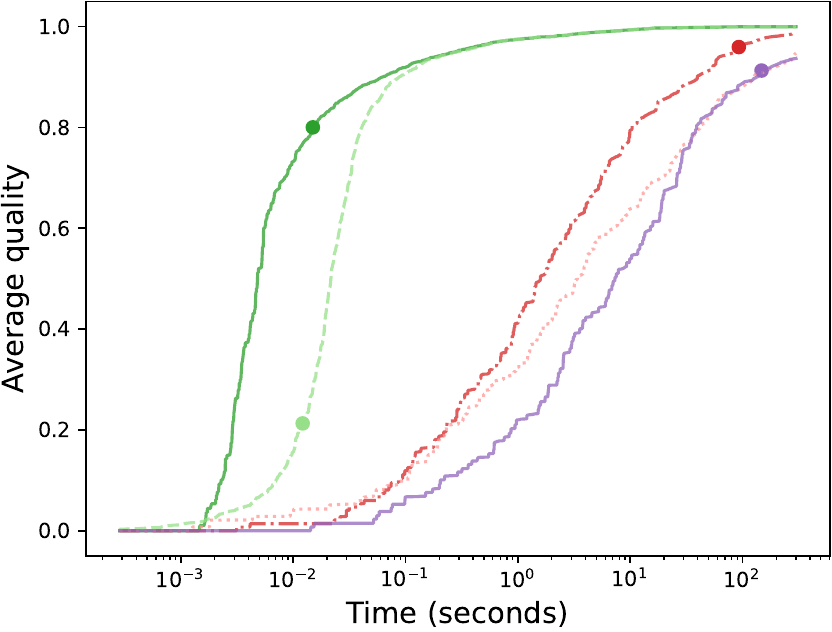}%
\centering%
\includegraphics[width=\plotwidth{}]{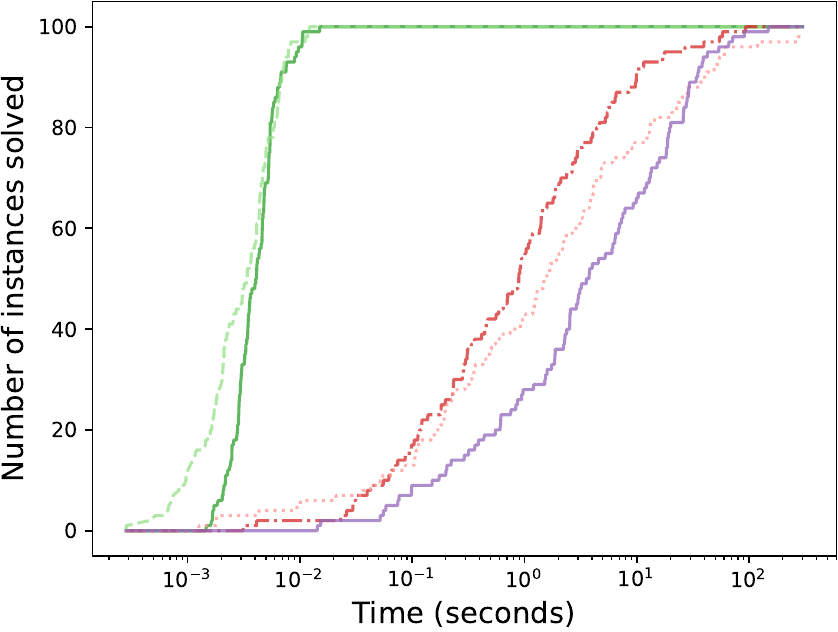}%
\centering%
\includegraphics[width=\plotwidth{}]{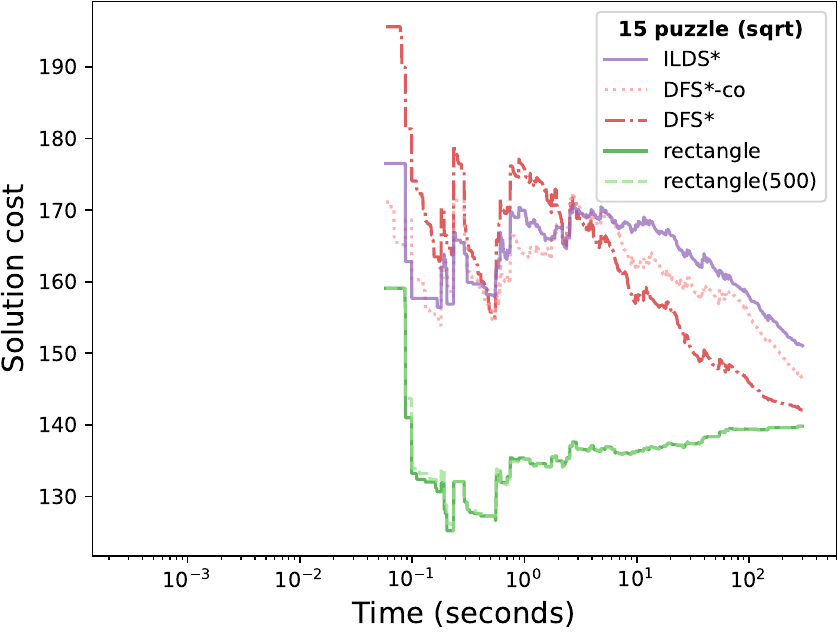}%
\caption{15 puzzle (sqrt cost)}%
\end{figure*}

\begin{figure*}[h!]%
\centering%
\includegraphics[width=\plotwidth{}]{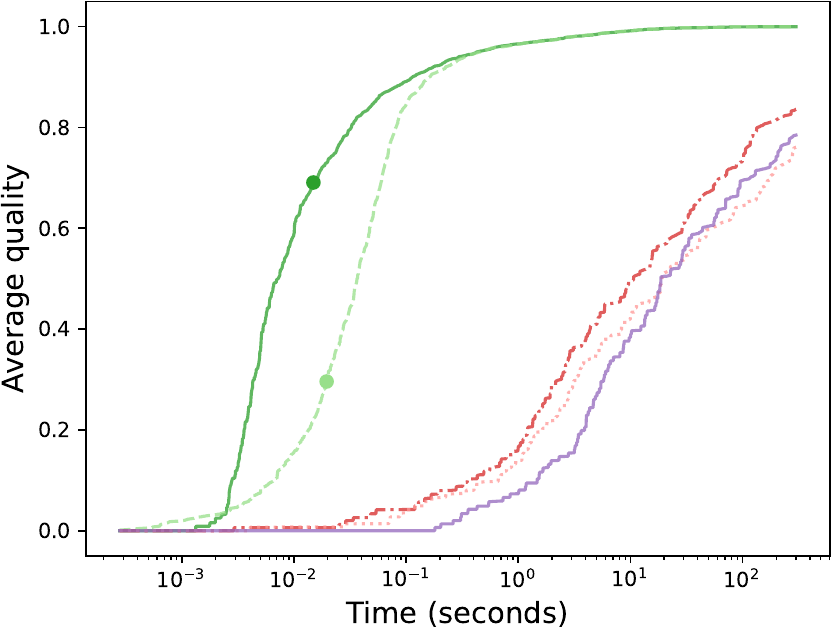}%
\centering%
\includegraphics[width=\plotwidth{}]{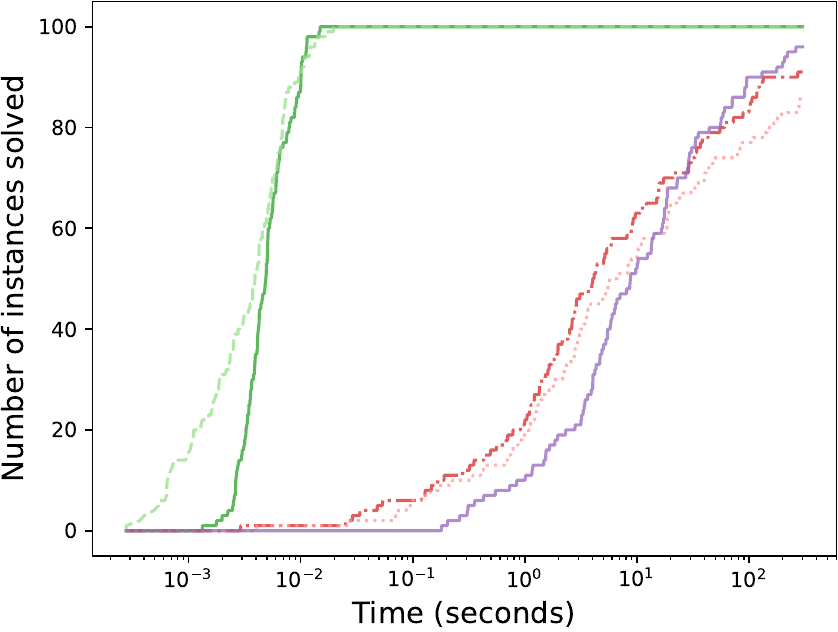}%
\centering%
\includegraphics[width=\plotwidth{}]{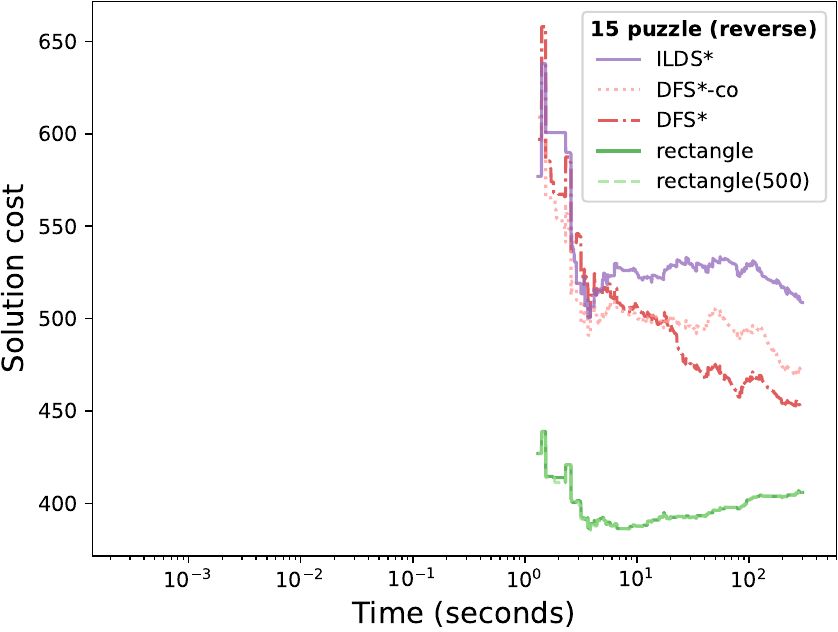}%
\caption{15 puzzle (reverse cost)}%
\end{figure*}

\begin{figure*}[h!]%
\centering%
\includegraphics[width=\plotwidth{}]{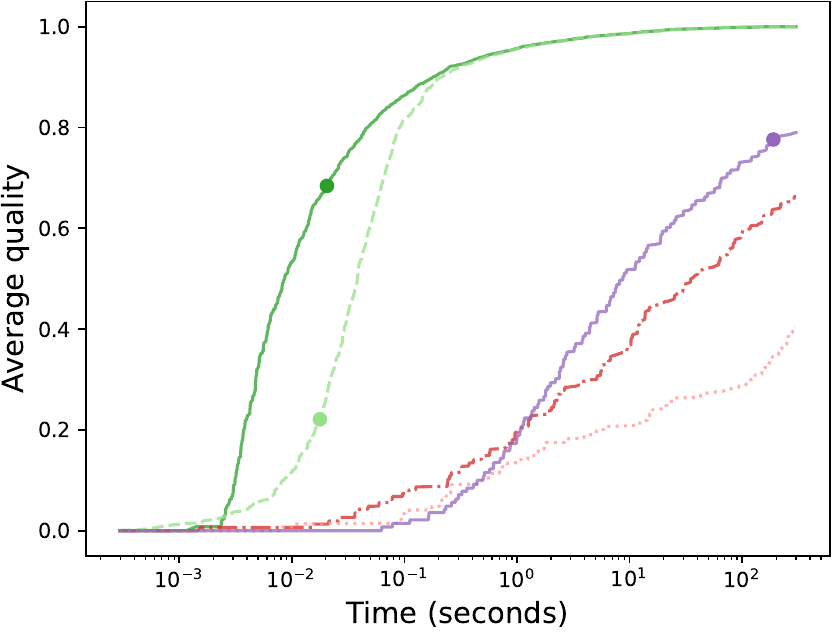}%
\centering%
\includegraphics[width=\plotwidth{}]{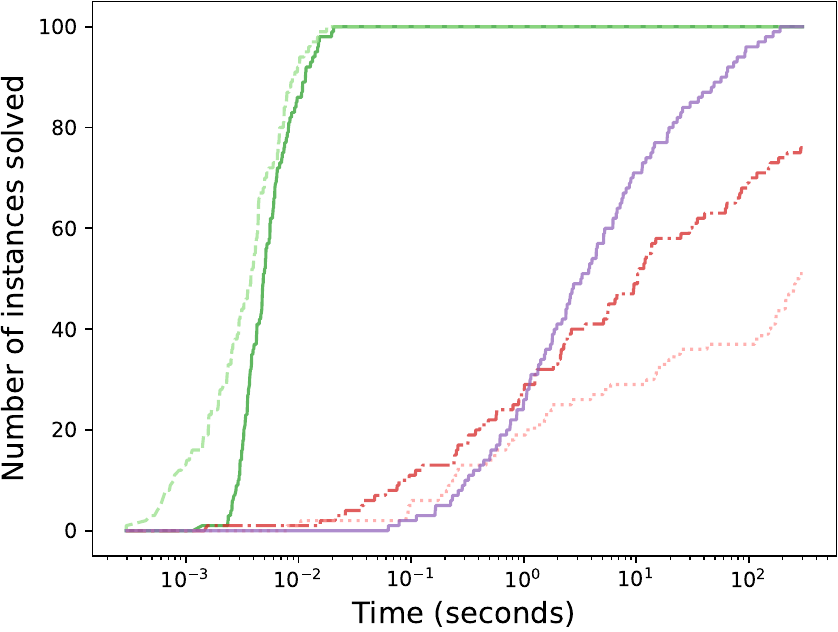}%
\centering%
\includegraphics[width=\plotwidth{}]{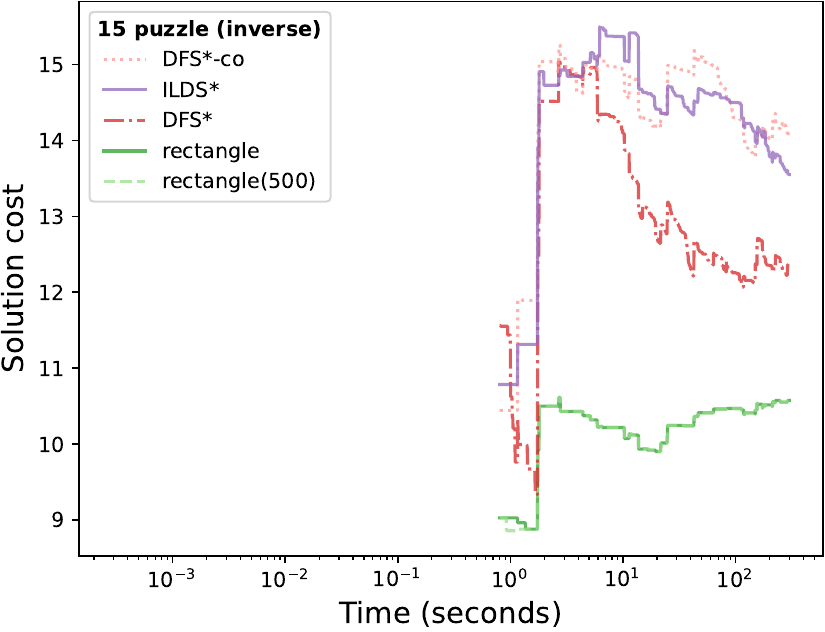}%
\caption{15 puzzle (inverse cost)}%
\end{figure*}

\begin{figure*}[h!]%
\centering%
\includegraphics[width=\plotwidth{}]{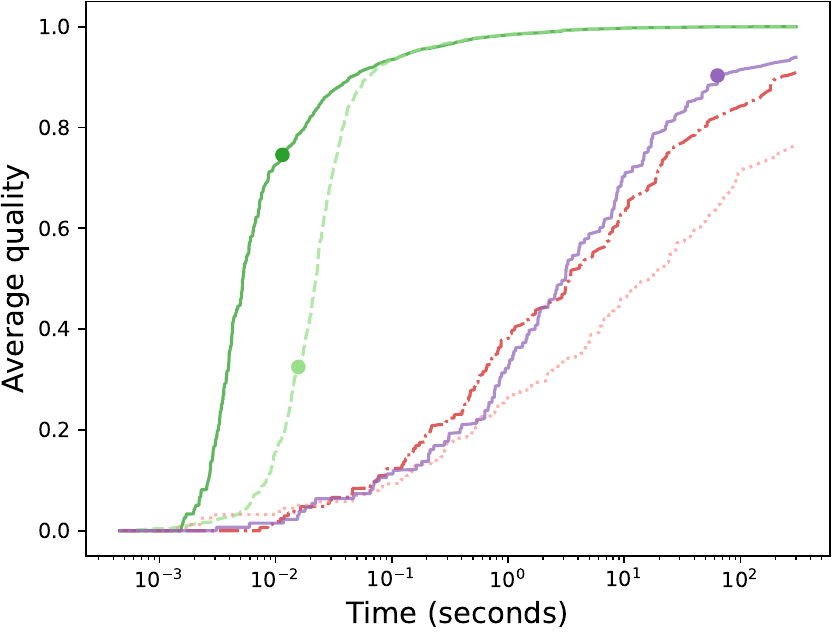}%
\centering%
\includegraphics[width=\plotwidth{}]{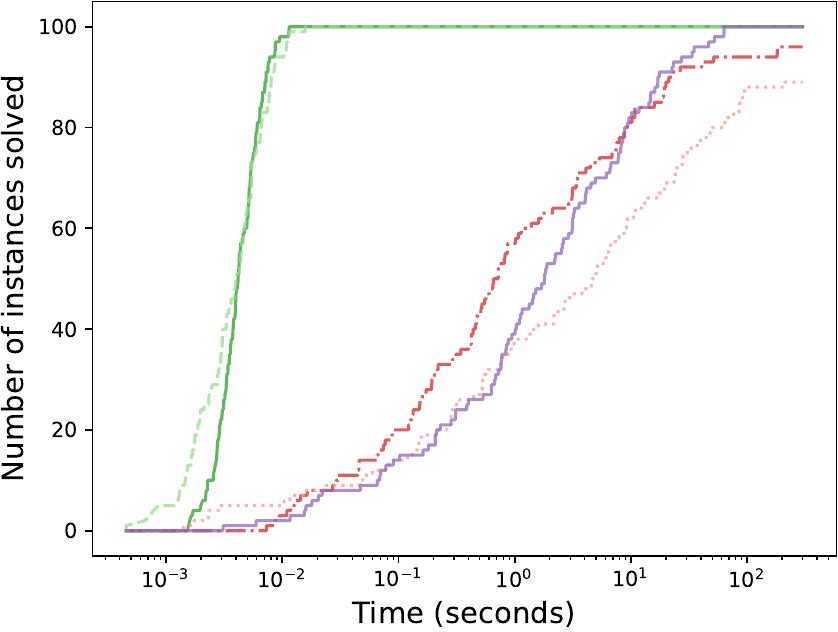}%
\centering%
\includegraphics[width=\plotwidth{}]{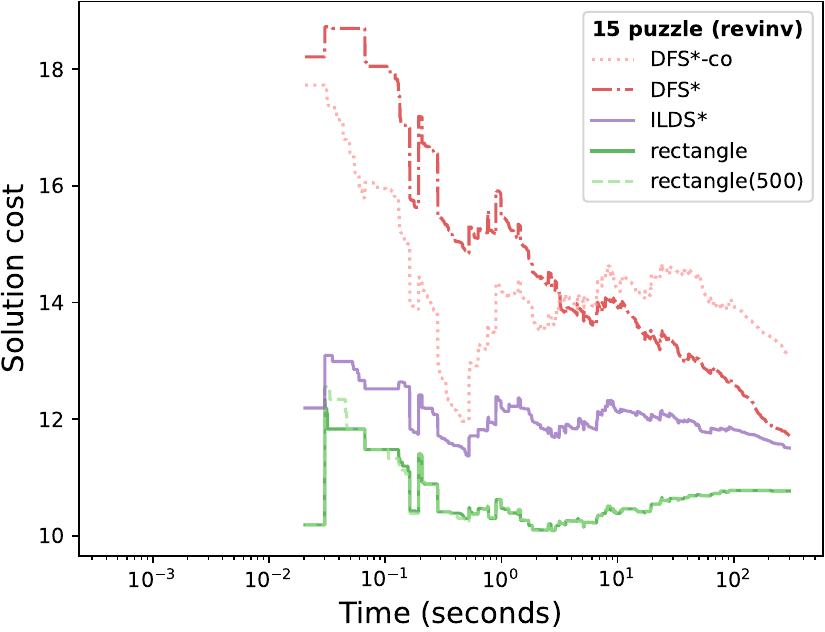}%
\caption{15 puzzle (reverse inverse cost)}%
\end{figure*}

\begin{figure*}[h!]%
\centering%
\includegraphics[width=\plotwidth{}]{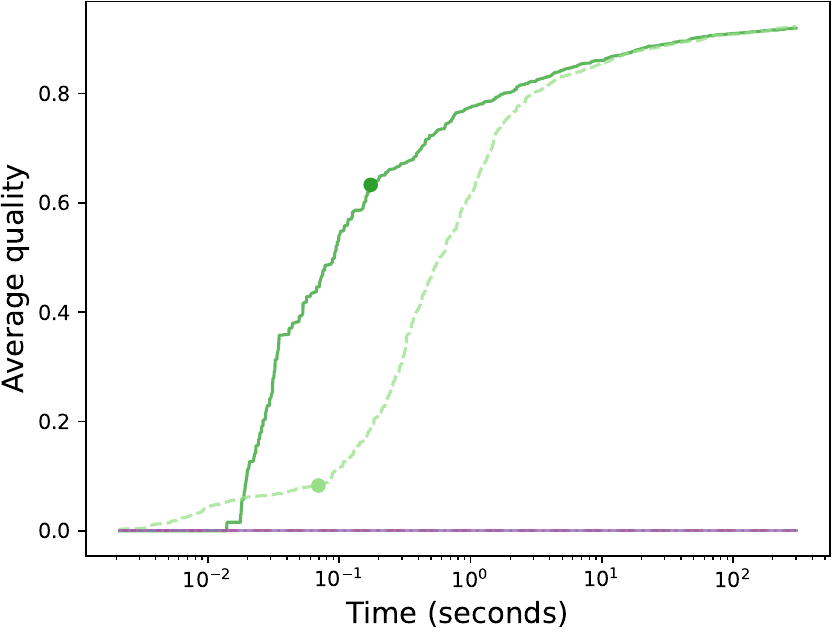}%
\centering%
\includegraphics[width=\plotwidth{}]{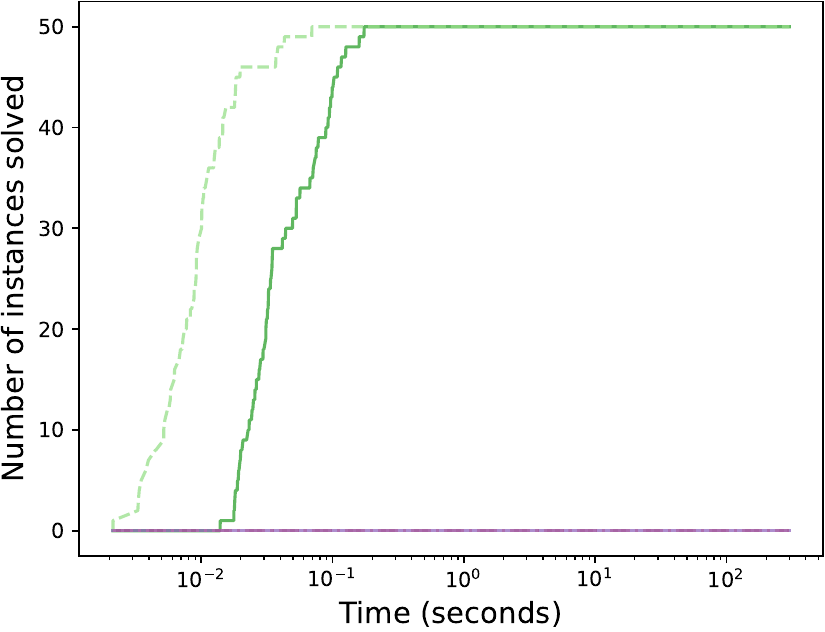}%
\centering%
\includegraphics[width=\plotwidth{}]{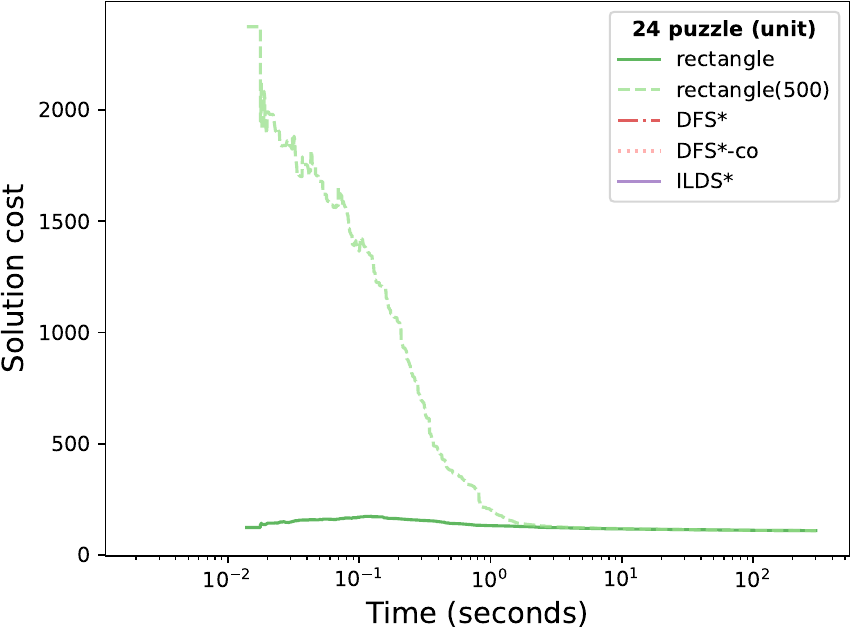}%
\caption{24 puzzle (unit cost)}%
\end{figure*}

\begin{figure*}[h!]%
\centering%
\includegraphics[width=\plotwidth{}]{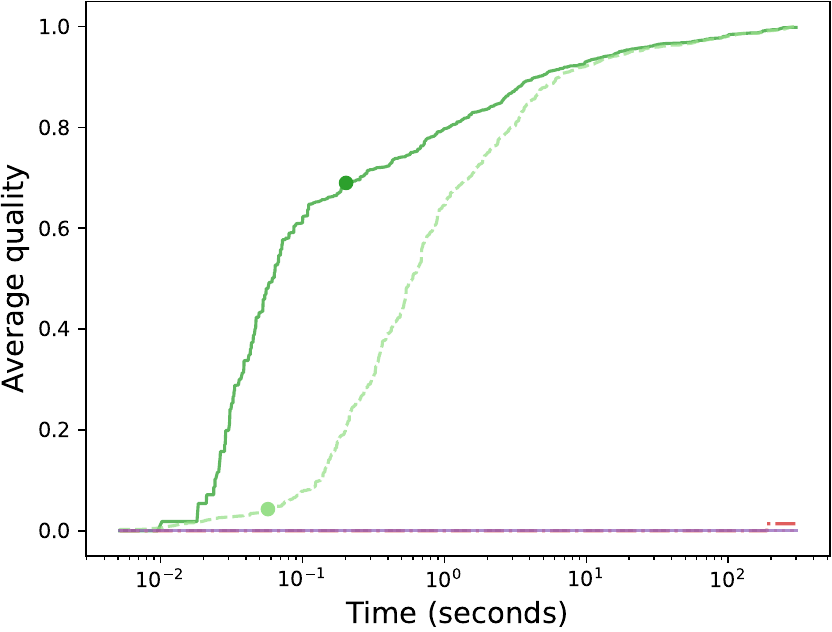}%
\centering%
\includegraphics[width=\plotwidth{}]{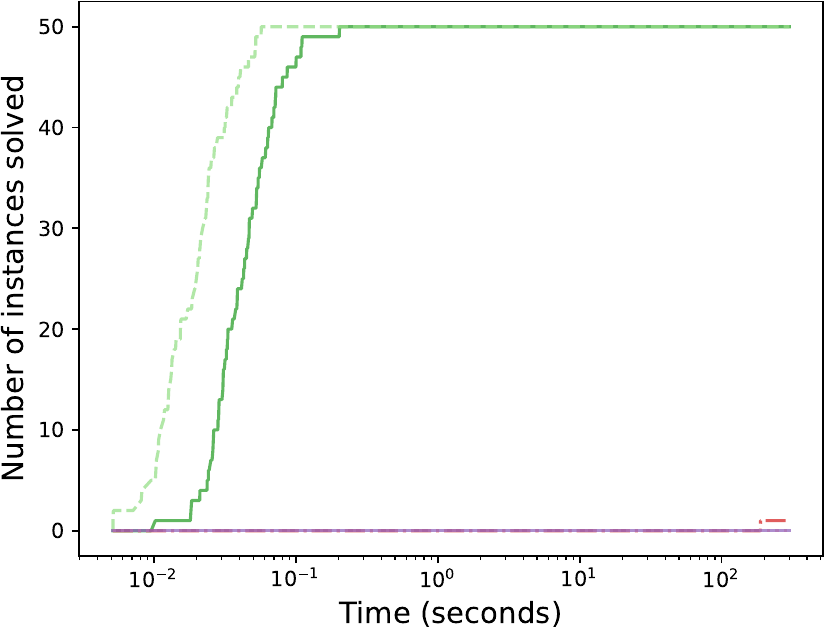}%
\centering%
\includegraphics[width=\plotwidth{}]{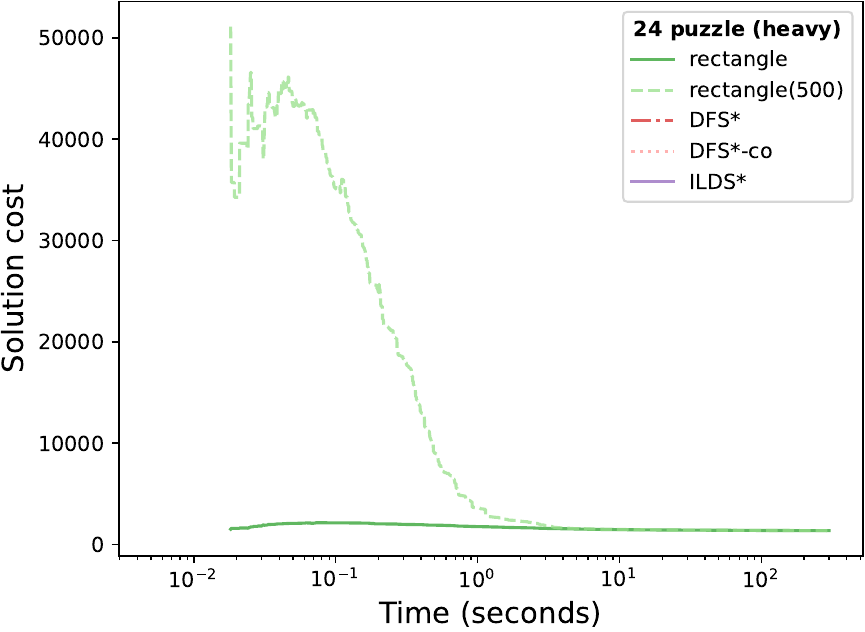}%
\caption{24 puzzle (heavy cost)}%
\end{figure*}

\begin{figure*}[h!]%
\centering%
\includegraphics[width=\plotwidth{}]{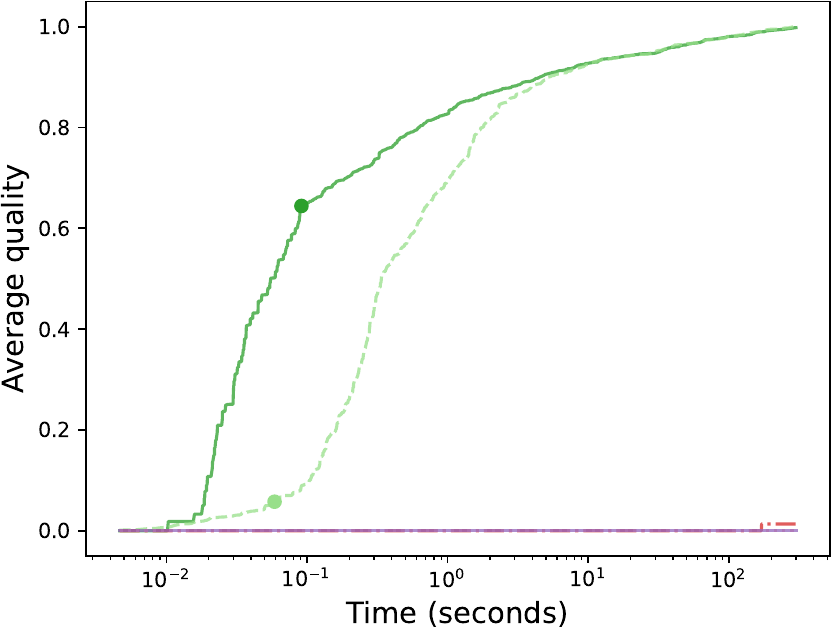}%
\centering%
\includegraphics[width=\plotwidth{}]{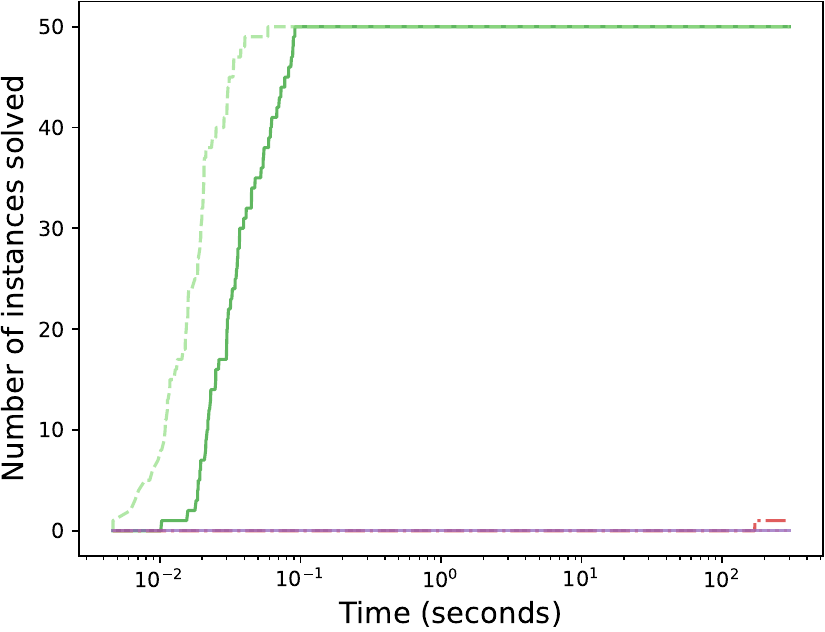}%
\centering%
\includegraphics[width=\plotwidth{}]{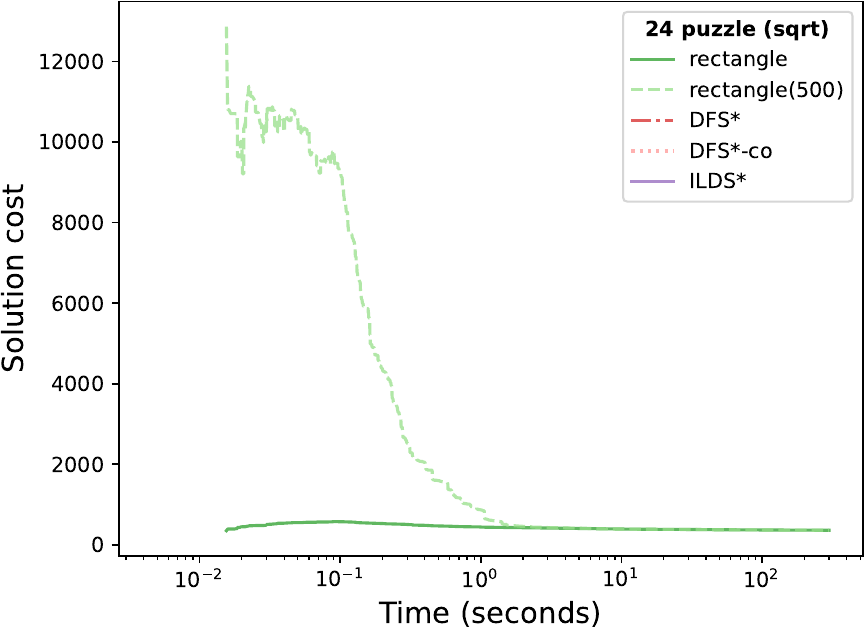}%
\caption{24 puzzle (sqrt cost)}%
\end{figure*}

\begin{figure*}[h!]%
\centering%
\includegraphics[width=\plotwidth{}]{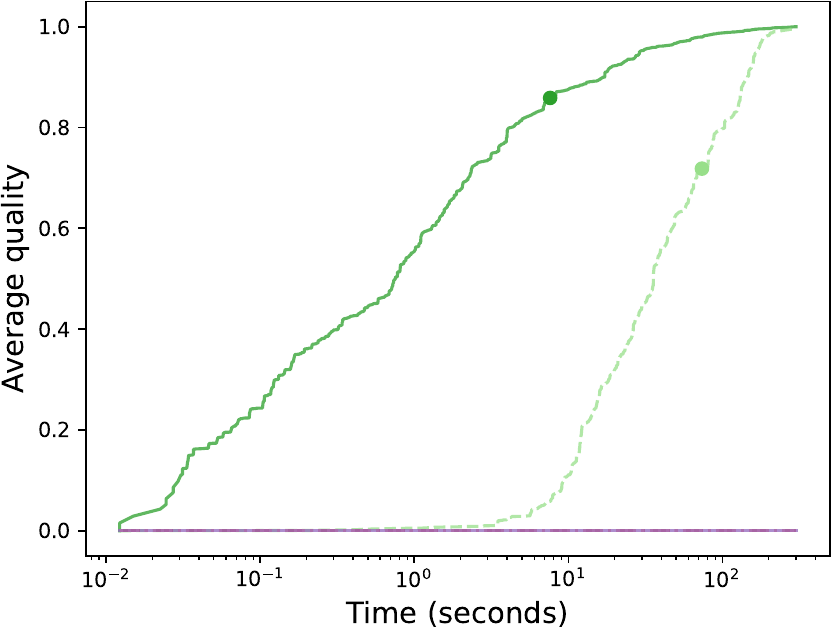}%
\centering%
\includegraphics[width=\plotwidth{}]{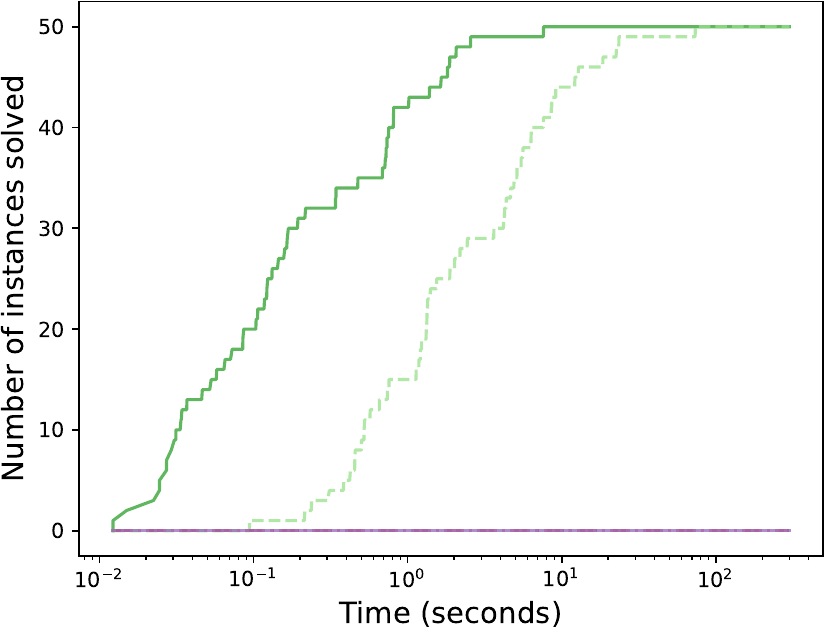}%
\centering%
\includegraphics[width=\plotwidth{}]{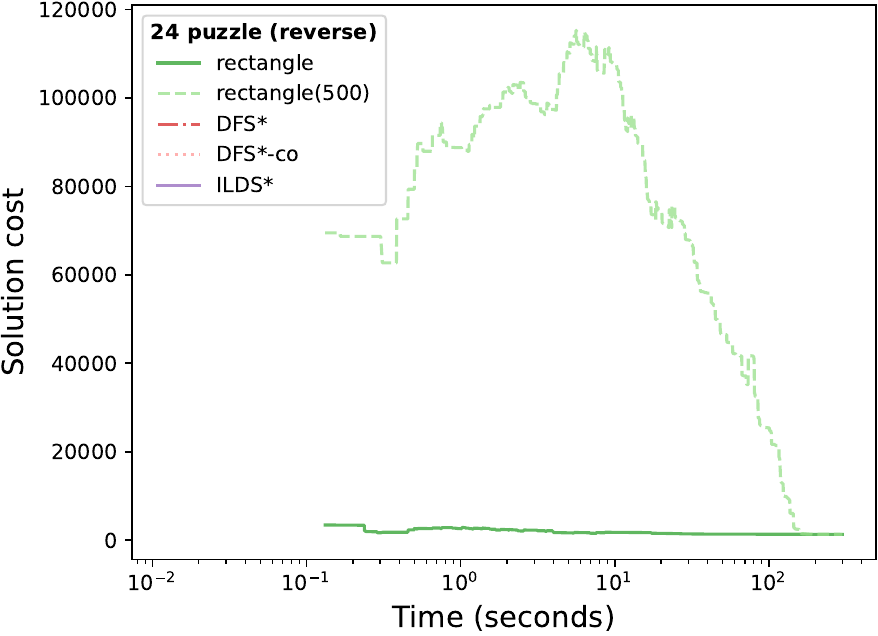}%
\caption{24 puzzle (reverse cost)}%
\end{figure*}

\begin{figure*}[h!]%
\centering%
\includegraphics[width=\plotwidth{}]{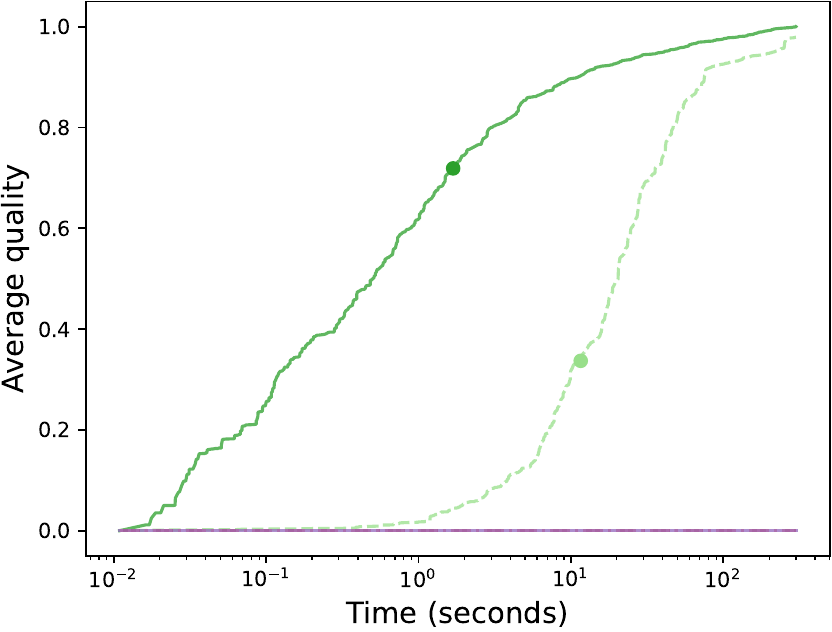}%
\centering%
\includegraphics[width=\plotwidth{}]{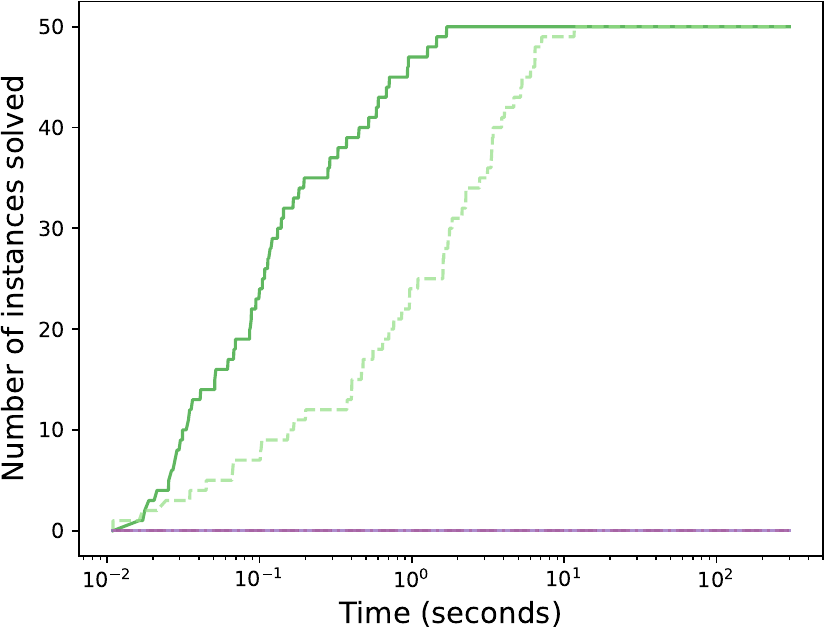}%
\centering%
\includegraphics[width=\plotwidth{}]{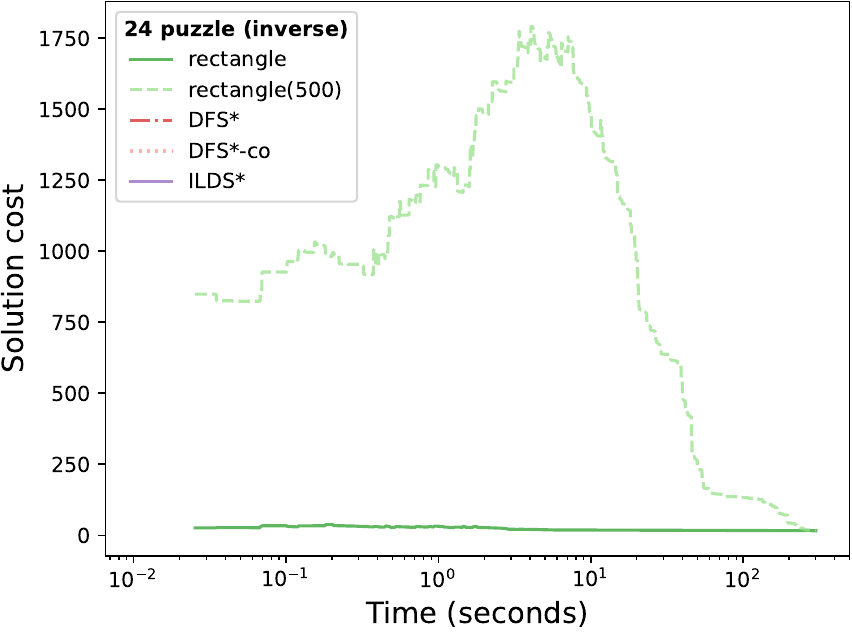}%
\caption{24 puzzle (inverse cost)}%
\end{figure*}

\begin{figure*}[h!]%
\centering%
\includegraphics[width=\plotwidth{}]{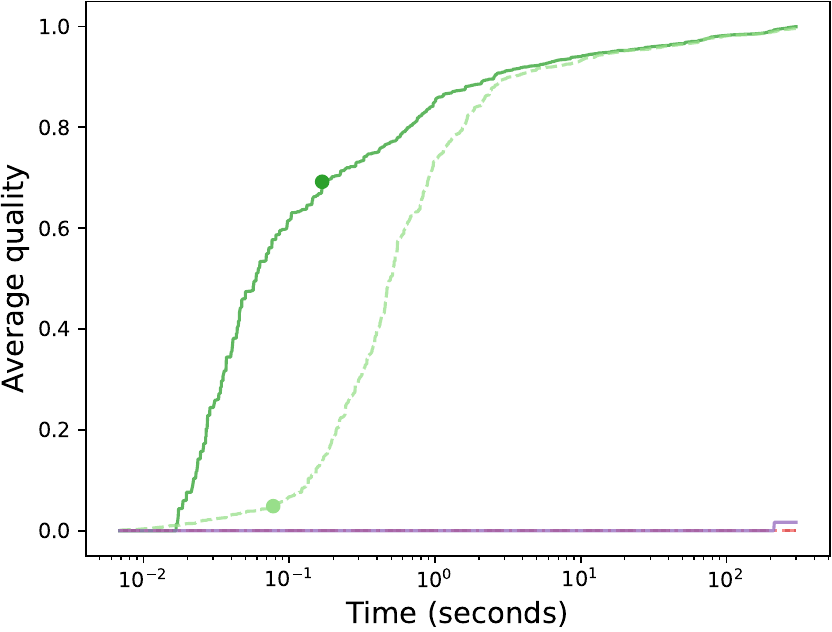}%
\centering%
\includegraphics[width=\plotwidth{}]{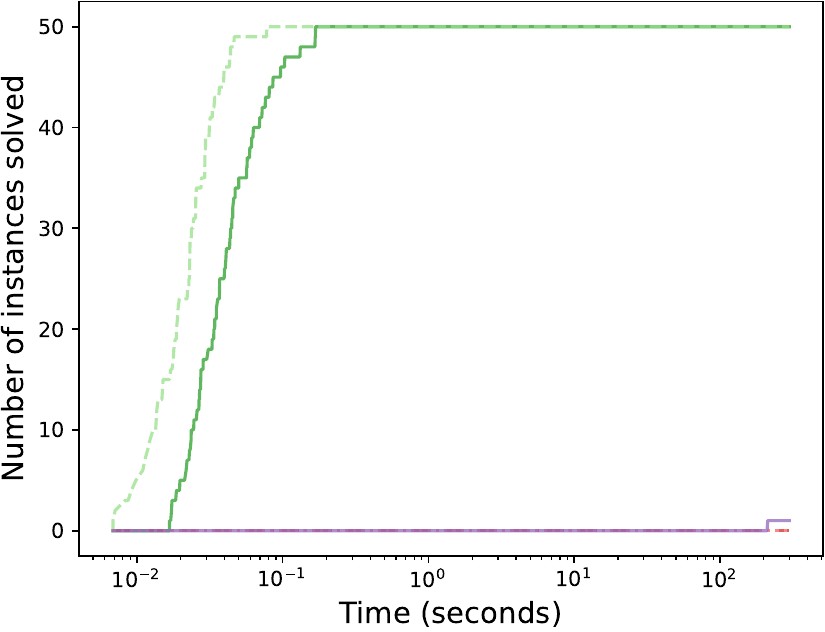}%
\centering%
\includegraphics[width=\plotwidth{}]{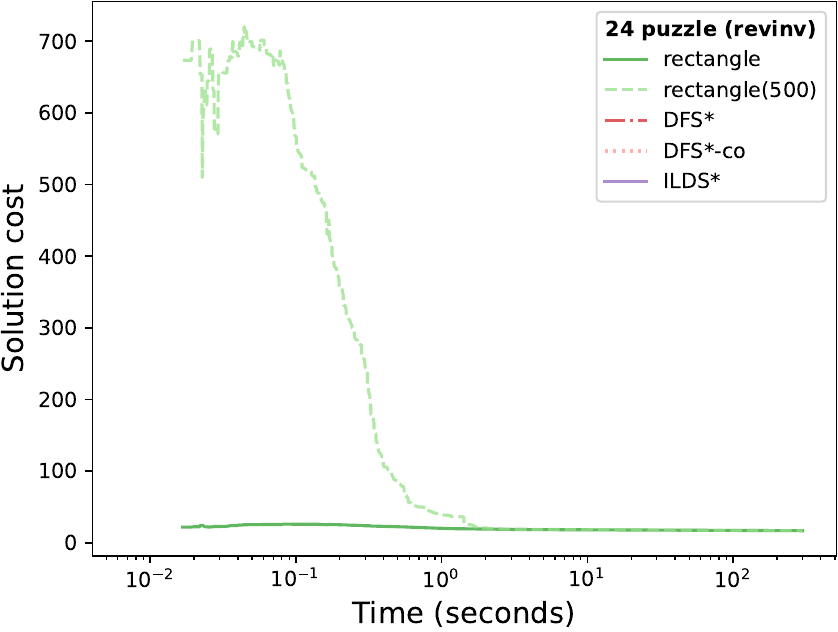}%
\caption{24 puzzle (reverse inverse cost)}%
\end{figure*}

\begin{figure*}[h!]%
\centering%
\includegraphics[width=\plotwidth{}]{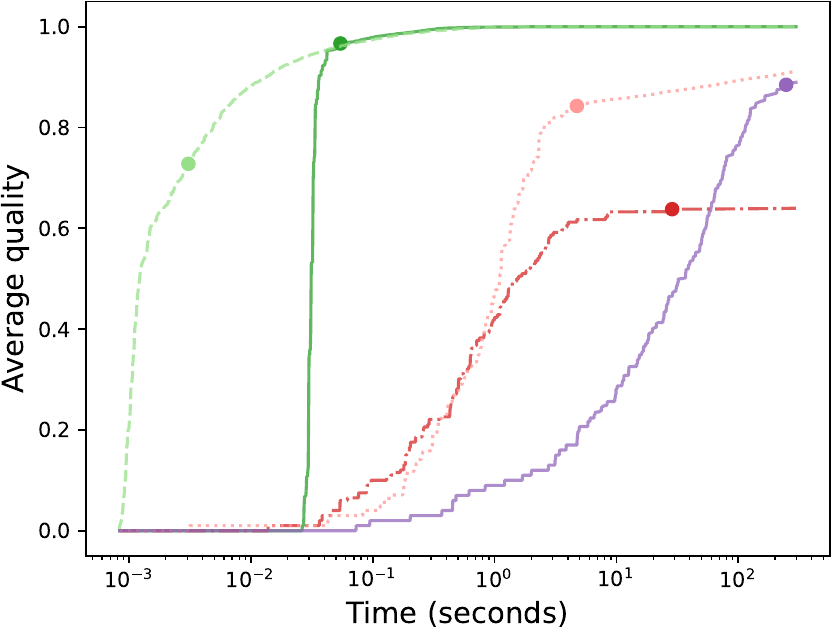}%
\centering%
\includegraphics[width=\plotwidth{}]{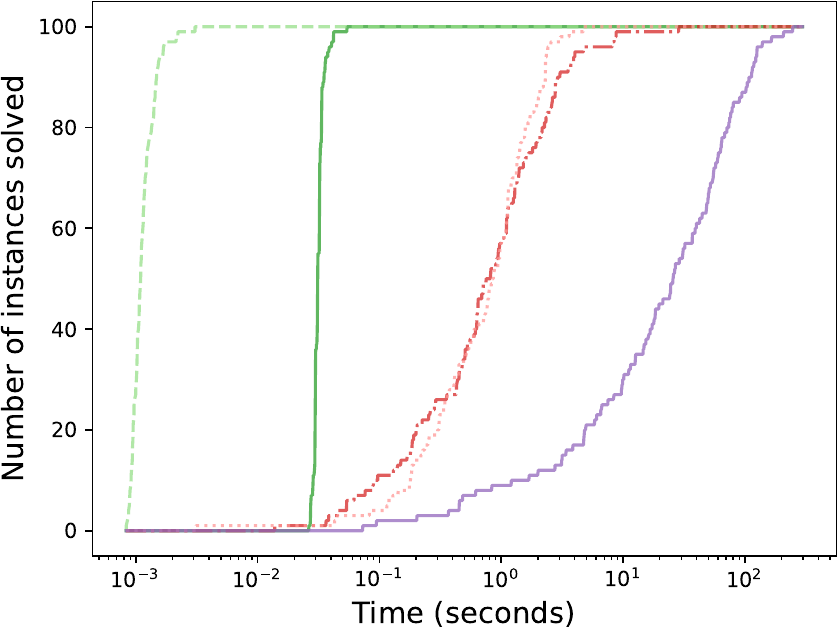}%
\centering%
\includegraphics[width=\plotwidth{}]{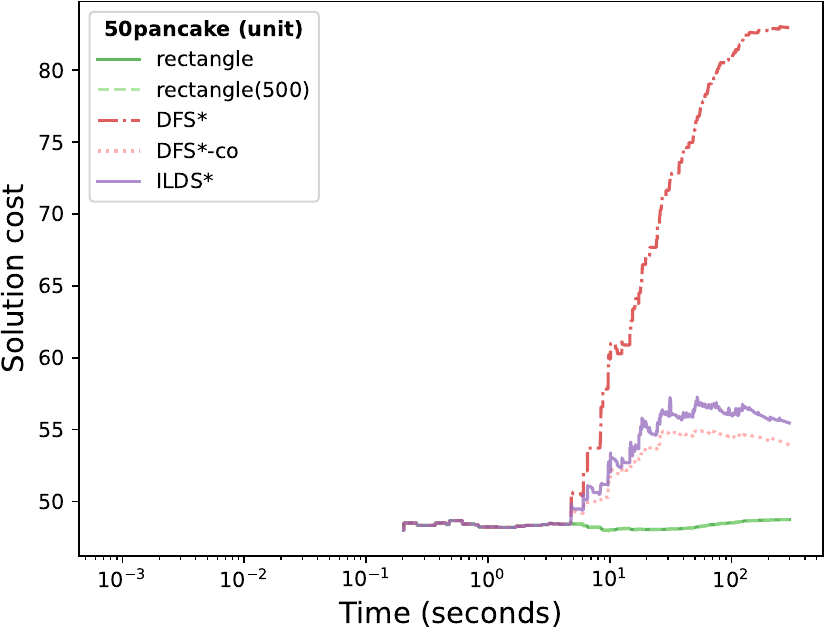}%
\caption{50pancake (unit cost)}%
\end{figure*}

\begin{figure*}[h!]%
\centering%
\includegraphics[width=\plotwidth{}]{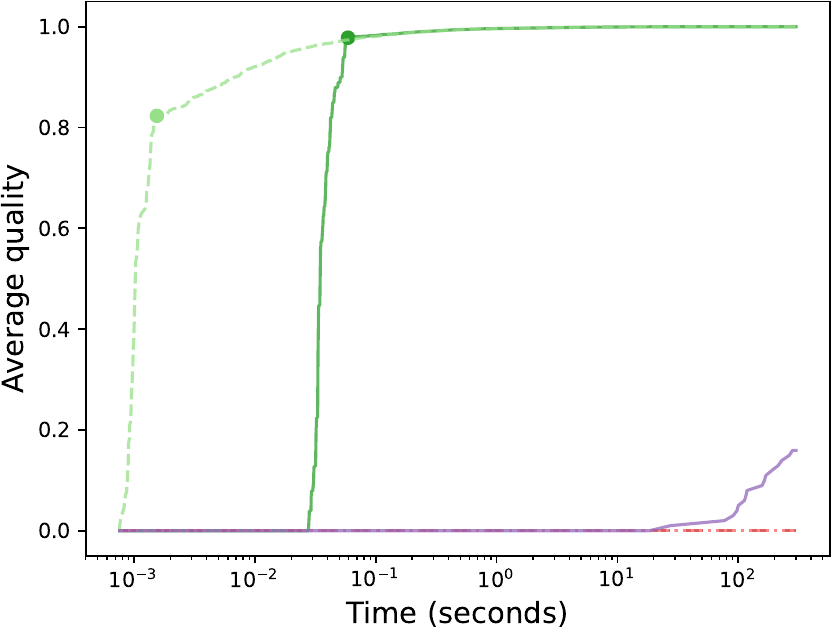}%
\centering%
\includegraphics[width=\plotwidth{}]{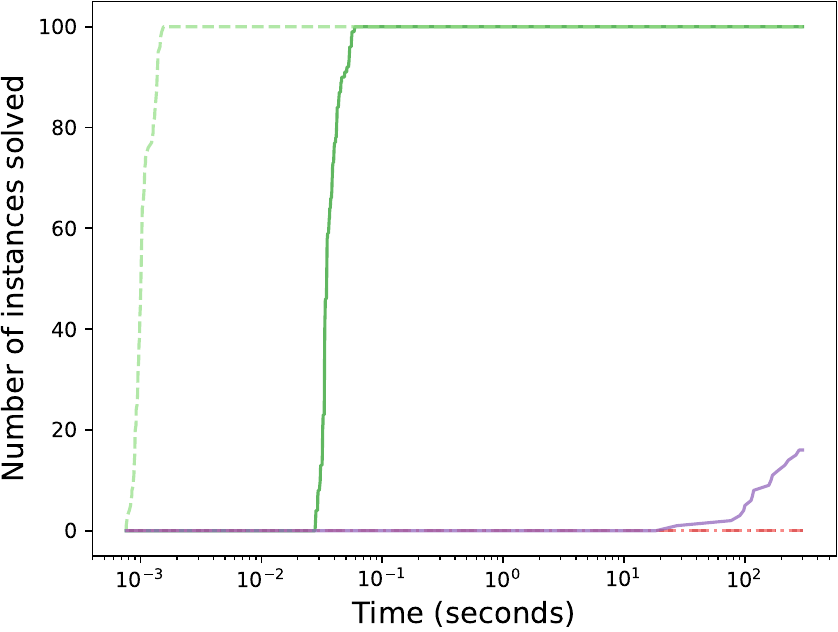}%
\centering%
\includegraphics[width=\plotwidth{}]{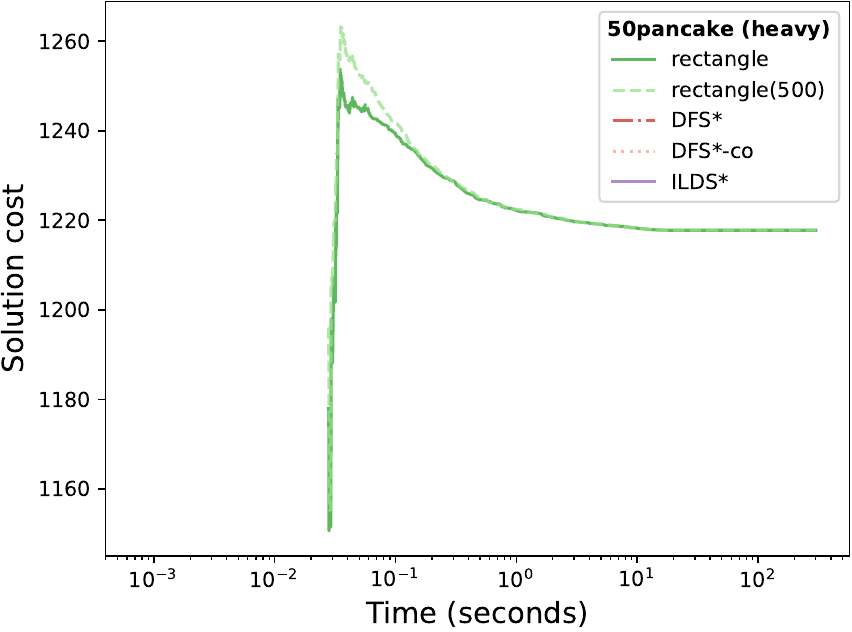}%
\caption{50pancake (heavy cost)}%
\end{figure*}

\begin{figure*}[h!]%
\centering%
\includegraphics[width=\plotwidth{}]{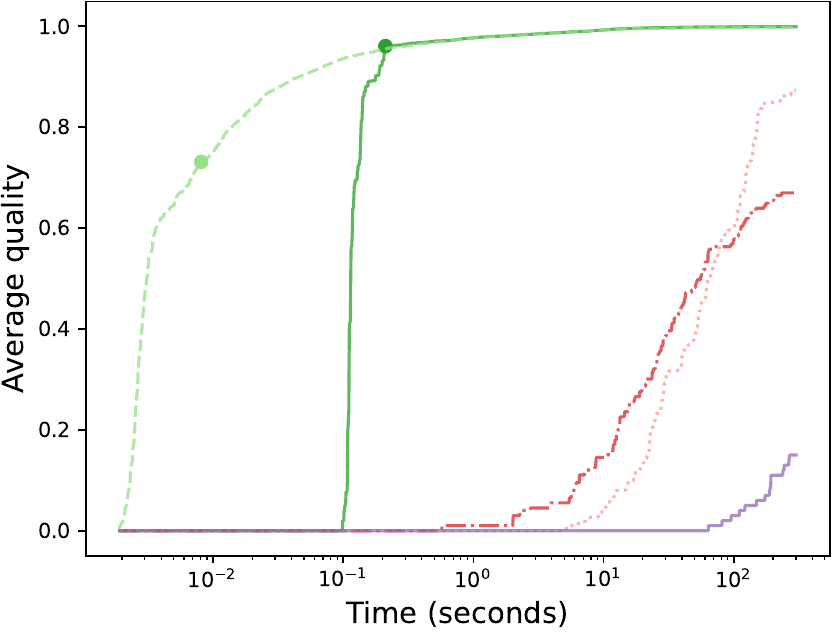}%
\centering%
\includegraphics[width=\plotwidth{}]{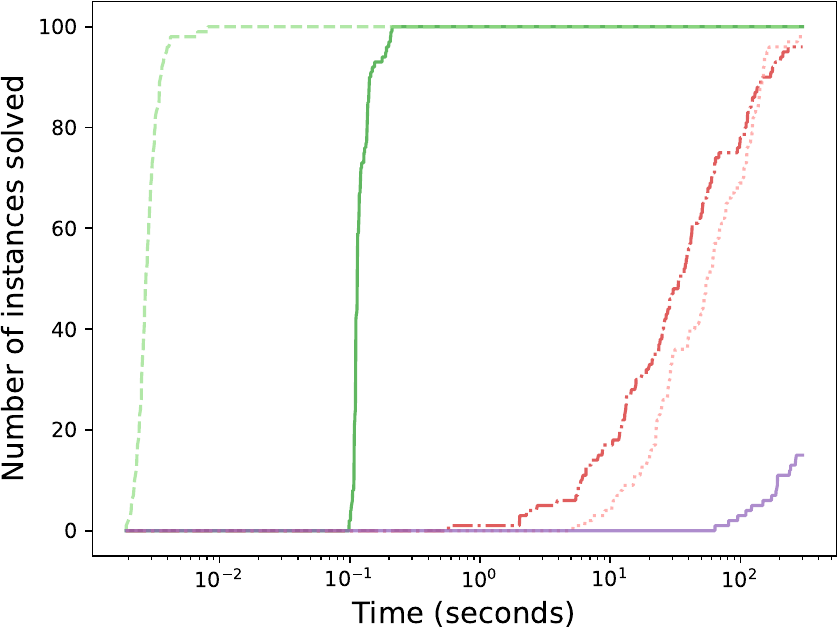}%
\centering%
\includegraphics[width=\plotwidth{}]{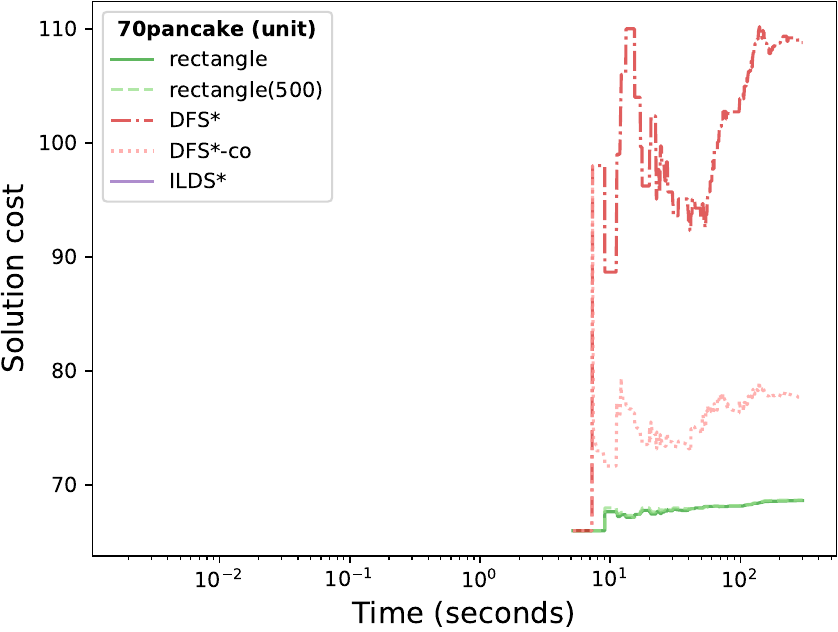}%
\caption{70pancake (unit cost)}%
\end{figure*}

\begin{figure*}[h!]%
\centering%
\includegraphics[width=\plotwidth{}]{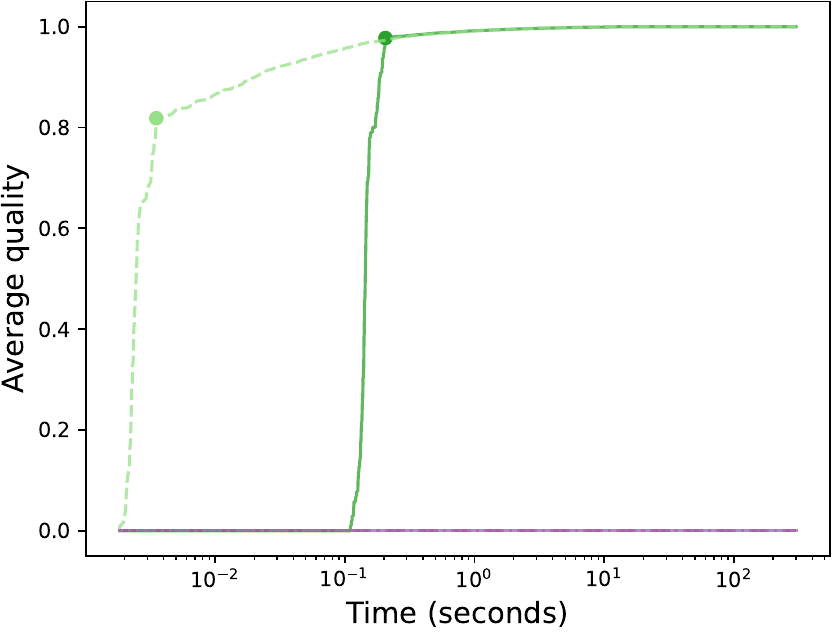}%
\centering%
\includegraphics[width=\plotwidth{}]{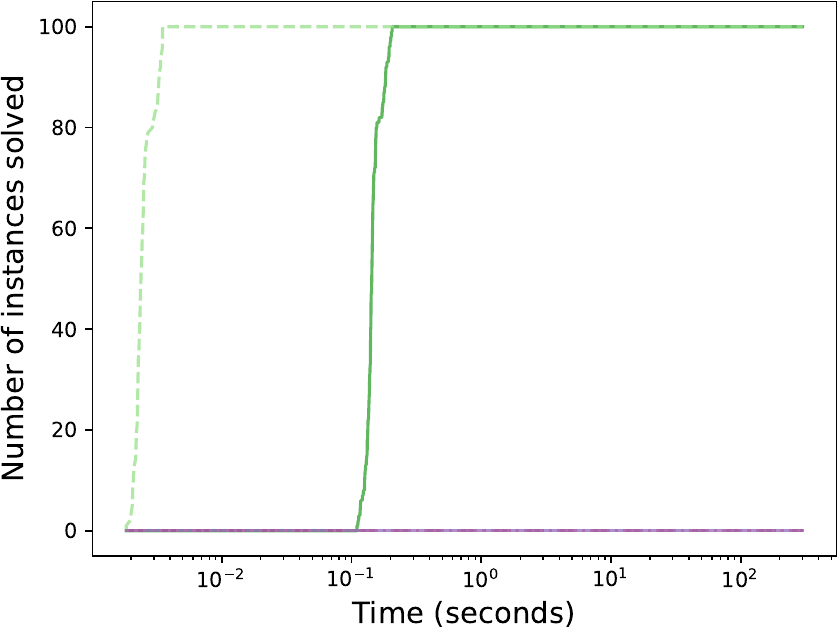}%
\centering%
\includegraphics[width=\plotwidth{}]{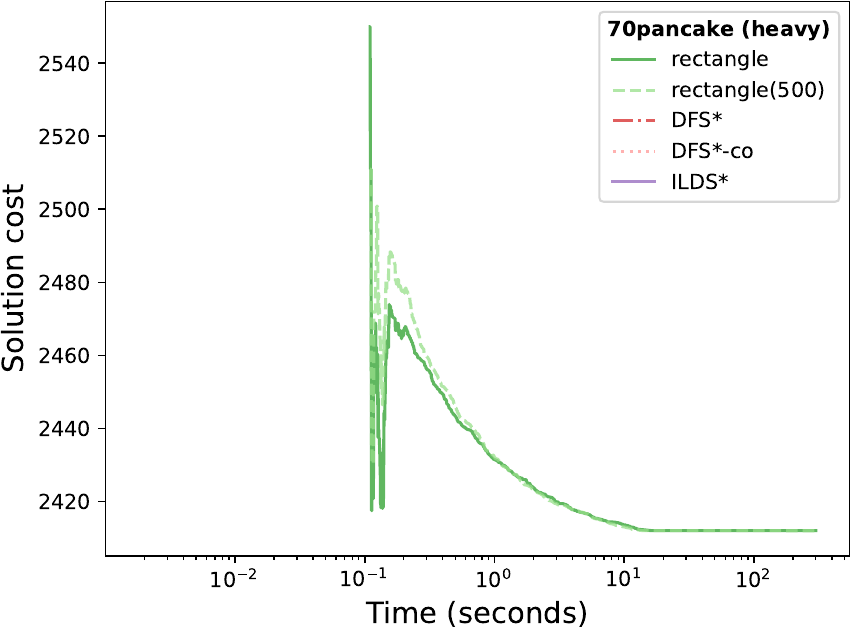}%
\caption{70pancake (heavy cost)}%
\end{figure*}

\begin{figure*}[h!]%
\centering%
\includegraphics[width=\plotwidth{}]{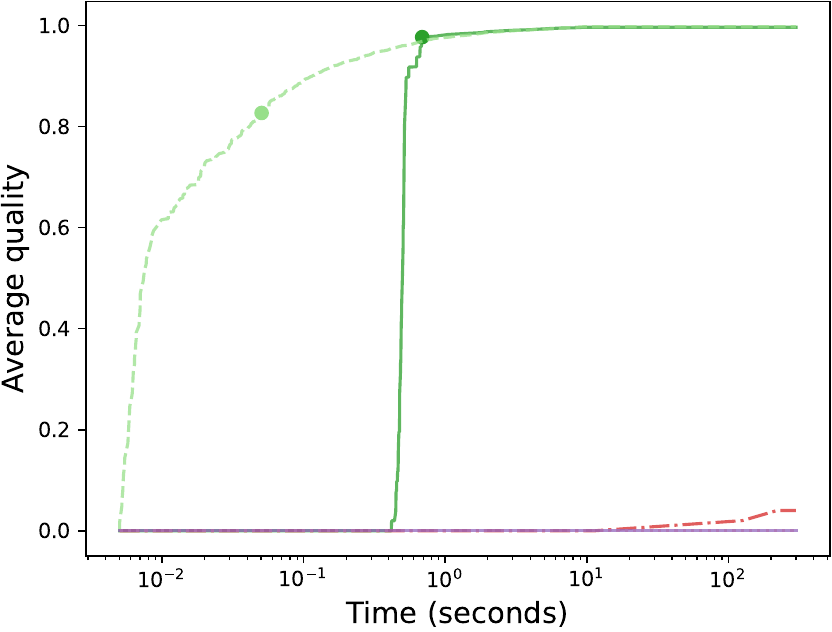}%
\centering%
\includegraphics[width=\plotwidth{}]{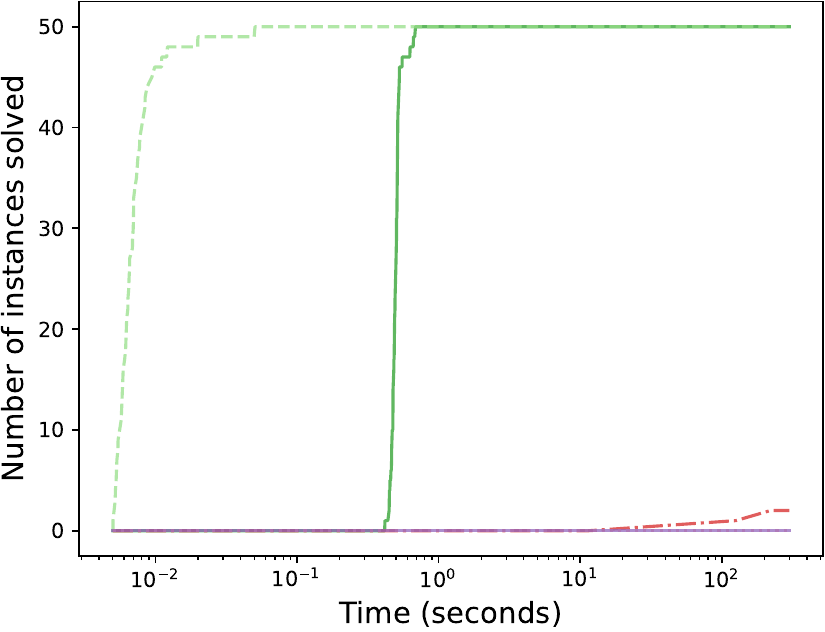}%
\centering%
\includegraphics[width=\plotwidth{}]{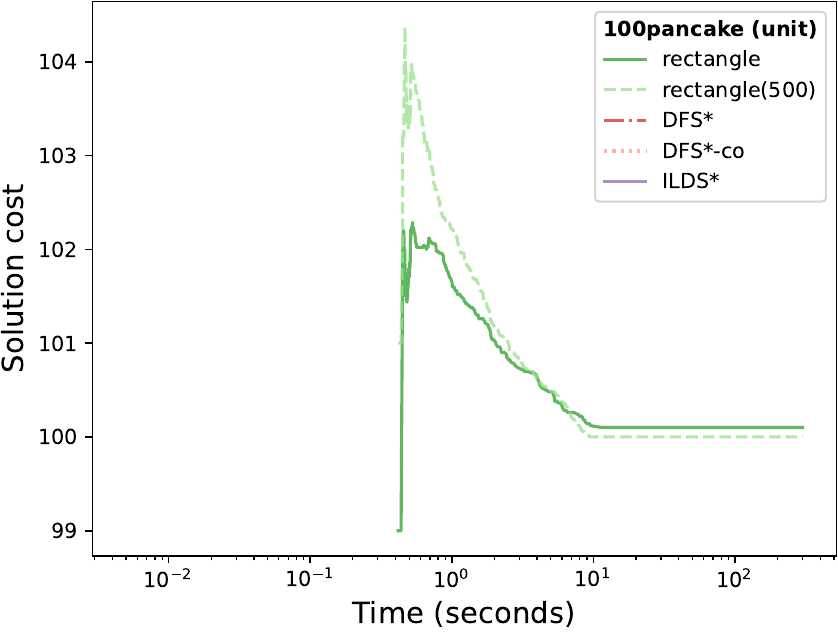}%
\caption{100pancake (unit cost)}%
\end{figure*}

\begin{figure*}[h!]%
\centering%
\includegraphics[width=\plotwidth{}]{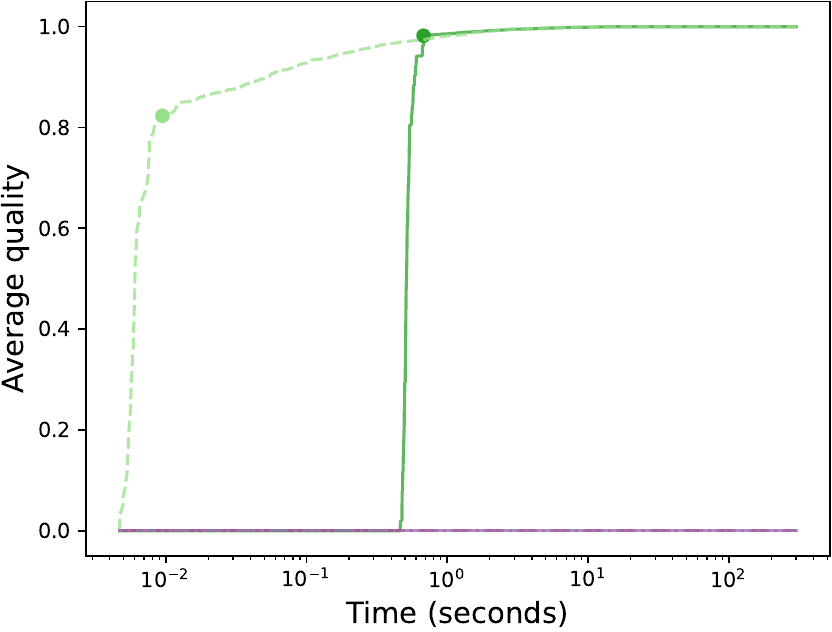}%
\centering%
\includegraphics[width=\plotwidth{}]{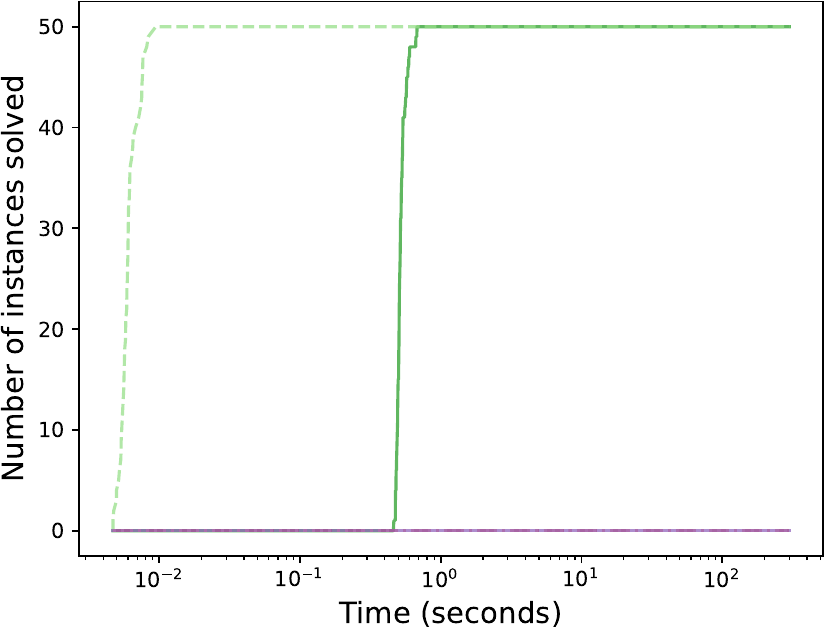}%
\centering%
\includegraphics[width=\plotwidth{}]{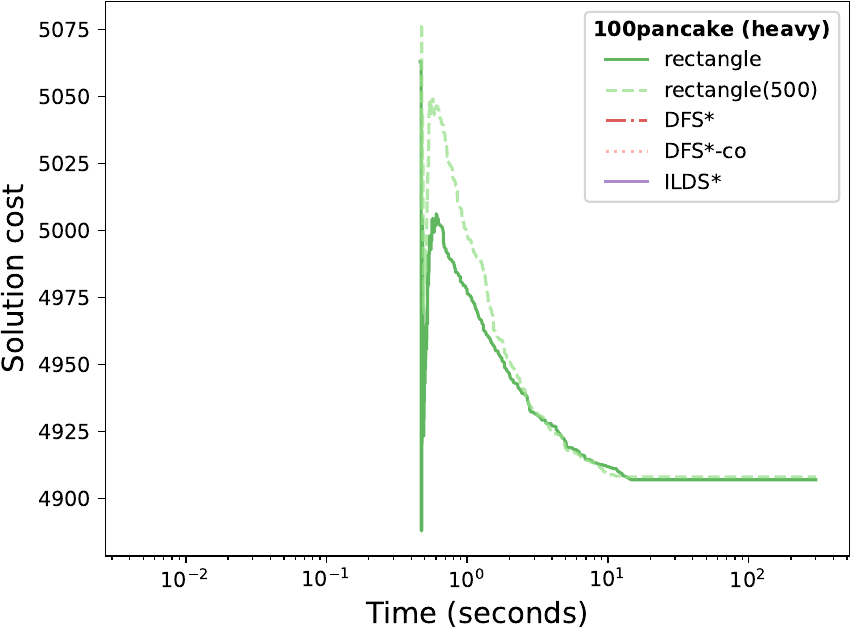}%
\caption{100pancake (heavy cost)}%
\end{figure*}

\begin{figure*}[h!]%
\centering%
\includegraphics[width=\plotwidth{}]{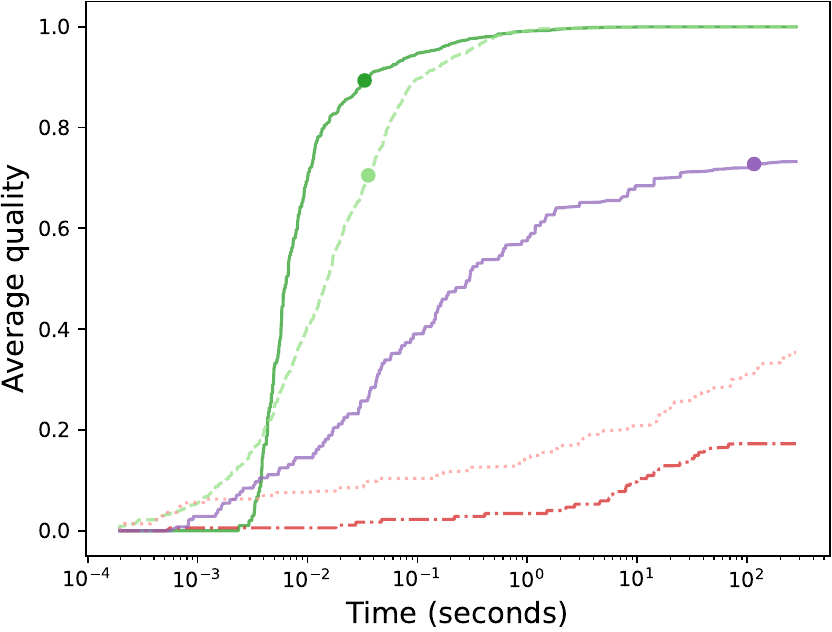}%
\centering%
\includegraphics[width=\plotwidth{}]{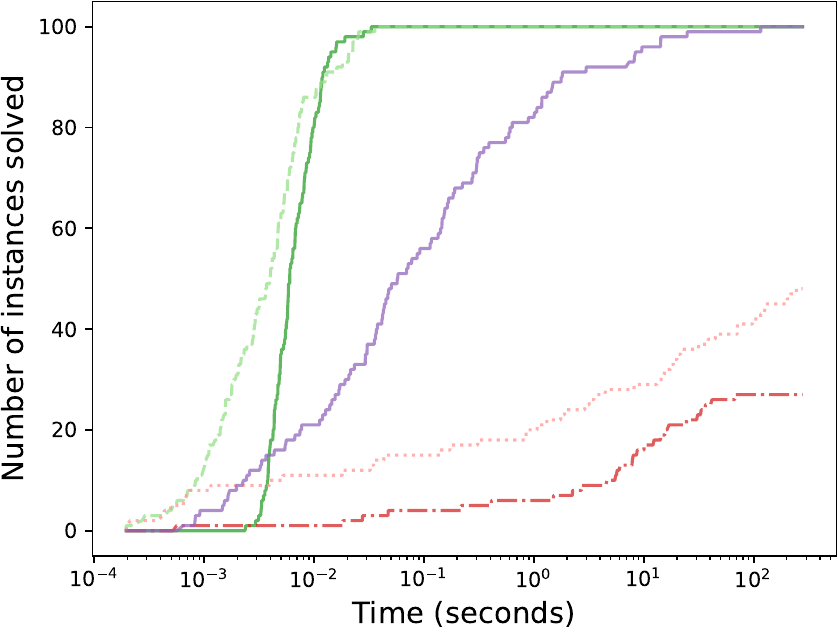}%
\centering%
\includegraphics[width=\plotwidth{}]{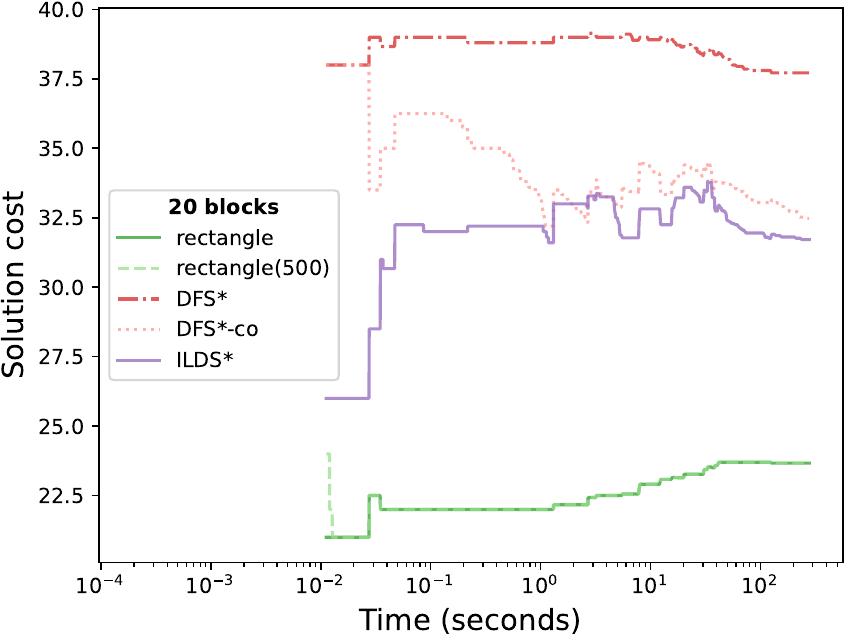}%
\caption{20 blocks}%
\end{figure*}

\begin{figure*}[h!]%
\centering%
\includegraphics[width=\plotwidth{}]{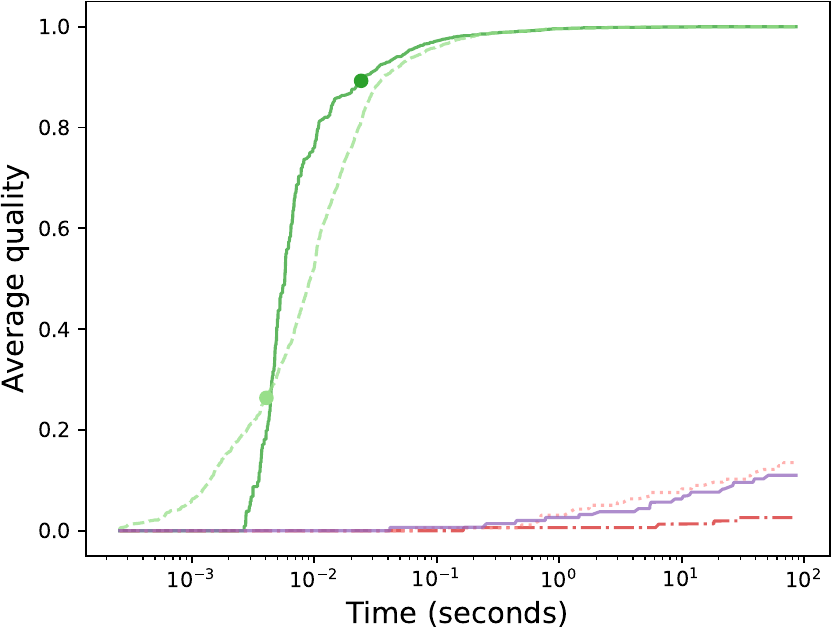}%
\centering%
\includegraphics[width=\plotwidth{}]{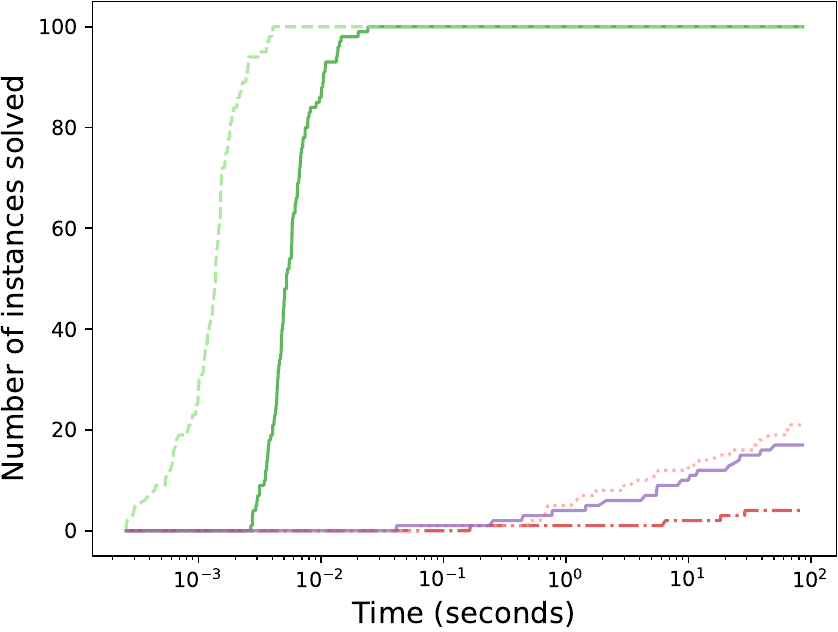}%
\centering%
\includegraphics[width=\plotwidth{}]{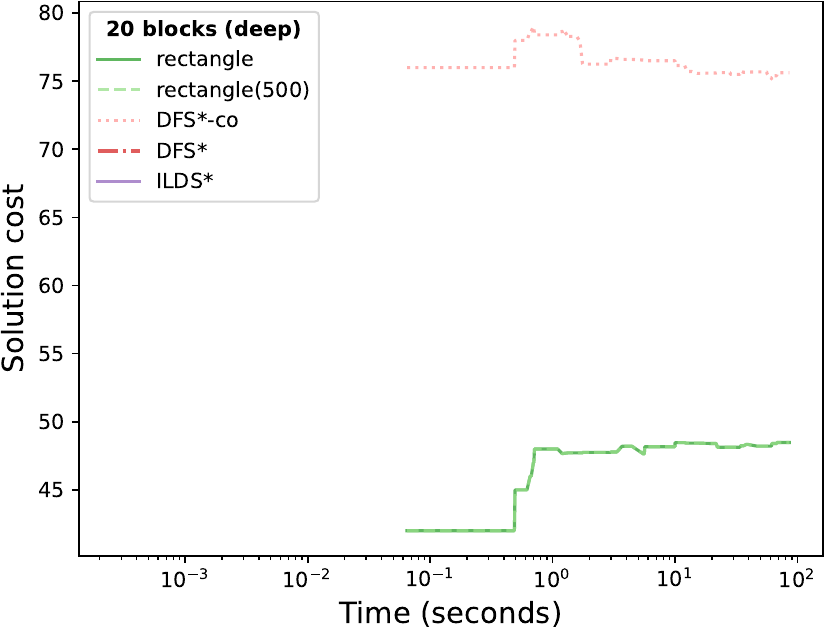}%
\caption{20 blocks (deep)}%
\end{figure*}

\begin{figure*}[h!]%
\centering%
\includegraphics[width=\plotwidth{}]{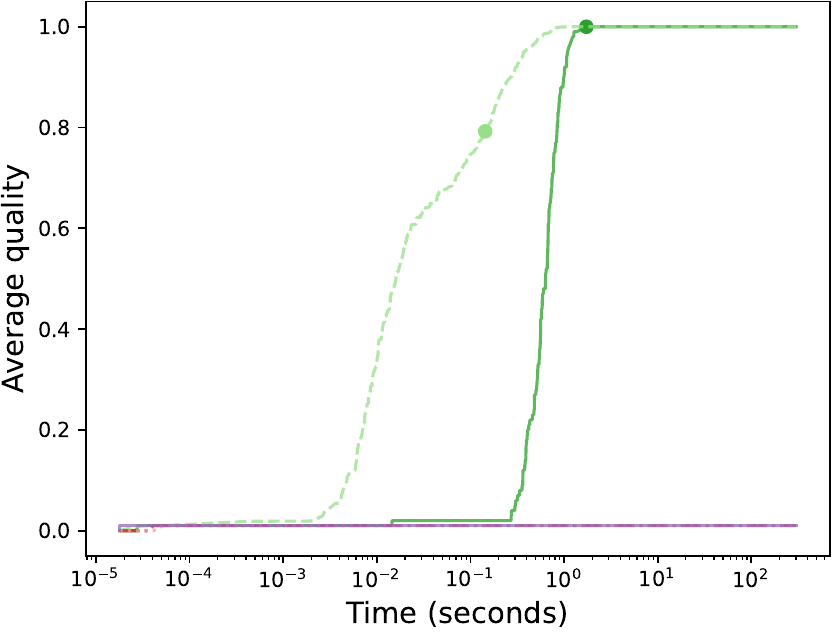}%
\centering%
\includegraphics[width=\plotwidth{}]{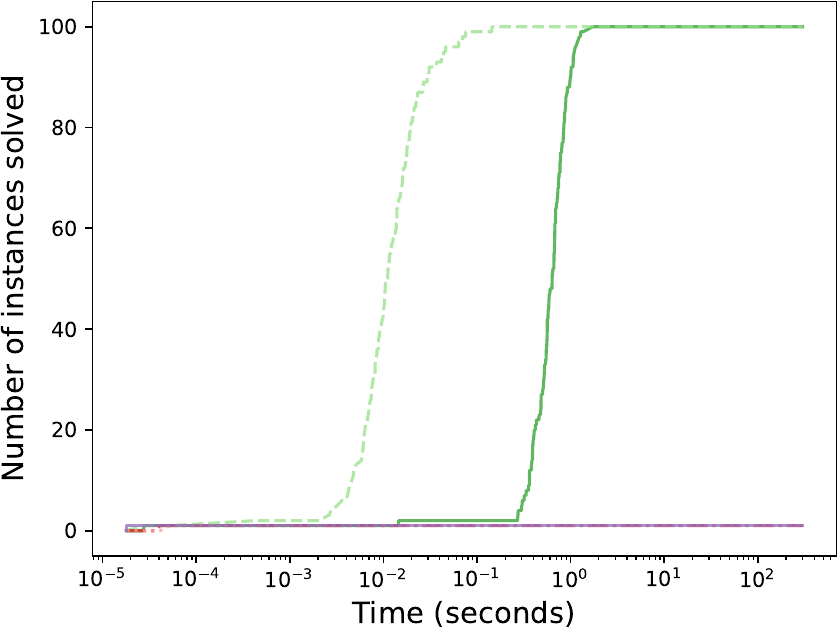}%
\centering%
\includegraphics[width=\plotwidth{}]{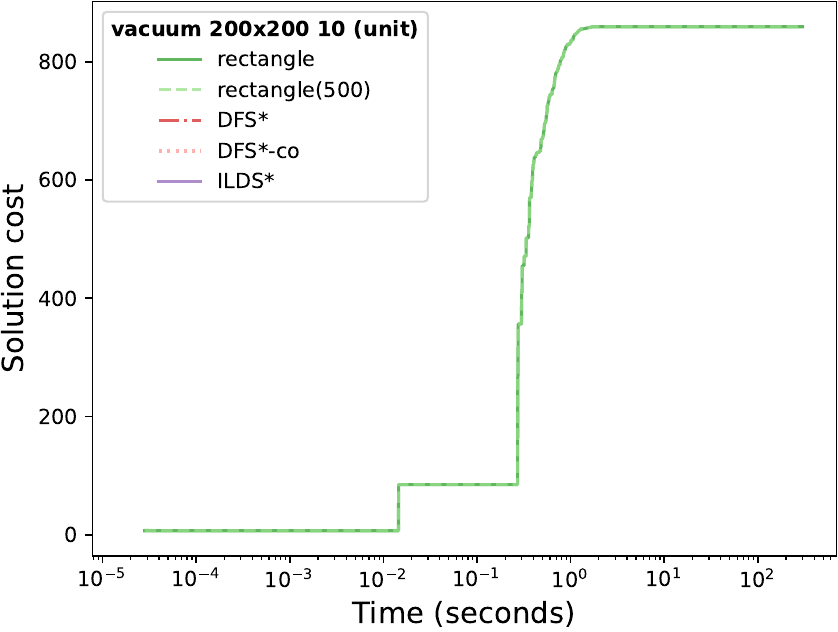}%
\caption{vacuum 200x200, 10 dirts (unit cost)}%
\end{figure*}

\begin{figure*}[h!]%
\centering%
\includegraphics[width=\plotwidth{}]{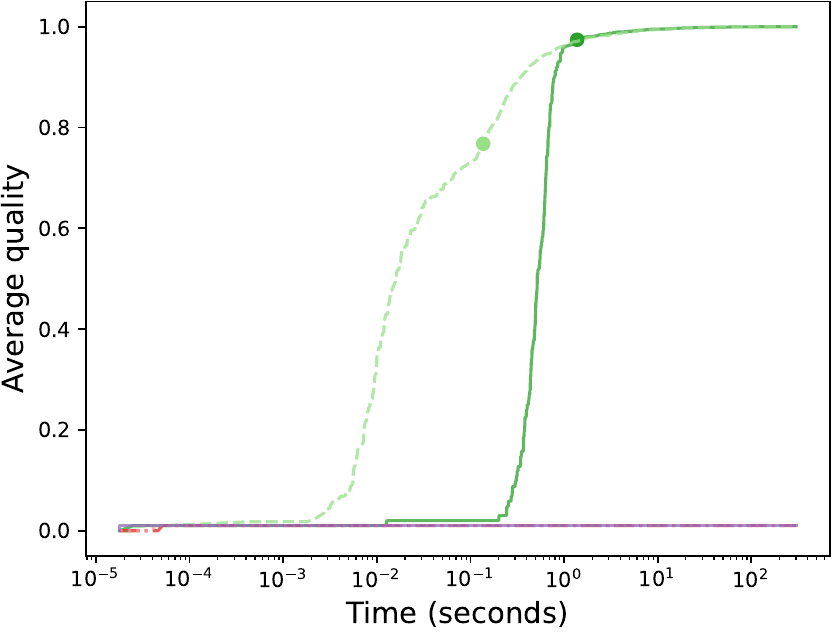}%
\centering%
\includegraphics[width=\plotwidth{}]{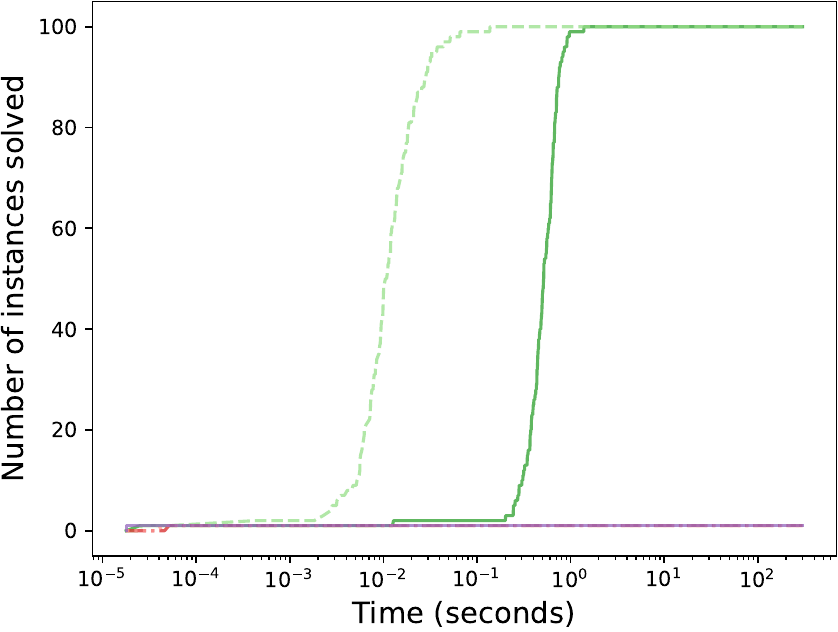}%
\centering%
\includegraphics[width=\plotwidth{}]{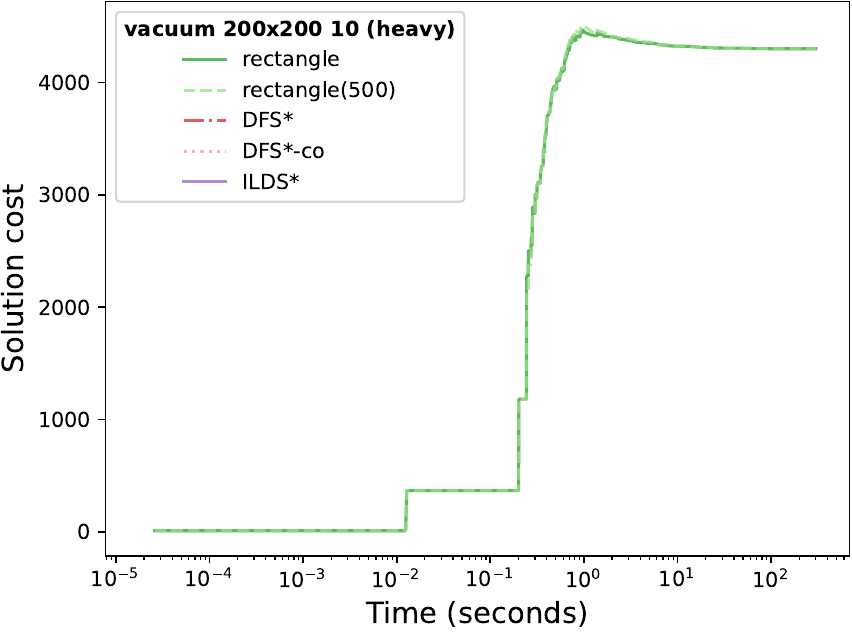}%
\caption{vacuum 200x200, 10 dirts (heavy cost)}%
\end{figure*}

\begin{figure*}[h!]%
\centering%
\includegraphics[width=\plotwidth{}]{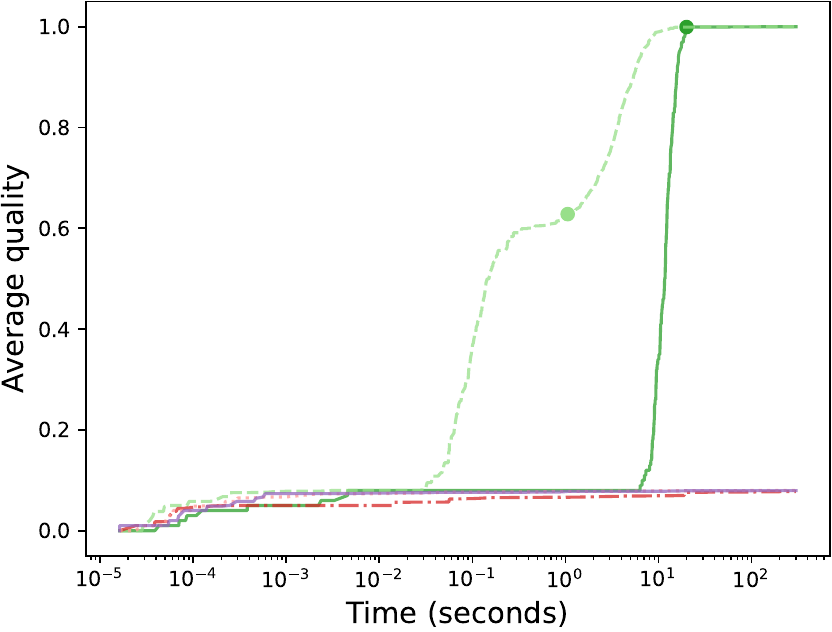}%
\centering%
\includegraphics[width=\plotwidth{}]{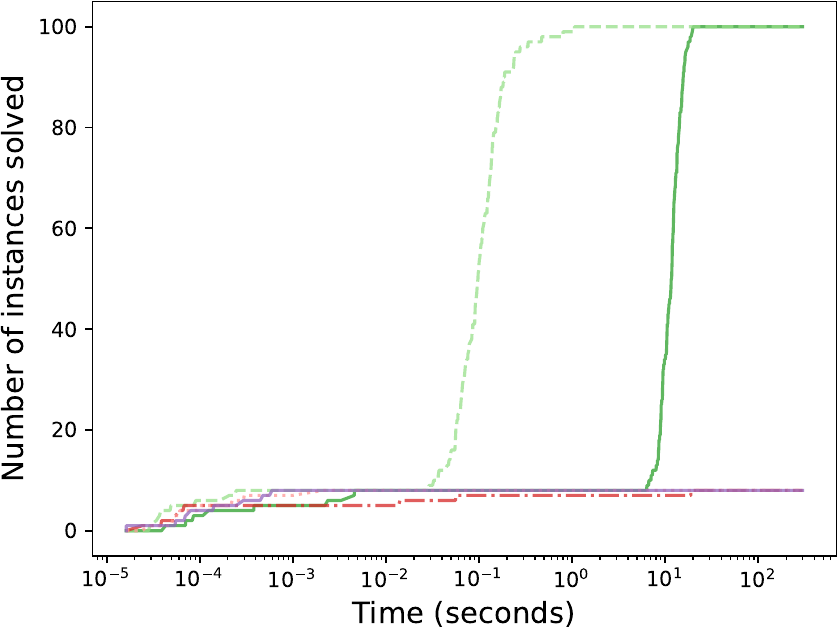}%
\centering%
\includegraphics[width=\plotwidth{}]{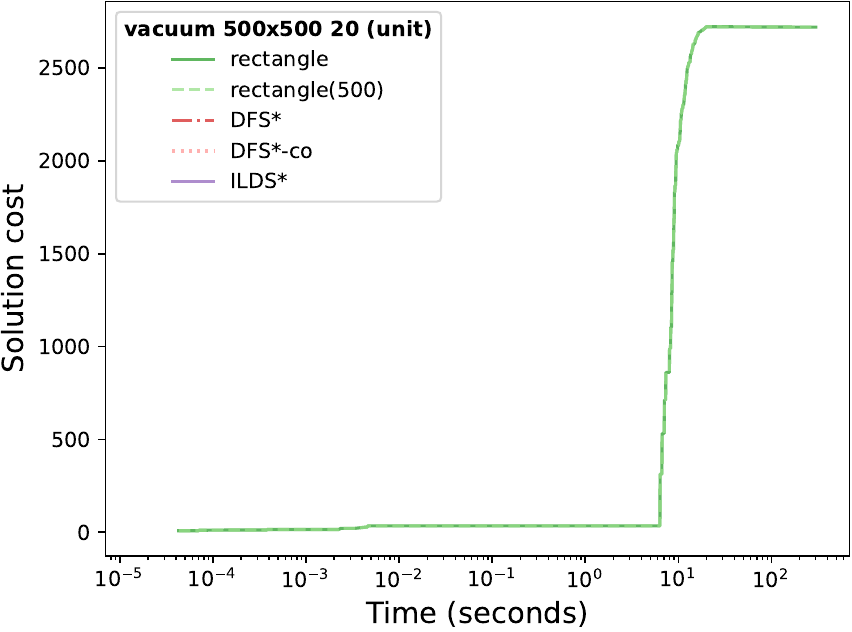}%
\caption{vacuum 500x500, 20 dirts (unit cost)}%
\end{figure*}

\begin{figure*}[h!]%
\centering%
\includegraphics[width=\plotwidth{}]{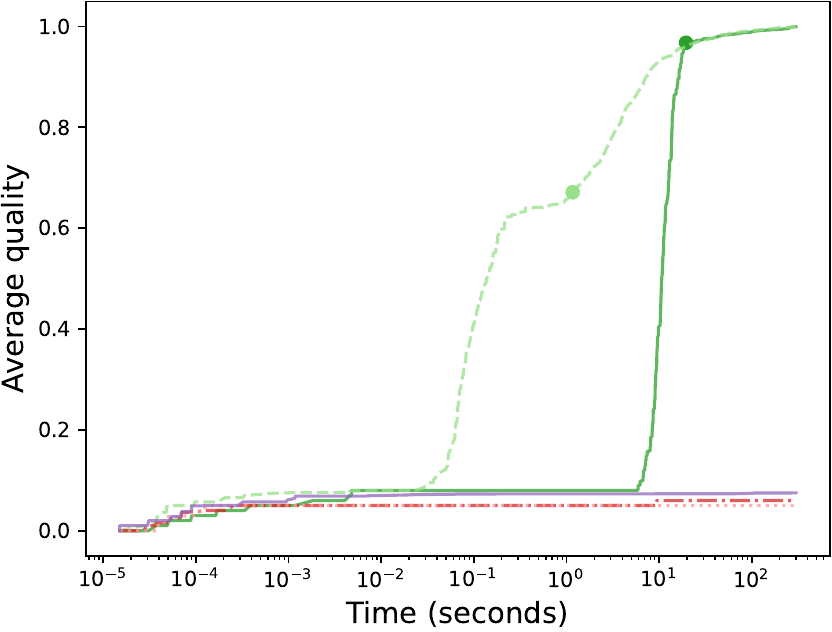}%
\centering%
\includegraphics[width=\plotwidth{}]{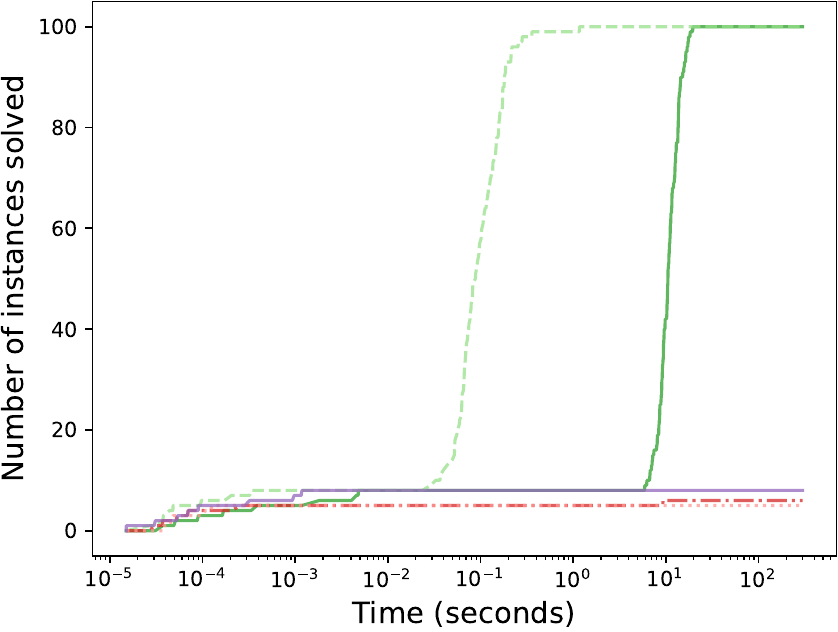}%
\centering%
\includegraphics[width=\plotwidth{}]{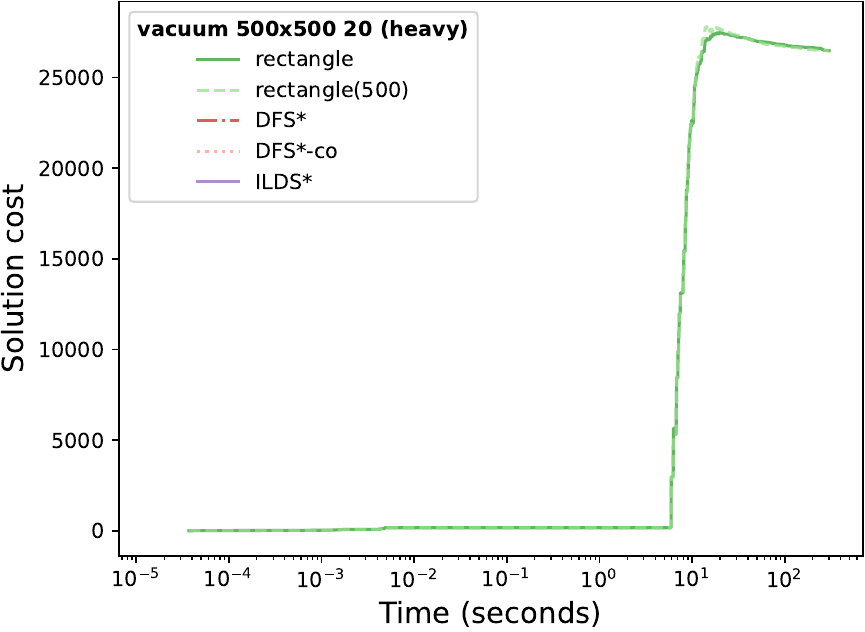}%
\caption{vacuum 500x500, 20 dirts (heavy cost)}%
\end{figure*}

\begin{figure*}[h!]%
\centering%
\includegraphics[width=\plotwidth{}]{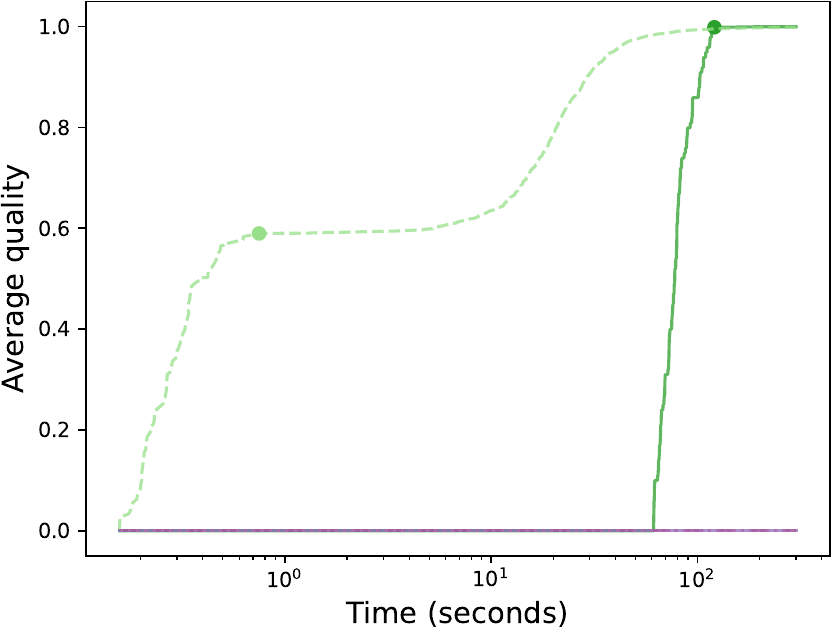}%
\centering%
\includegraphics[width=\plotwidth{}]{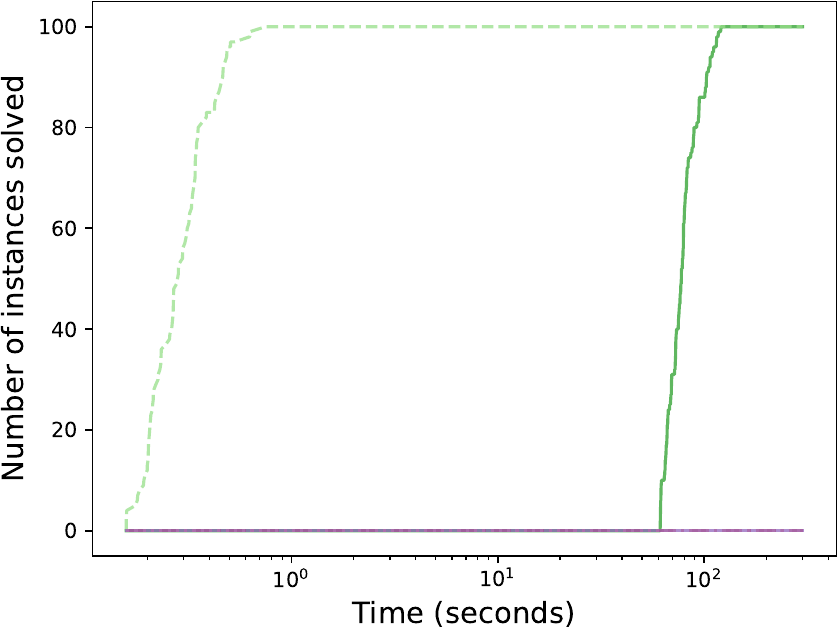}%
\centering%
\includegraphics[width=\plotwidth{}]{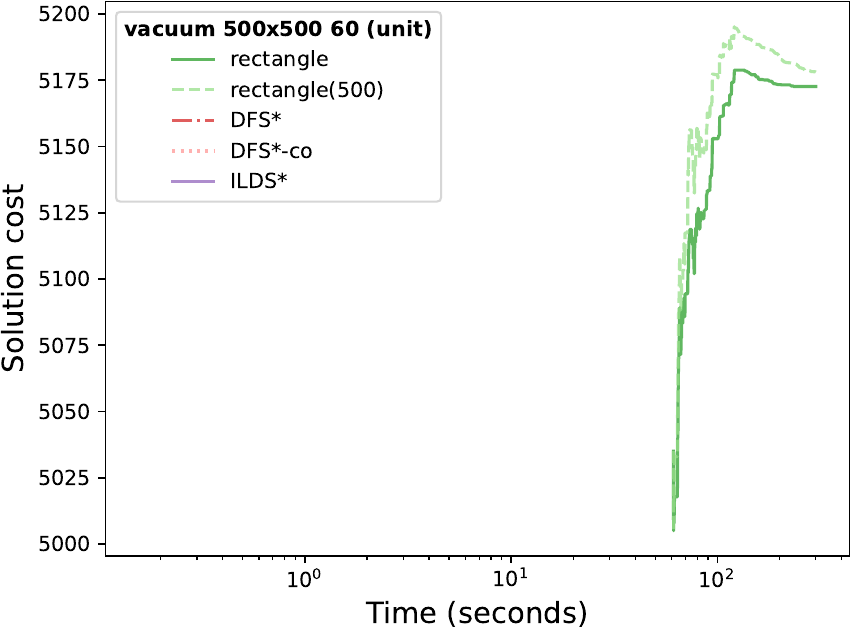}%
\caption{vacuum 500x500, 60 dirts (unit cost)}%
\end{figure*}

\begin{figure*}[h!]%
\centering%
\includegraphics[width=\plotwidth{}]{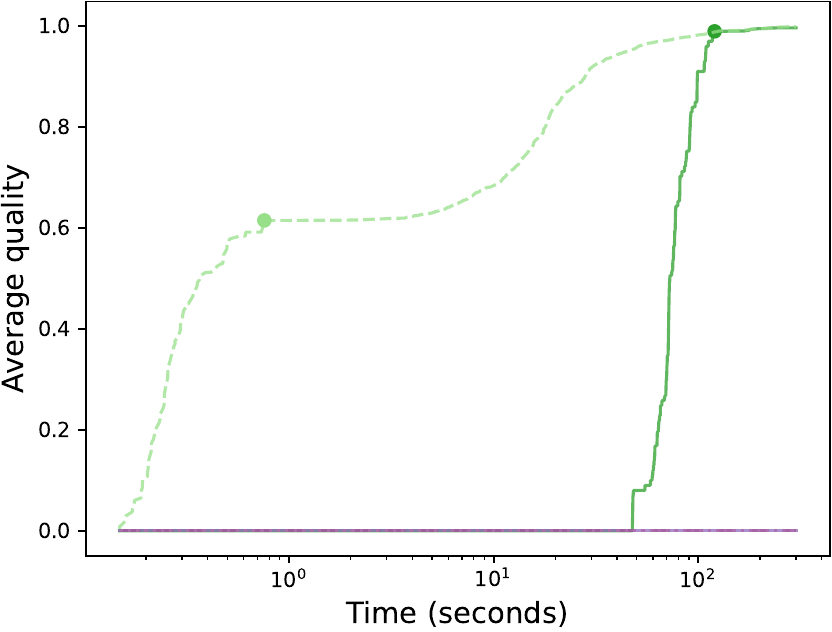}%
\centering%
\includegraphics[width=\plotwidth{}]{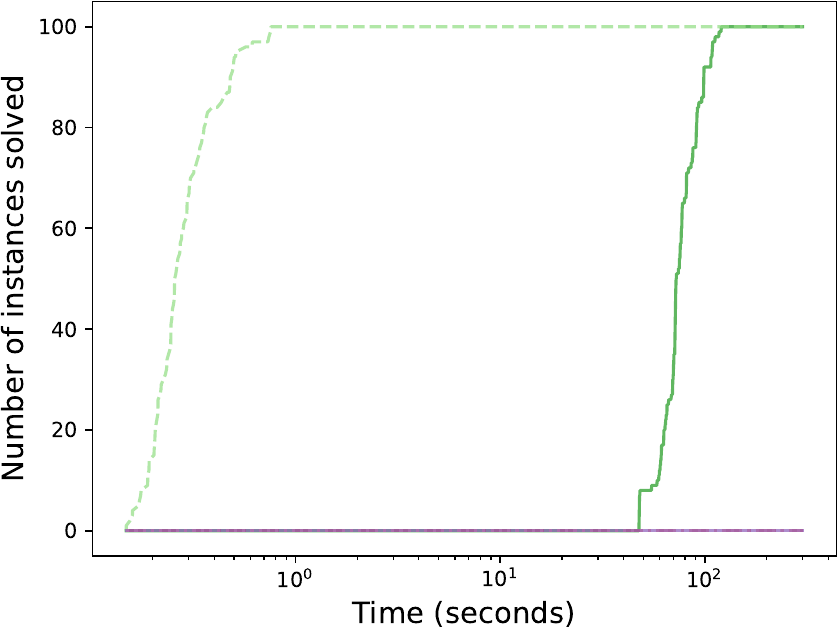}%
\centering%
\includegraphics[width=\plotwidth{}]{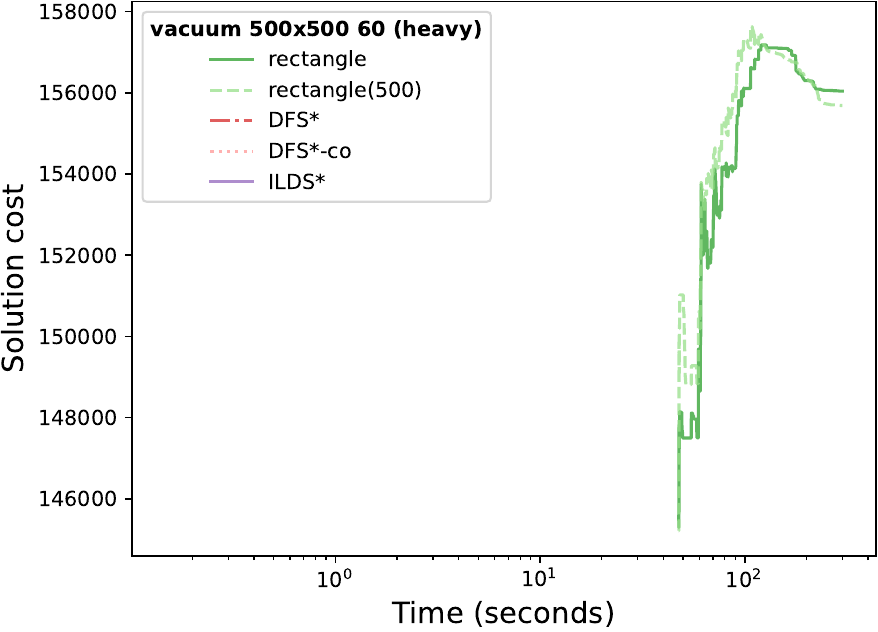}%
\caption{vacuum 500x500, 60 dirts (heavy cost)}%
\end{figure*}

\begin{figure*}[h!]%
\centering%
\includegraphics[width=\plotwidth{}]{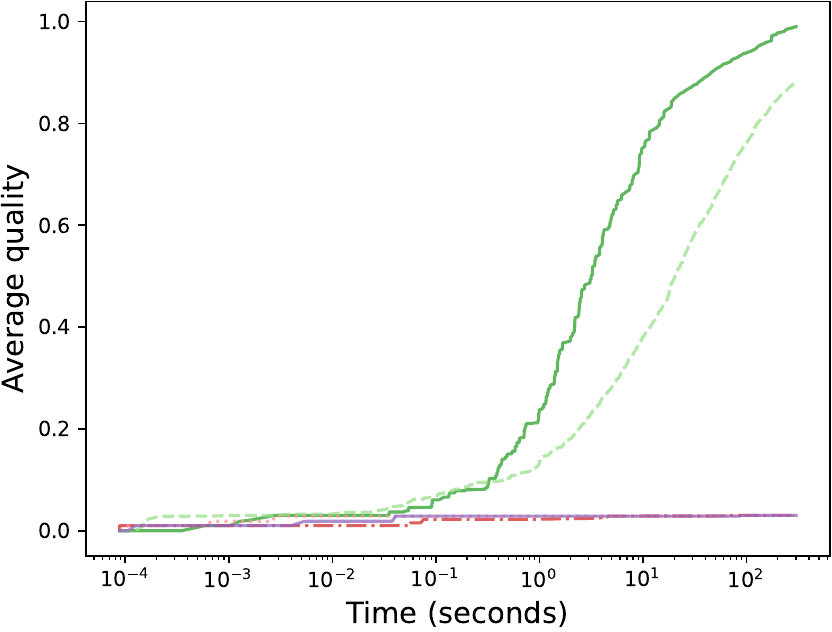}%
\centering%
\includegraphics[width=\plotwidth{}]{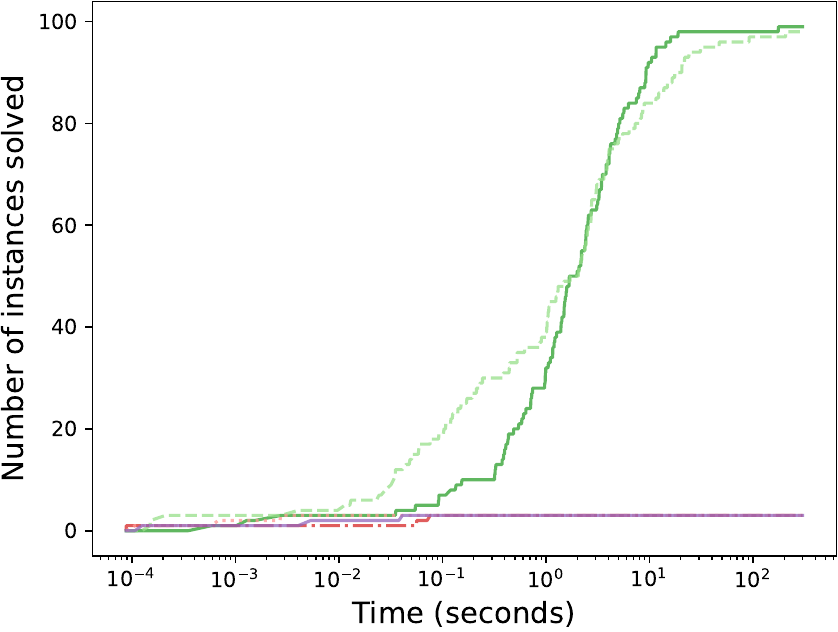}%
\centering%
\includegraphics[width=\plotwidth{}]{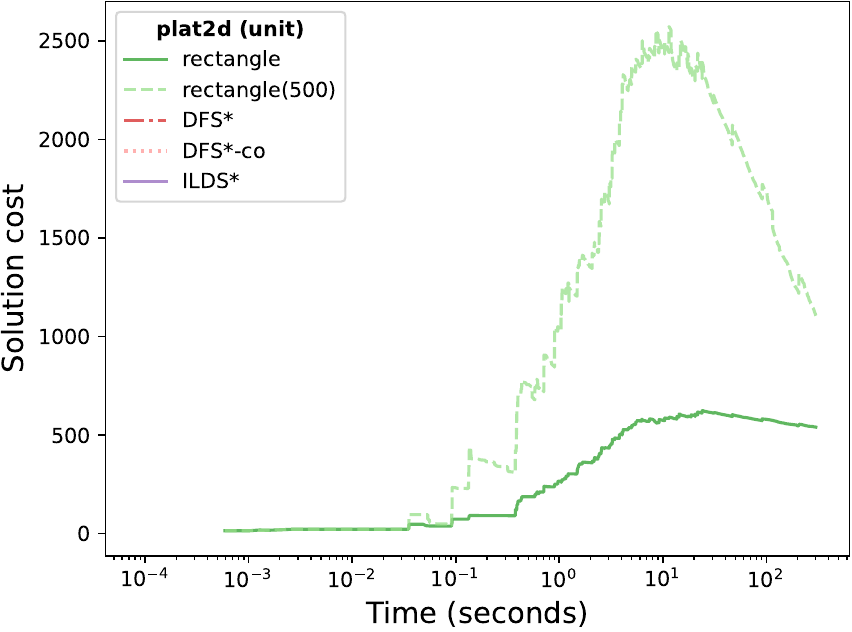}%
\caption{platform game (unit cost)}%
\end{figure*}

\begin{figure*}[h!]%
\centering%
\includegraphics[width=\plotwidth{}]{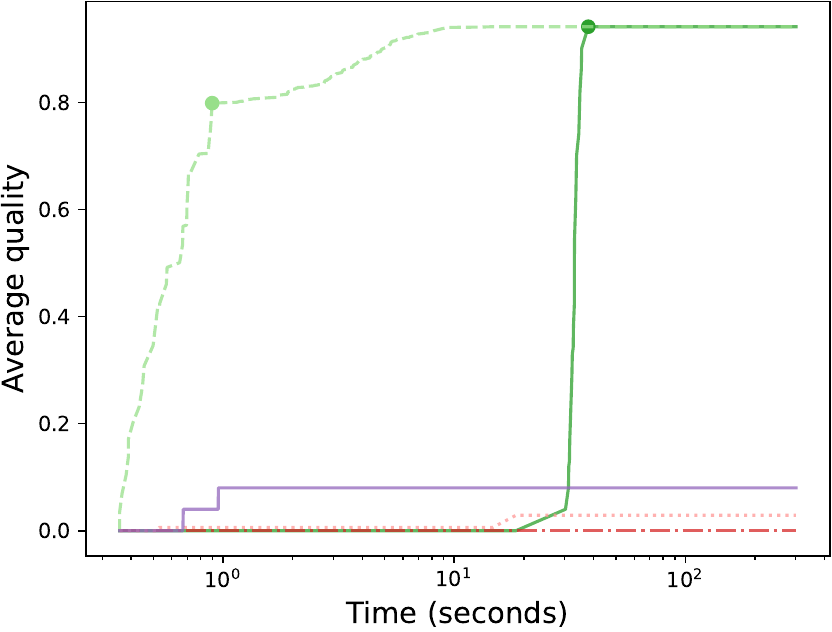}%
\centering%
\includegraphics[width=\plotwidth{}]{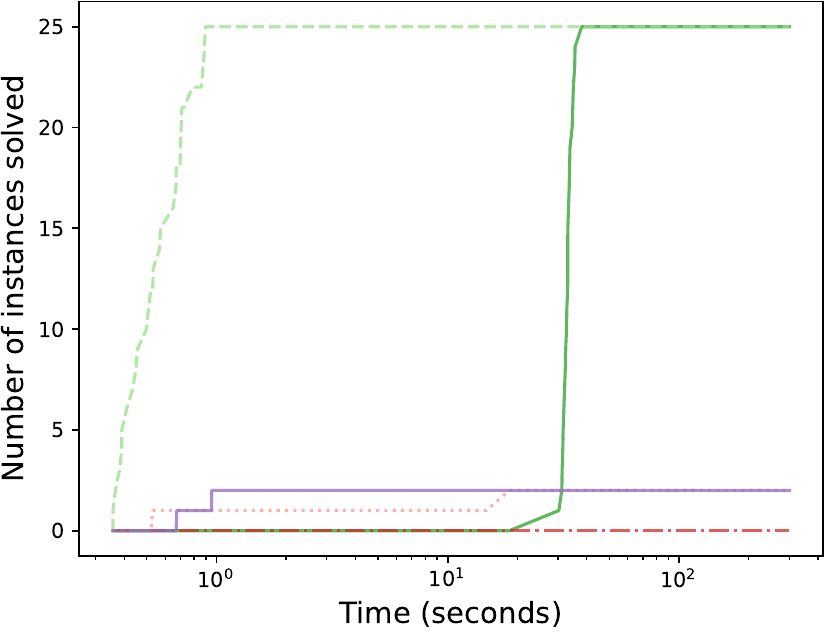}%
\centering%
\includegraphics[width=\plotwidth{}]{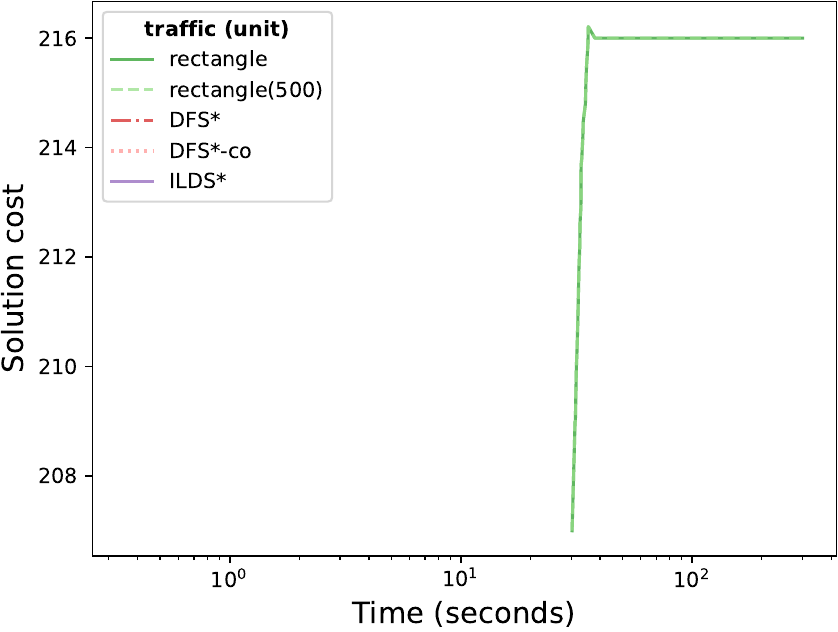}%
\caption{traffic (unit cost)}%
\end{figure*}

\begin{figure*}[h!]%
\centering%
\includegraphics[width=\plotwidth{}]{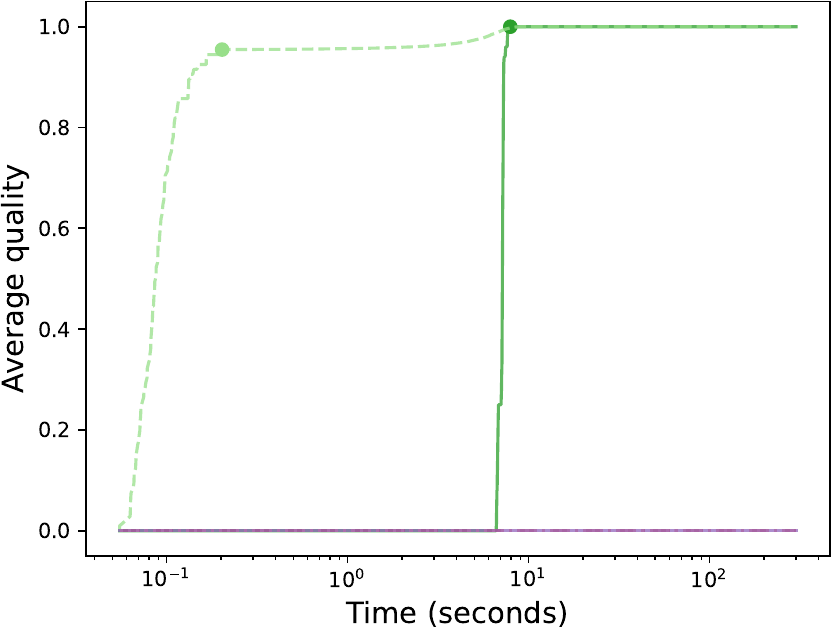}%
\centering%
\includegraphics[width=\plotwidth{}]{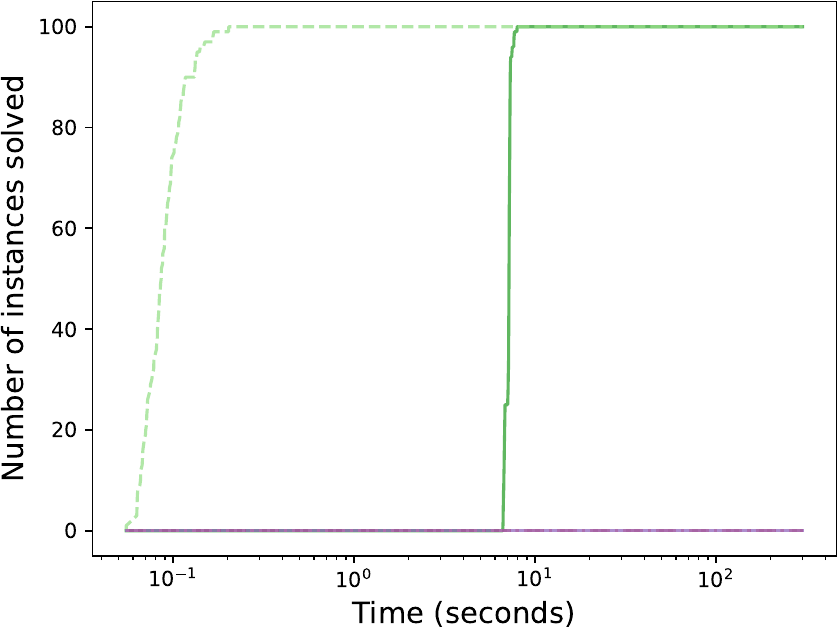}%
\centering%
\includegraphics[width=\plotwidth{}]{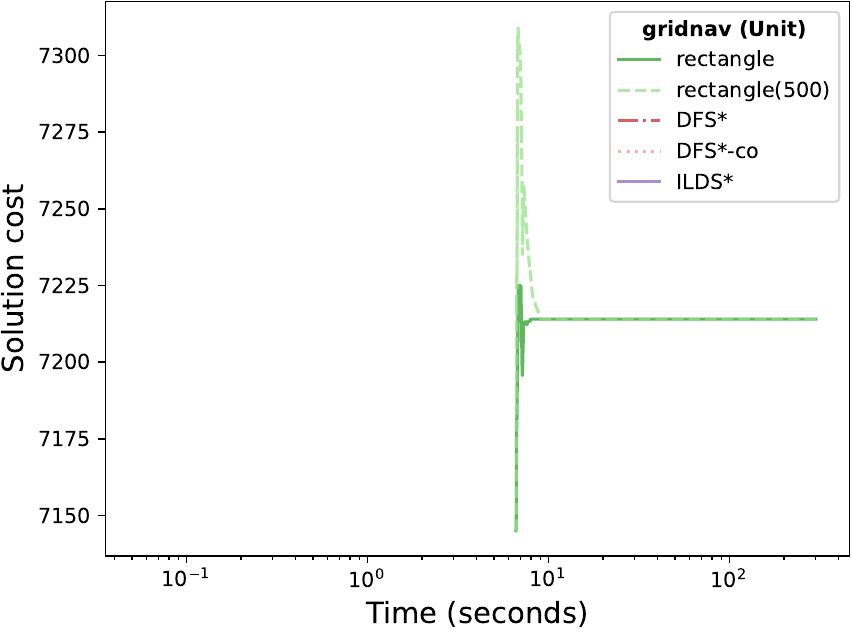}%
\caption{random grids (unit cost)}%
\end{figure*}

\begin{figure*}[h!]%
\centering%
\includegraphics[width=\plotwidth{}]{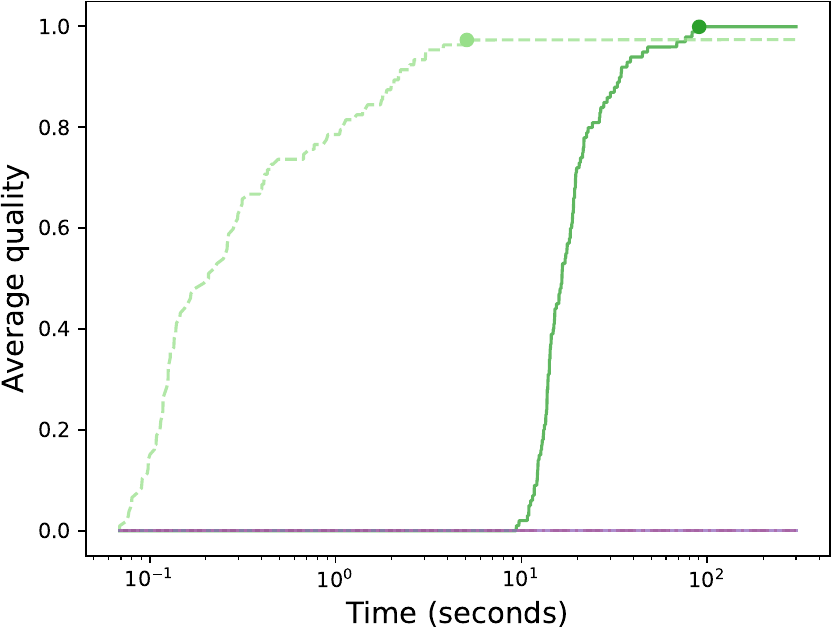}%
\centering%
\includegraphics[width=\plotwidth{}]{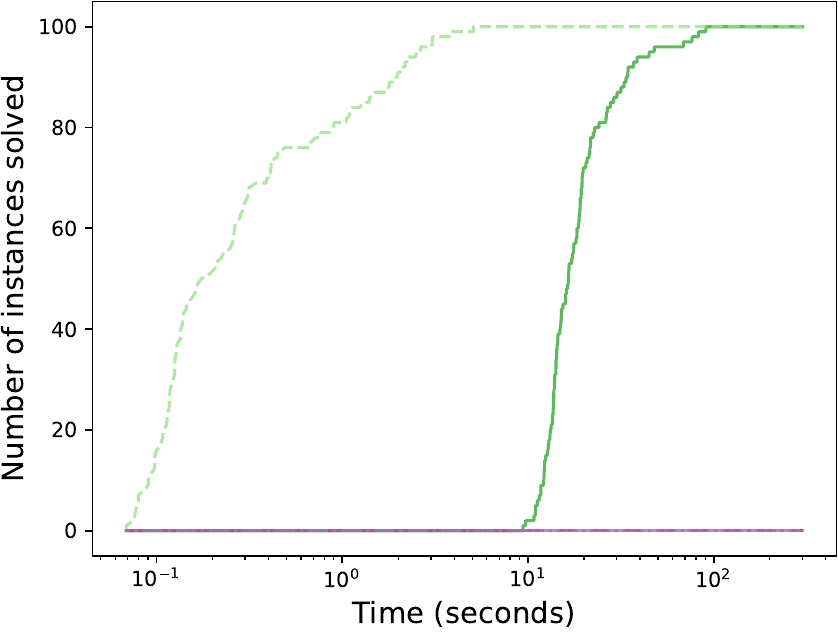}%
\centering%
\includegraphics[width=\plotwidth{}]{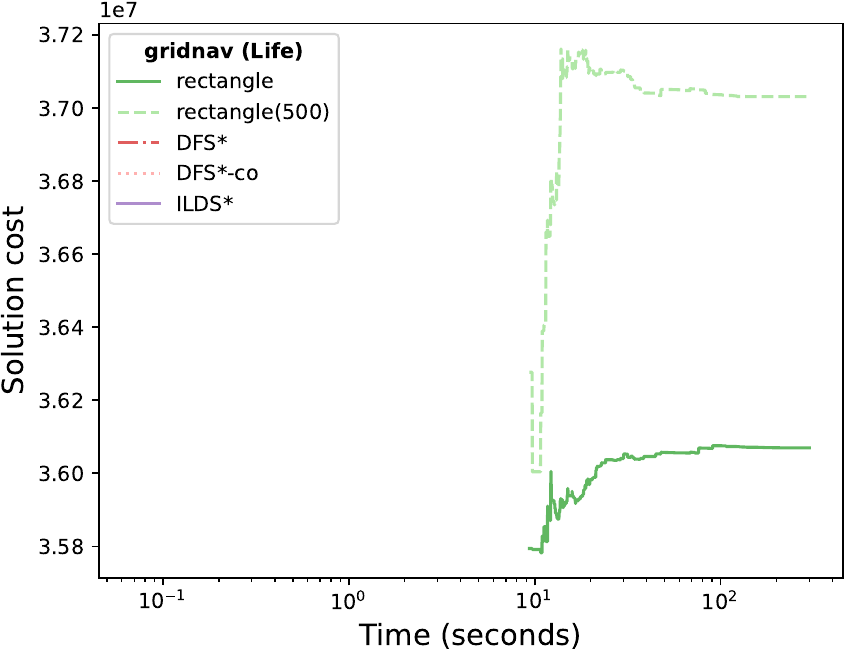}%
\caption{random grids (life cost)}%
\end{figure*}

\clearpage

\section{Fixed-width Beam Search Results} \label{sec:beam}

\begin{figure*}[h!]%
\centering%
\begin{subfigure}{\plotwidth{}}%
\includegraphics[width=\plotwidth{}]{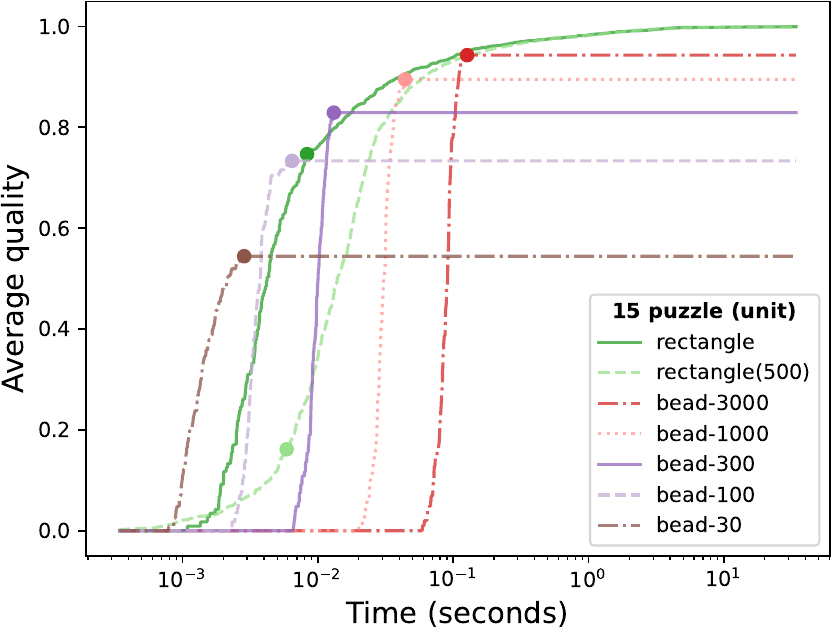}%
\caption{15 puzzle (unit)}%
\end{subfigure}
\begin{subfigure}{\plotwidth{}}%
\includegraphics[width=\plotwidth{}]{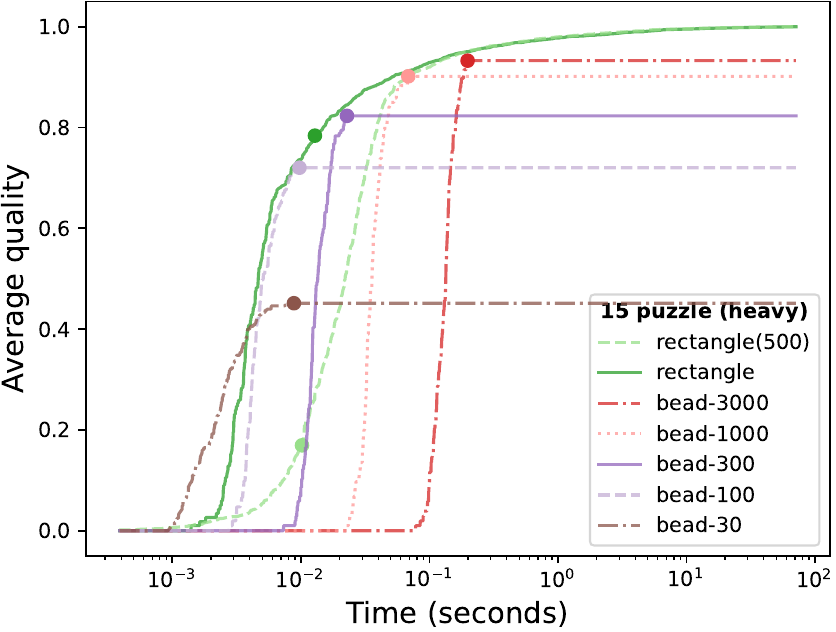}%
\caption{15 puzzle (heavy)}%
\end{subfigure}
\begin{subfigure}{\plotwidth{}}%
\includegraphics[width=\plotwidth{}]{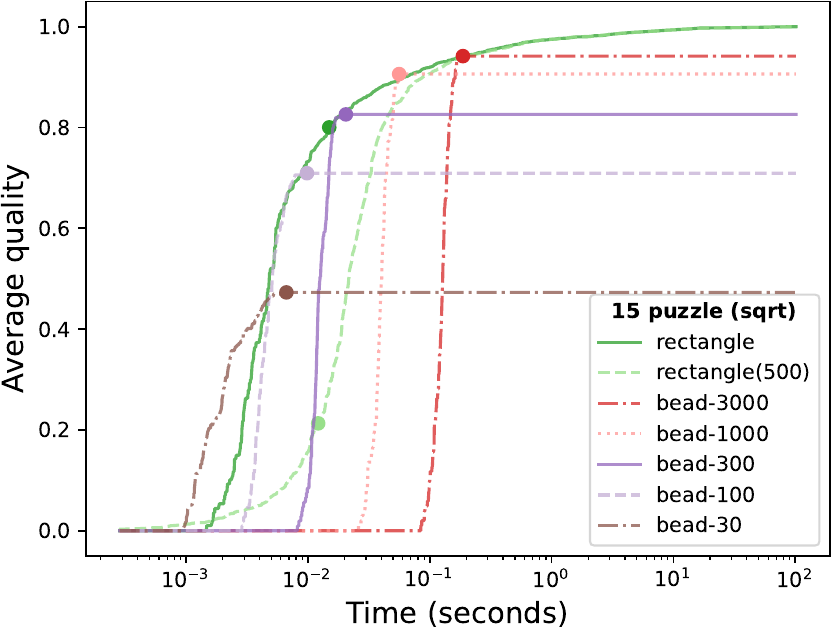}%
\caption{15 puzzle (sqrt)}%
\end{subfigure}
\caption{}
\end{figure*}

\begin{figure*}[h!]%
\centering%
\begin{subfigure}{\plotwidth{}}%
\includegraphics[width=\plotwidth{}]{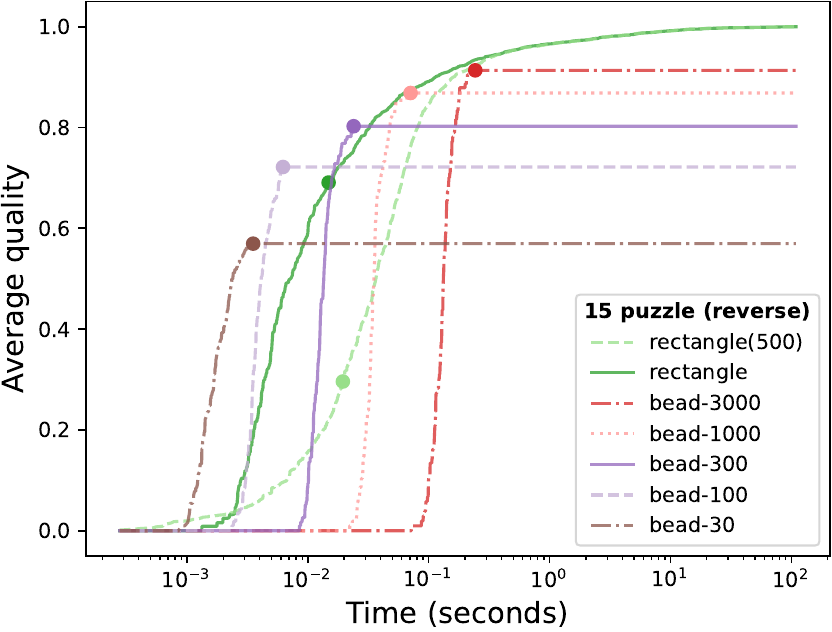}%
\caption{15 puzzle (reverse)}%
\end{subfigure}
\begin{subfigure}{\plotwidth{}}%
\includegraphics[width=\plotwidth{}]{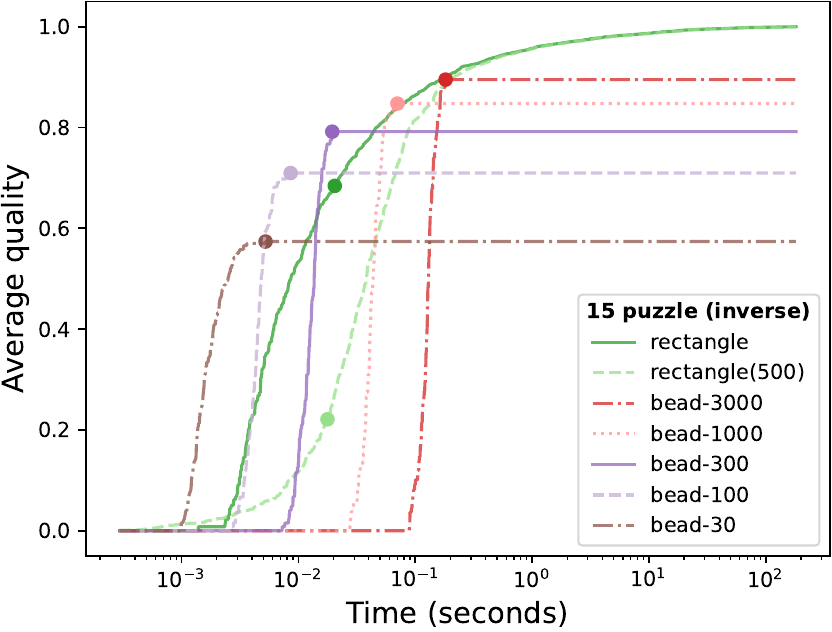}%
\caption{15 puzzle (inverse)}%
\end{subfigure}
\begin{subfigure}{\plotwidth{}}%
\includegraphics[width=\plotwidth{}]{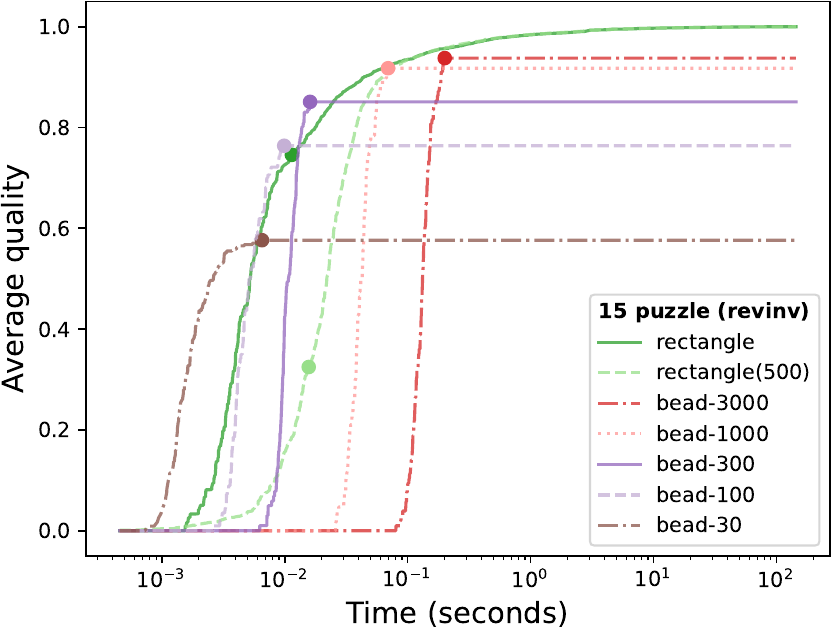}%
\caption{15 puzzle (reverse inverse)}%
\end{subfigure}
\caption{}
\end{figure*}

\begin{figure*}[h!]%
\centering%
\begin{subfigure}{\plotwidth{}}%
\includegraphics[width=\plotwidth{}]{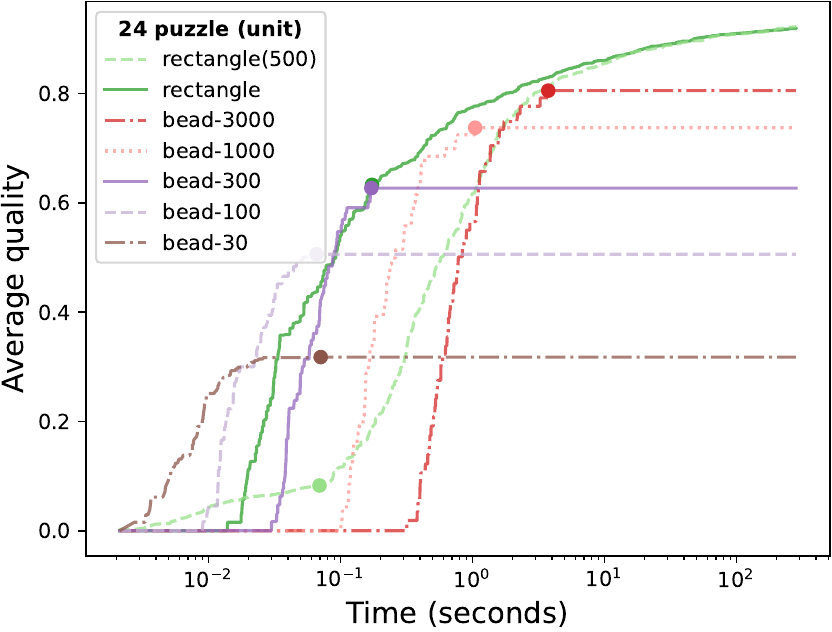}%
\caption{24 puzzle (unit)}%
\end{subfigure}
\begin{subfigure}{\plotwidth{}}%
\includegraphics[width=\plotwidth{}]{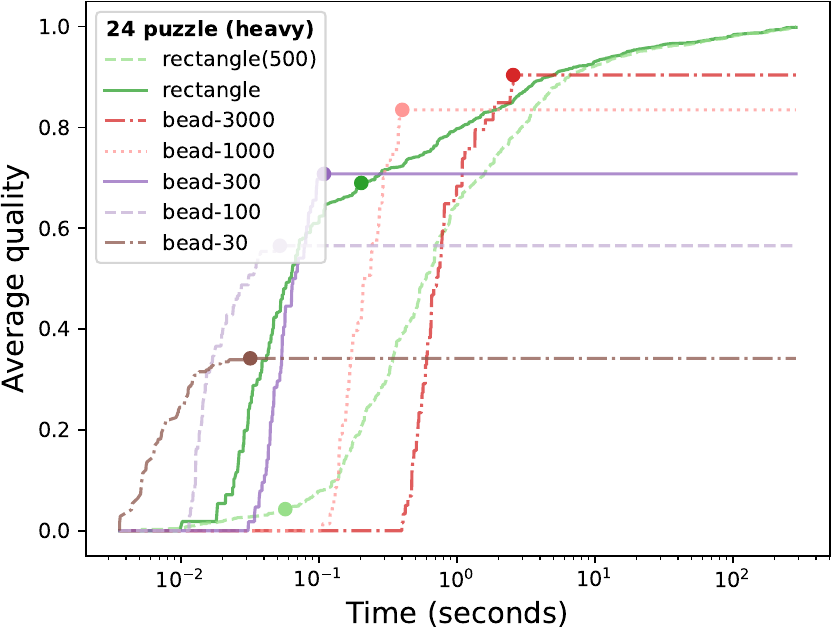}%
\caption{24 puzzle (heavy)}%
\end{subfigure}
\begin{subfigure}{\plotwidth{}}%
\includegraphics[width=\plotwidth{}]{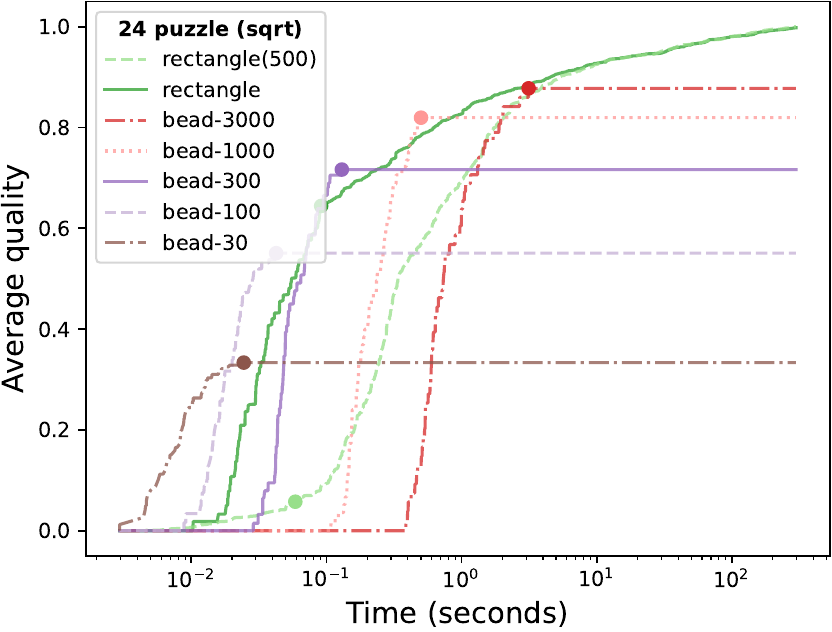}%
\caption{24 puzzle (sqrt)}%
\end{subfigure}
\caption{}
\end{figure*}

\begin{figure*}[h!]%
\centering%
\begin{subfigure}{\plotwidth{}}%
\includegraphics[width=\plotwidth{}]{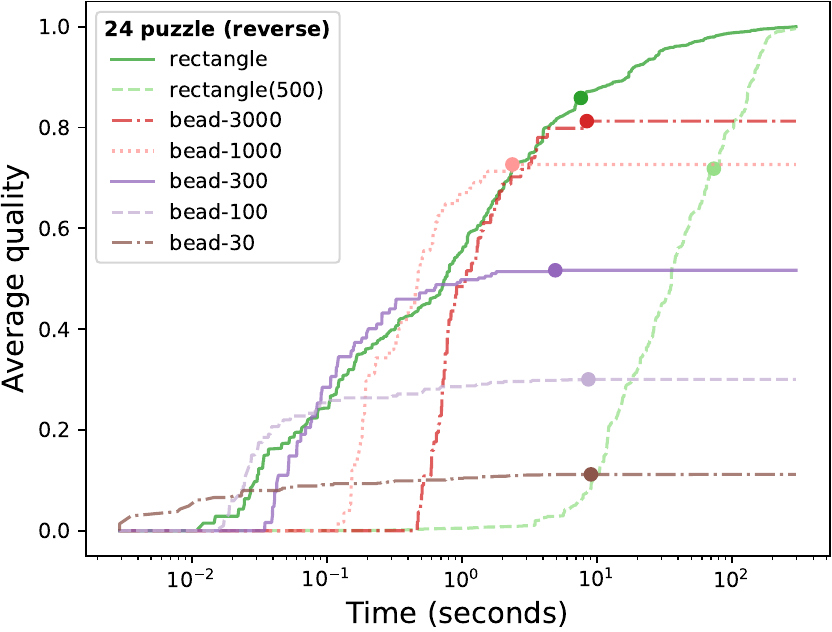}%
\caption{24 puzzle (reverse)}%
\end{subfigure}
\begin{subfigure}{\plotwidth{}}%
\includegraphics[width=\plotwidth{}]{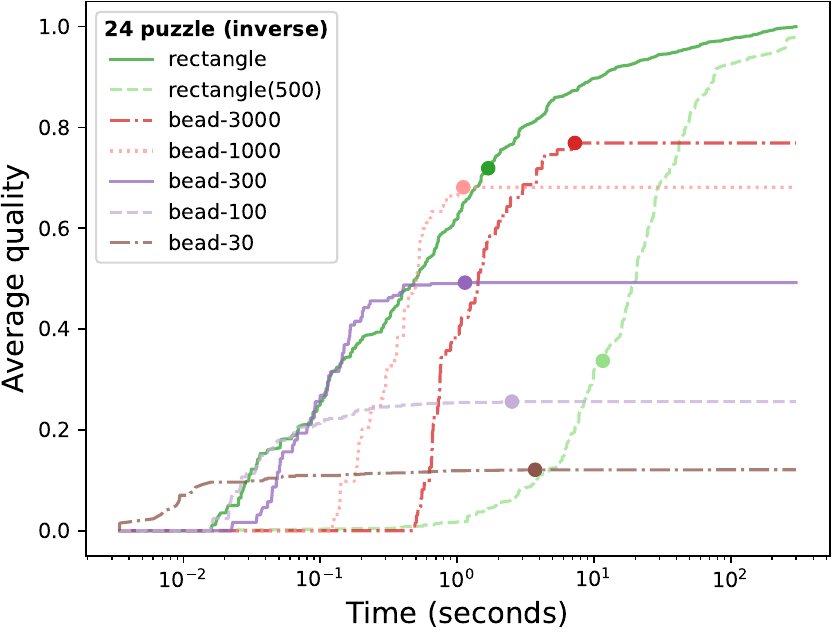}%
\caption{24 puzzle (inverse)}%
\end{subfigure}
\begin{subfigure}{\plotwidth{}}%
\includegraphics[width=\plotwidth{}]{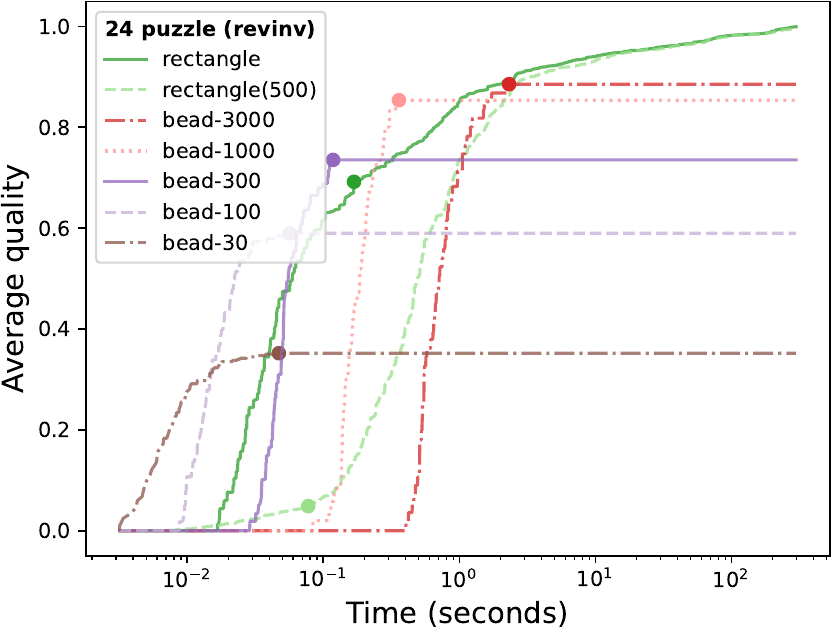}%
\caption{24 puzzle (reverse inverse)}%
\end{subfigure}
\caption{}
\end{figure*}

\begin{figure*}[h!]%
\centering%
\begin{subfigure}{\plotwidth{}}%
\includegraphics[width=\plotwidth{}]{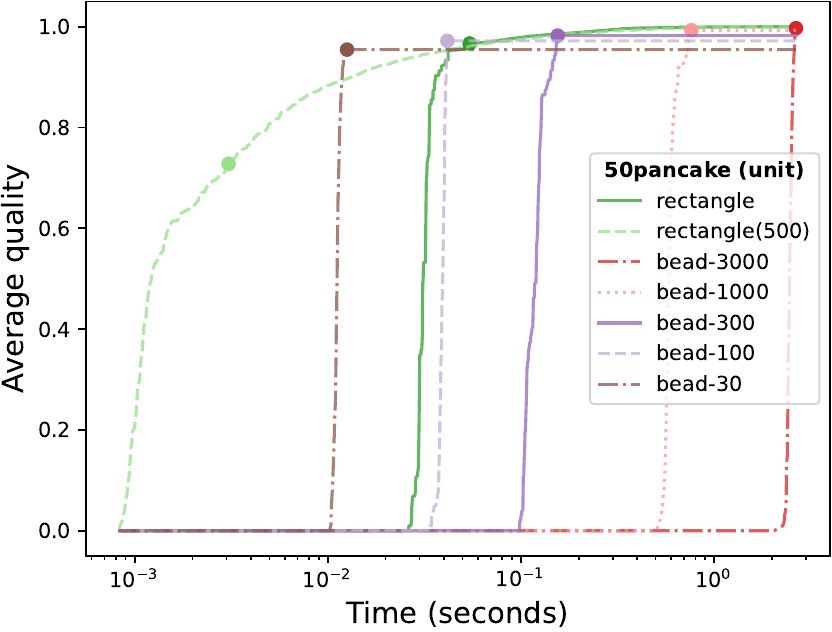}%
\caption{50pancake (unit)}%
\end{subfigure}
\begin{subfigure}{\plotwidth{}}%
\includegraphics[width=\plotwidth{}]{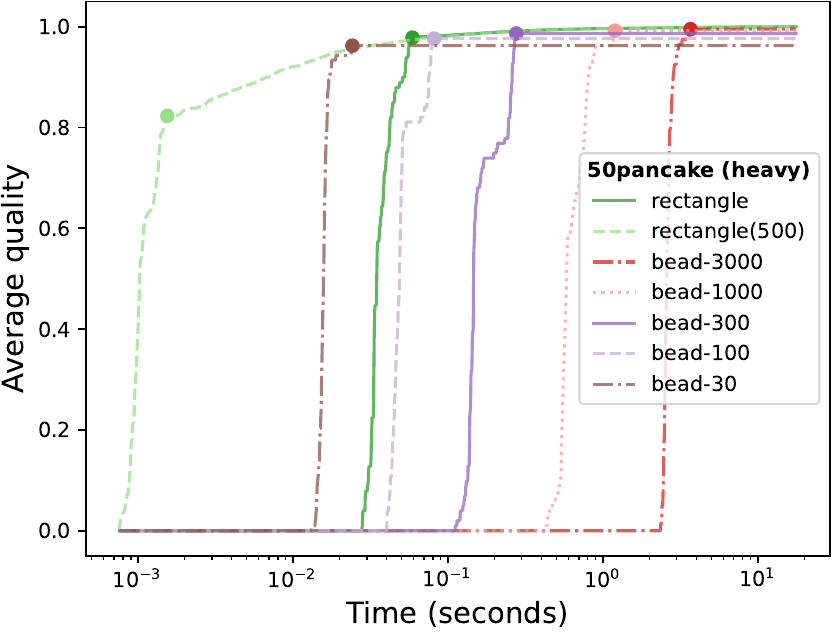}%
\caption{50pancake (heavy)}%
\end{subfigure}
\begin{subfigure}{\plotwidth{}}%
\includegraphics[width=\plotwidth{}]{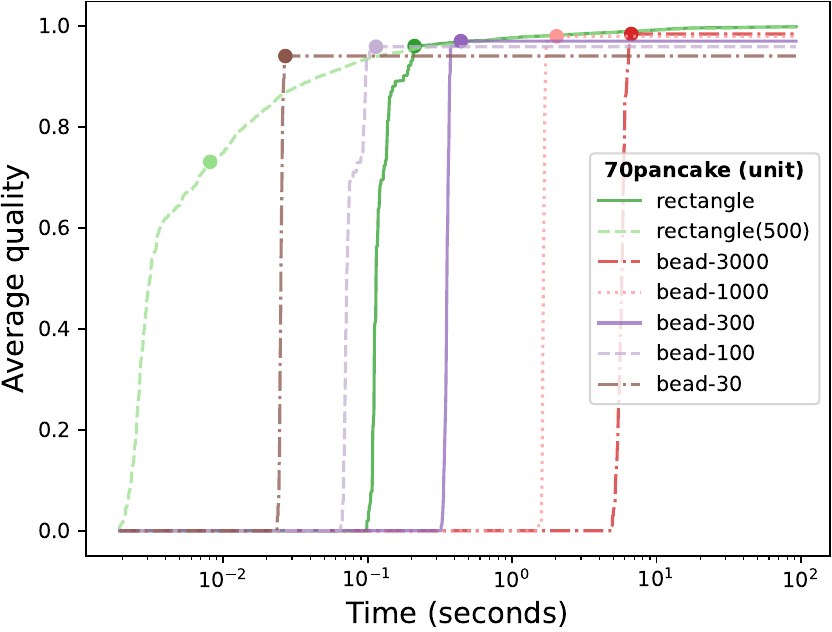}%
\caption{70pancake (unit)}%
\end{subfigure}
\caption{}
\end{figure*}

\begin{figure*}[h!]%
\centering%
\begin{subfigure}{\plotwidth{}}%
\includegraphics[width=\plotwidth{}]{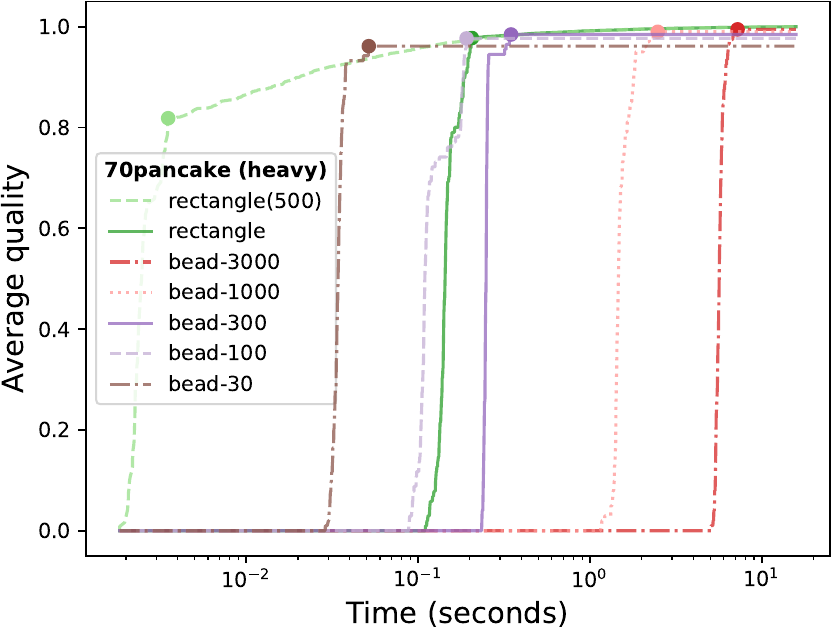}%
\caption{70pancake (heavy)}%
\end{subfigure}
\begin{subfigure}{\plotwidth{}}%
\includegraphics[width=\plotwidth{}]{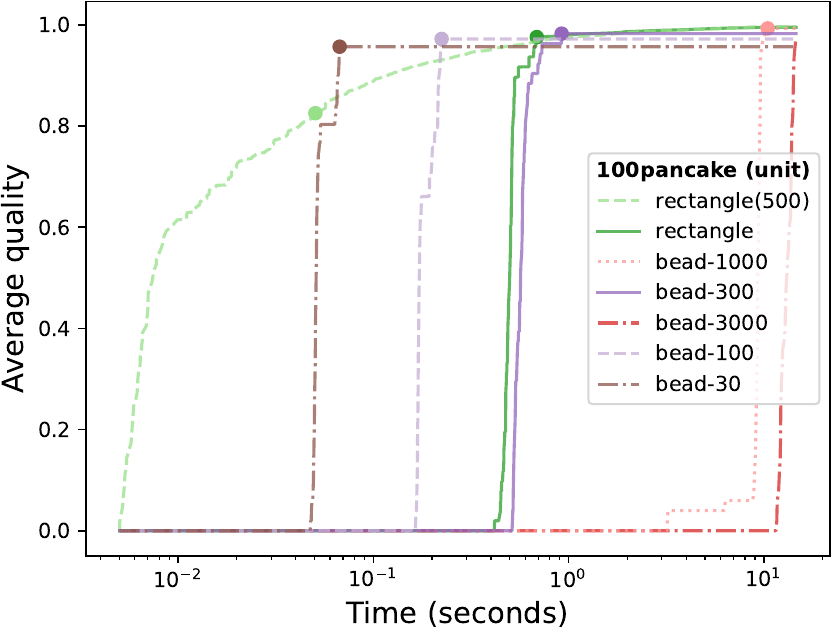}%
\caption{100pancake (unit)}%
\end{subfigure}
\begin{subfigure}{\plotwidth{}}%
\includegraphics[width=\plotwidth{}]{figures/plots/100pancake/rectangle-beam-quality-unit}%
\caption{100pancake (heavy)}%
\end{subfigure}
\caption{}
\end{figure*}

\begin{figure*}[h!]%
\centering%
\begin{subfigure}{\plotwidth{}}%
\includegraphics[width=\plotwidth{}]{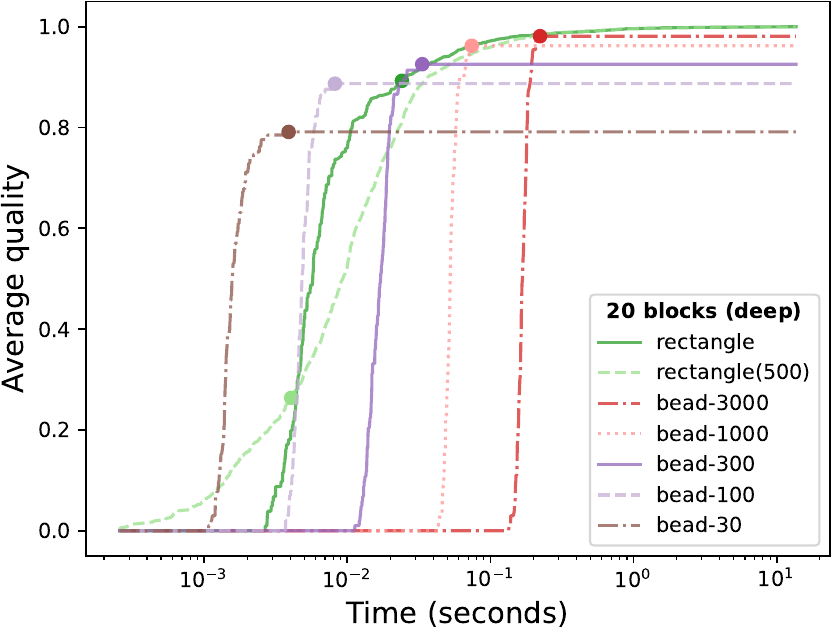}%
\caption{20 blocks (deep)}%
\end{subfigure}
\begin{subfigure}{\plotwidth{}}%
\includegraphics[width=\plotwidth{}]{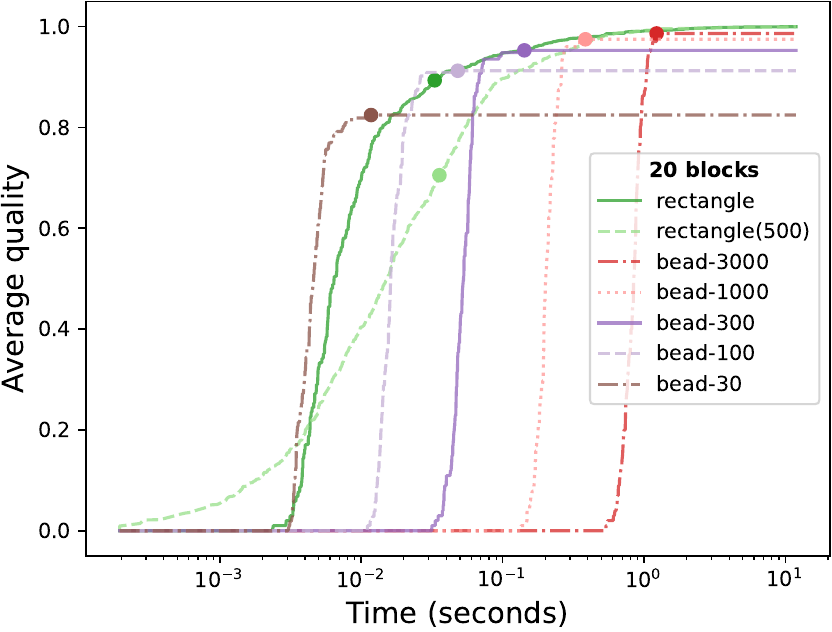}%
\caption{20 blocks}%
\end{subfigure}
\begin{subfigure}{\plotwidth{}}%
\includegraphics[width=\plotwidth{}]{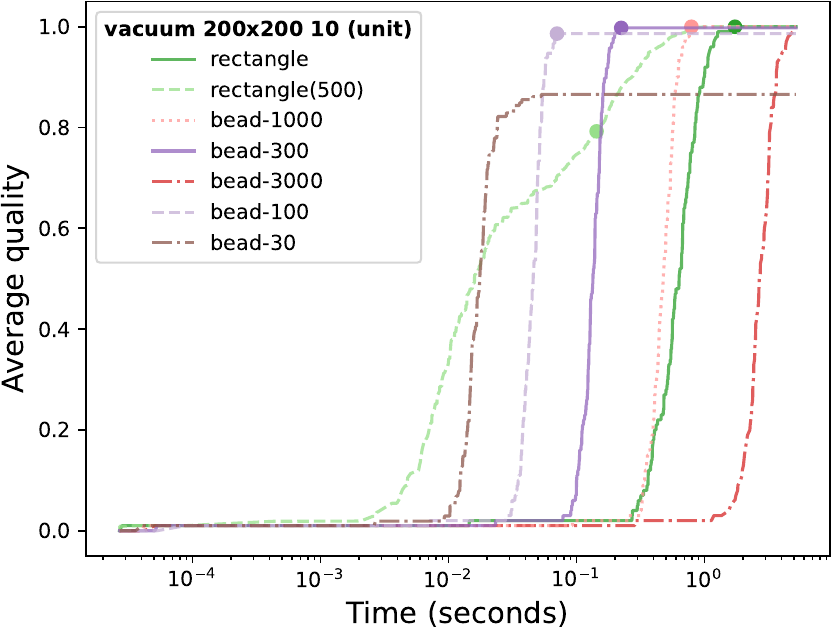}%
\caption{vacuum 200x200 10 (unit)}%
\end{subfigure}
\caption{}
\end{figure*}

\begin{figure*}[h!]%
\centering%
\begin{subfigure}{\plotwidth{}}%
\includegraphics[width=\plotwidth{}]{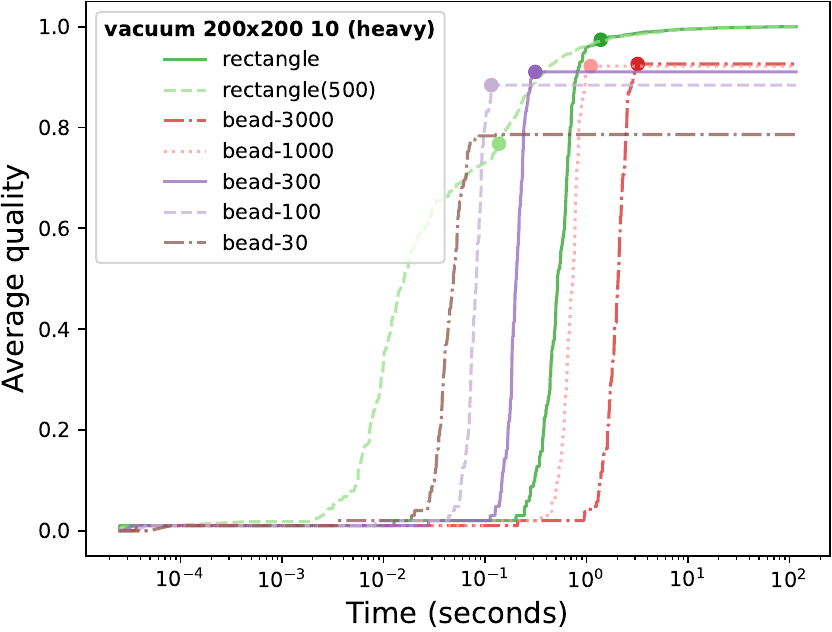}%
\caption{vacuum 200x200 10 (heavy)}%
\end{subfigure}
\begin{subfigure}{\plotwidth{}}%
\includegraphics[width=\plotwidth{}]{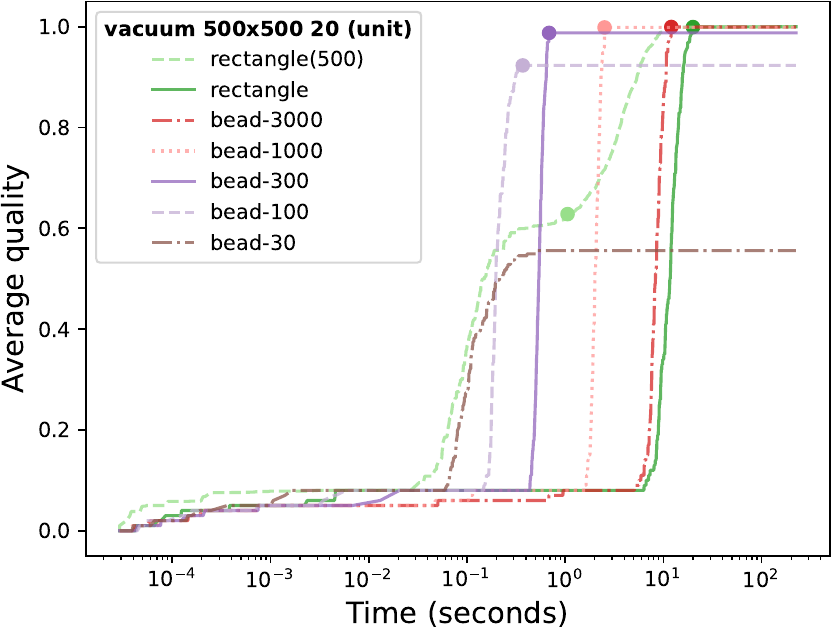}%
\caption{vacuum 500x500 20 (unit)}%
\end{subfigure}
\begin{subfigure}{\plotwidth{}}%
\includegraphics[width=\plotwidth{}]{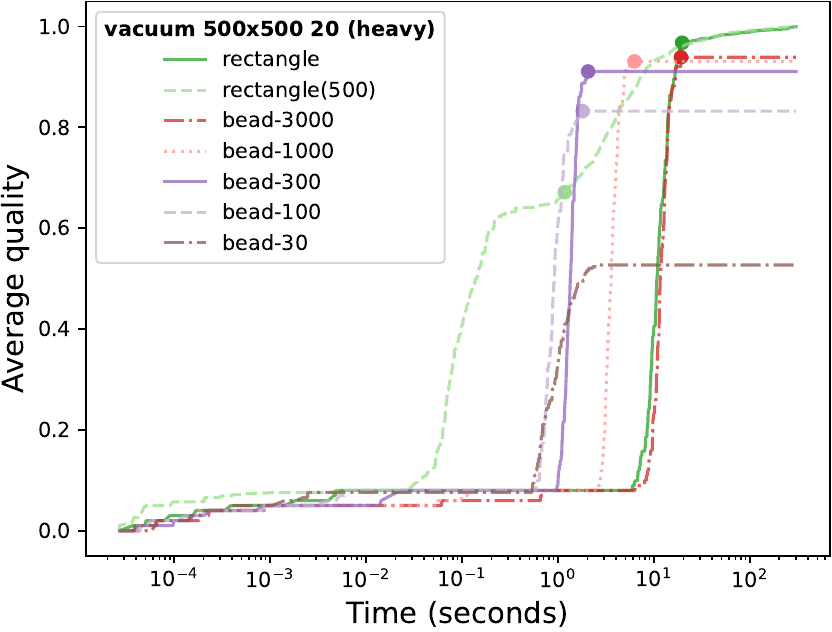}%
\caption{vacuum 500x500 20 (heavy)}%
\end{subfigure}
\caption{}
\end{figure*}

\begin{figure*}[h!]%

\centering%
\begin{subfigure}{\plotwidth{}}%
\includegraphics[width=\plotwidth{}]{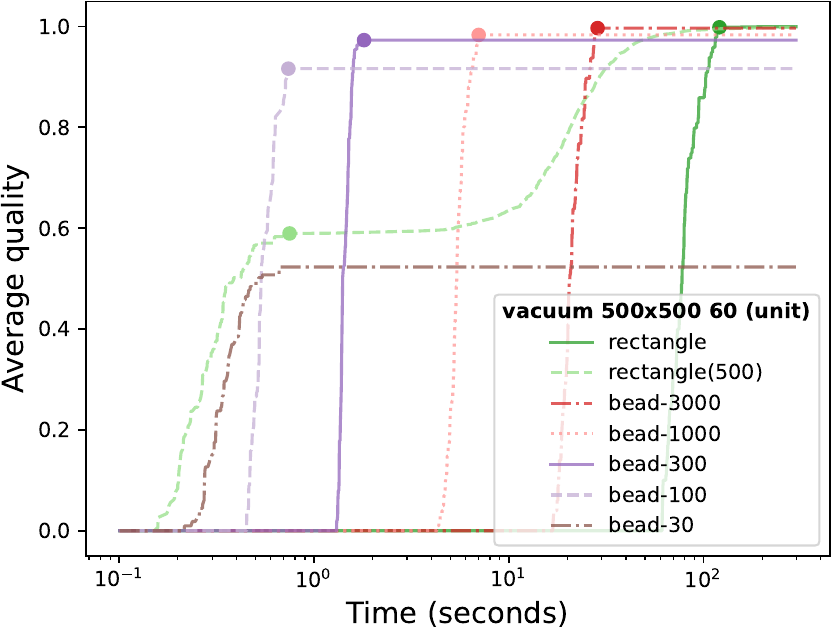}%
\caption{vacuum 500x500 60 (unit)}%
\end{subfigure}
\begin{subfigure}{\plotwidth{}}%
\includegraphics[width=\plotwidth{}]{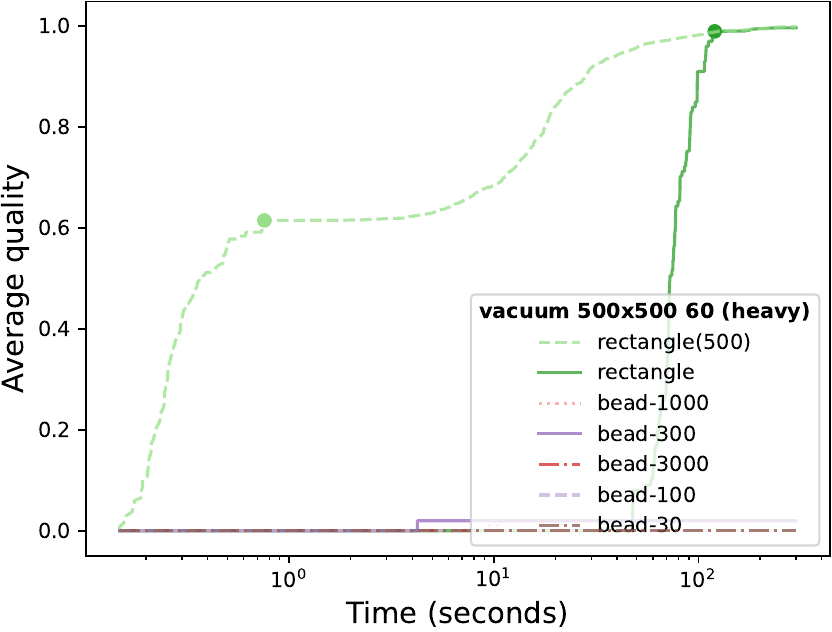}%
\caption{vacuum 500x500 60 (heavy)}%
\end{subfigure}
\begin{subfigure}{\plotwidth{}}%
\includegraphics[width=\plotwidth{}]{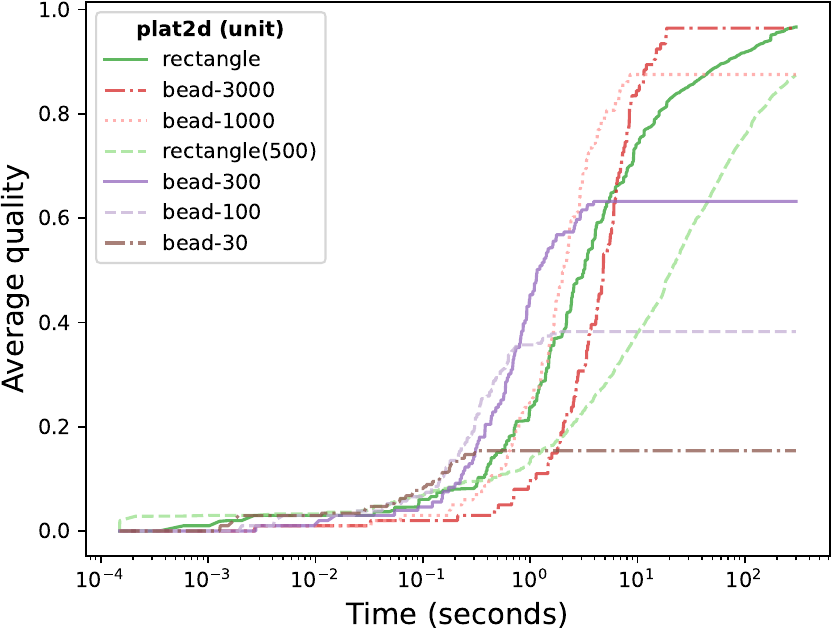}%
\caption{plat2d (unit)}%
\end{subfigure}
\caption{}
\end{figure*}

\begin{figure*}[h!]%
\centering%
\begin{subfigure}{\plotwidth{}}%
\includegraphics[width=\plotwidth{}]{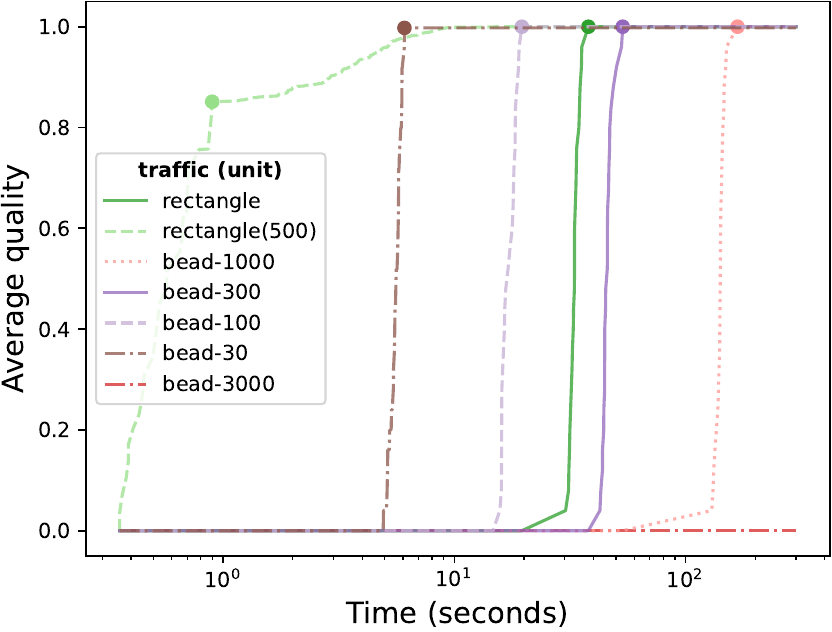}%
\caption{traffic (unit)}%
\end{subfigure}
\begin{subfigure}{\plotwidth{}}%
\includegraphics[width=\plotwidth{}]{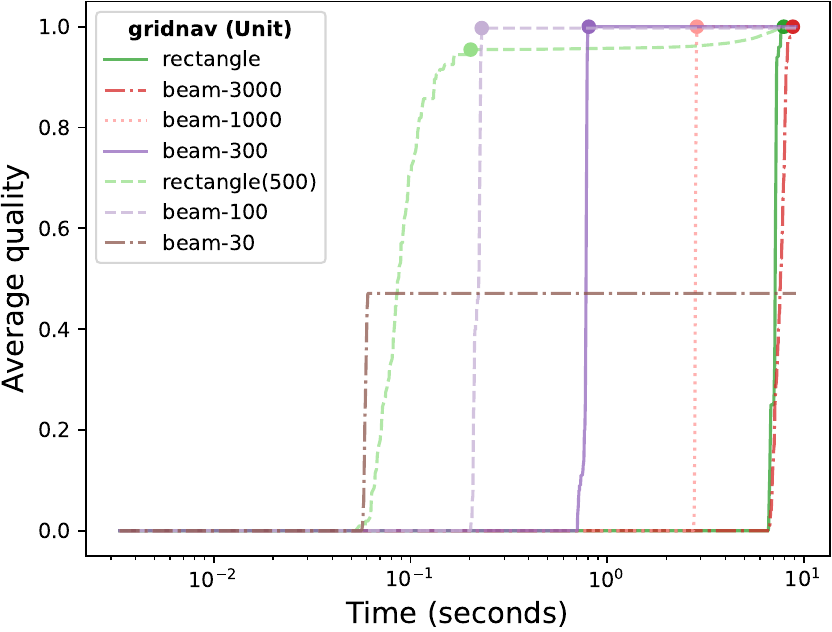}%
\caption{random grids (unit)}%
\end{subfigure}
\begin{subfigure}{\plotwidth{}}%
\includegraphics[width=\plotwidth{}]{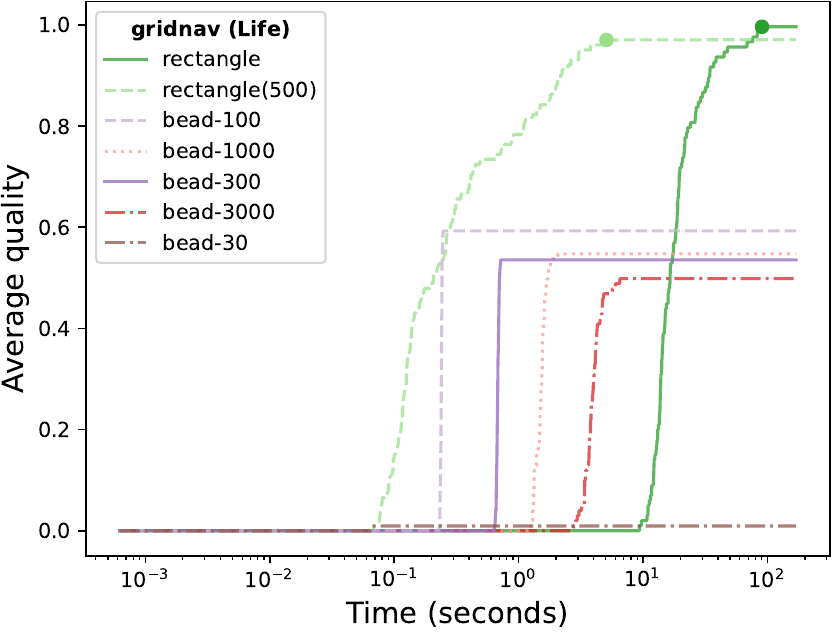}%
\caption{random grids (life)}%
\end{subfigure}
\caption{}
\end{figure*}

\clearpage
\bibliography{rectangle}

\end{document}